\documentclass{article}

\usepackage[preprint,nonatbib]{neurips_2019}

\usepackage[utf8]{inputenc} 
\usepackage[T1]{fontenc}    
\usepackage{hyperref}       
\usepackage{url}            
\usepackage{booktabs}       
\usepackage{amsfonts}       
\usepackage{nicefrac}       
\usepackage{microtype}      
\usepackage{graphicx}
\usepackage{bbm}
\usepackage{dsfont}
\usepackage{amsmath, amsthm, amssymb}
\usepackage{xcolor}
\usepackage{bm}

\newtheorem{theorem}{Theorem}
\newtheorem{proposition}[theorem]{Proposition}
\newtheorem{lemma}[theorem]{Lemma}
\newtheorem{corollary}[theorem]{Corollary}

\newcommand{\RR}{\mathbb{R}}

\newcommand{\dd}{\mathrm{d}}

\newcommand{\ie}{{\em i.e.}}

\title{Gradient Dynamics of Shallow Univariate ReLU Networks}

\author{
    Francis Williams\thanks{Equal contribution.} \\
    New York University \\
    \texttt{francis.williams@nyu.edu} \\
    \And
    Matthew Trager\footnotemark[1]\\
    New York University \\
    \texttt{matthew.trager@cims.nyu.edu} \\
    \AND
    Claudio Silva \\
    New York University \\
    \texttt{csilva@nyu.edu} \\
    \And
    Daniele Panozzo \\
    New York University \\
    \texttt{panozzo@nyu.edu} \\
    \And
    Denis Zorin \\
    New York University \\
    \texttt{dzorin@cs.nyu.edu} \\
    \And
    Joan Bruna \\
    New York University\\
    \texttt{bruna@cims.nyu.edu}\\
}

\begin{document}

\maketitle
\setcounter{footnote}{0}

\begin{abstract}

We present a theoretical and empirical study of the gradient dynamics of overparameterized shallow  ReLU networks with one-dimensional input, solving least-squares interpolation. We show that the gradient dynamics of such networks are determined by the gradient flow in a non-redundant parameterization of the network function. We examine the principal qualitative features of this gradient flow. In particular, we determine conditions for two learning regimes: \emph{kernel} and \emph{adaptive}, which depend both on the relative magnitude of initialization of weights in different layers and the asymptotic behavior of initialization coefficients in the limit of large network widths. We show that learning in the kernel regime yields smooth interpolants, minimizing curvature, and reduces to \emph{cubic splines} for uniform initializations. Learning in the adaptive regime favors instead \emph{linear splines}, where knots cluster adaptively at the sample points.

\end{abstract}

\section{Introduction}

An important open problem in the theoretical study of neural networks is to describe the dynamical behavior of the parameters during training and, in particular, the influence of the dynamics on the generalization error. To make progress on these issues, a number of studies have focused on a tractable class of architectures, namely single hidden-layer neural networks. For a fixed number of neurons, negative results establish that, even with random initialization, gradient descent may be trapped in arbitrarily bad local minima \cite{safran2017spurious,venturi2018spurious}, which motivates an asymptotic analysis that studies the optimization and generalization properties of these models as the number of neurons $m$ grows. 

Recently, several works \cite{du2018gradient, arora2019fine, cao2019generalization, du2018gradient2,oymak2018overparameterized} explain the success of gradient descent at optimizing the loss in the so-called \emph{over-parameterized} regime, \ie, for the number of neurons significantly higher than the number of training samples. In parallel, another line of work also established global convergence of gradient descent in the infinite-width limit using a seemingly different scaling \cite{chizat2018global, rotskoff2018neural, mei_mean_2018, sirignano2018mean}, relying on tools from optimal transport and mean-field theory.
In a remarkable effort,~\cite{chizat2018note} captured an essential difference between these two methodologies, stemming from their different scaling as $m\to \infty$: in one case, the neural network model asymptotically behaves as a kernel machine, with a kernel given by the linear approximation around its initialization \cite{jacot2018neural}, which in turn implies that as over-parameterization increases, the first-layer parameters stay close to their initial value. In contrast, the mean-field scaling 
offers a radically different picture, whereby parameters evolve in the limit $m\to \infty$ following a PDE based on a continuity equation. 

Although both scaling regimes explain the success of gradient descent optimization on overparametrized networks, they paint a different picture when it comes to generalization. The generalization properties in the kernel regime borrow from the well established theory of kernel regression in Reproducing Kernel Hilbert Spaces (RKHS), which has been applied to kernels arising from neural networks in \cite{hastie2019surprises, ghorbani2019linearized, ma2019comparative,rahimi2008random,daniely2017sgd}, and provide a somehow underwhelming answer to the benefit of neural networks compared to kernel methods. However, in practice large neural networks do not exhibit the traits of kernel/lazy learnings, since filter weights significantly deviate from their initialization despite the over-parameterization. Also, empirically, active learning provides better generalization than kernel learning \cite{chizat2018note}, although the theoretical reasons for this are still poorly understood. 

In order to progress in answering this important question, we consider a simplified setting, and study wide, single-hidden layer ReLU networks defined with one-dimensional inputs. We show how the kernel and active dynamics define fundamentally different function estimation models. For a fixed number of neurons, the neural network may follow either of these dynamics, depending on the initialization and the learning rates of the layers, and described by a simple condition on the initial weights. Specifically, we show that kernel dynamics corresponds to interpolation with \emph{cubic splines}, whereas active dynamics yields \emph{adaptive linear} splines, where neurons accumulate at the discontinuities and yield piecewise linear approximations. 

\subsection{Further related work}

Our work lies at the intersection between two lines of research: the works, described above, that develop a theory for 
optimization and generalization for shallow neural networks, and the works that attempt to shed 
light on these properties on low-dimensional inputs. 
In the latter category, we mention~\cite{basri2016efficient} for their study of the 
abilities of ReLU networks to approximate low-dimensional manifolds, and \cite{williams2018deep} 
for their empirical study of 3D surface reconstruction using precisely the intrinsic 
bias of SGD in overparametrised ReLU networks. Another remarkable recent work is \cite{daubechies2019nonlinear}, 
where the approximation power of deep ReLU networks is studied in the context of univariate functions. 
Our analysis in the active regime (cf Section \ref{sec:reduced_dynamics}) is closely related to \cite{maennel2018gradient}, in which the authors establish convergence of gradient descent to piece-wise linear functions under initialisations sufficiently close to zero. We provide an Eulerian perspective using Wasserstein gradient flows that simplifies the analysis, and is consistent with their conclusions. 
The implicit bias of SGD dynamics appears in several works, such as \cite{soudry2018implicit, gunasekar2018characterizing}, 
and, closest to our setup, in \cite{savarese2019infinite}, where the authors observe a link 
between gradient dynamics and \emph{linear splines}. They do not however observe the connection with \emph{cubic splines}, although they observe experimentally that the function returned by a network is often smooth and not piecewise linear.
Finally, we mention the works that attempt to study the tesselation of ReLU networks on the 
input space \cite{hanin2019complexity}.



\subsection{Main contributions}

The goal of this paper is to describe the qualitative behavior of the dynamics or 1D shallow ReLU networks. Our main contributions can be summarized as follows.
\begin{itemize}
    \item We investigate the gradient dynamics of shallow 1D ReLU networks using a ``canonical'' parameterization (Section~\ref{sec:reduced_dynamics}). We use machinery from mean field theory to show that the dynamics in this case are are completely determined by the evolution of the residuals. Furthermore, neurons will always converge at certain samples points where the residual is large and of opposite sign compared to neighboring samples. This means that the dynamics in the reduced parameterization biases towards functions that are \emph{piecewise-linear}.
    \item We observe that the dynamics in full parameters are related to the dynamics in canonical parameters by a change of metric which depends \emph{only} on the ``lift'' at initialization (\ie, on which full parameters are associated with a given function). In particular, the change of metric is expressible in terms of invariants $\delta_i$ associated with each neuron (Proposition~\ref{thm:reduced_parameter_grad}). When $\delta_i \gg 0$ the dynamics in full parameters (locally) agree with the dynamics in reduced parameters; when $\delta_i \ll 0$, the dynamics in full parameters (locally) follow \emph{kernel dynamics}, in which only the outer layer weights change.
    \item We consider the idealized kernel dynamics in the limit of infinite neurons, and we show that the RKHS norm of a function $f$ corresponds to a weighted L2 norm of the second derivative $f''$, \ie, a form of \emph{linearized curvature}.
   For appropriate initial distributions of neurons, the solution to kernel learning is a smooth \emph{cubic spline} (Theorem~\ref{prop:splines}). This illustrates the qualitative difference between the ``reduced'' and ``kernel'' regimes, which depend on parameter lift at initialization. Arbitrary initializations will locally interpolate between these two regimes, as can be seen using NTK kernel~\cite{jacot2018neural}. 
    \item Throughout our presentation, we also discuss the effect of a scaling parameter $\alpha(m)$ applied the network function  (where $m$ is the number of neurons), which becomes important as the number of neurons tends to infinity. As argued in~\cite{chizat2018note}, when $\alpha(m) = o(m)$, the variation of each neuron will asymptotically go to zero (\emph{lazy regime}), so our local analysis translates into a global one. 
\end{itemize}
\section{Preliminaries}

\subsection{Problem Setup}

The training of a two-layer ReLU neural network with $m$ scalar inputs, a single scalar output, and least-squares loss solves the following problem:
\begin{equation}\label{eq:leastsquares}
\begin{aligned}
    &\min_{\bm z} L(\bm z) =  \frac{1}{2} \sum_{j=1}^s |f_{\bm z}(x_j) - y_j|^2\\
    &\mbox{ where } \,\,\, f_{\bm z}(x) = \frac{1}{\alpha(m)} \sum_{i=1}^m c_i [a_i x - b_i]_+, \quad \bm z = (\bm a \in \RR^m, \bm b \in \RR^m, \bm c \in \RR^m).
\end{aligned}
\end{equation}
Here, $S = \{ (x_i, y_i) \in \RR^2, \, i =1,\ldots, s\}$ is a given set of $s$ samples, $\bm z$ is a vector of \emph{parameters}, and $\alpha(m)$ is a normalization factor 
that will be important as we consider the limit $m \to \infty$. 

We are interested in the minimisation of (\ref{eq:leastsquares}) performed by gradient descent over the parameters $\bm z$, evolving in a domain $\mathcal{D} = \RR^{3m}$: 
$$\bm z^{(k+1)} = \bm z^{(k)} - \eta \nabla_{\bm z} L(\bm z^{(k)})~.$$
As $\eta \to 0$, this scheme may be analysed through its continuous-time counterpart, the gradient flow 
\begin{equation}\label{eq:gradient_flow}
    \bm z(0) = \bm z_0, \qquad \bm z'(t) = -\nabla L(\bm z(t)).
\end{equation}
While (\ref{eq:gradient_flow}) describes the dynamics $\bm z(t)$ in the parameter space $\mathcal{D}$, one is ultimately interested in the learning trajectories of the function $f_{\bm z(t)}$ they represent. Let $\mathcal{F}:=\{ f: \RR \to \RR\}$ denote the space of square-integrable scalar functions, and $\varphi: \mathcal{D} \to \mathcal{F}$ the function-valued mapping $\varphi(\bm z) := f_{\bm z}$. Then, by observing
that $L(\bm z) = R ( \varphi(\bm z))$, with $R(f)= \frac{1}{2} \sum_{j \leq s} |f(x_j) - y_j|^2$, the chain rule immediately shows that the dynamics of $g(t):=\varphi(\bm z(t)) = f_{\bm z(t)}$ are described by
\begin{equation}
\label{eq:gradient_flow_functional}
    g(0) = f_{\bm z_0},\qquad g'(t) = - \nabla \varphi(\bm z(t))^\top \nabla \varphi(\bm z(t)) \nabla R( g(t))~. 
\end{equation}
The dynamics in function space are thus controlled by a time-varying \emph{tangent kernel} $K_t:= \nabla \varphi(\bm z(t))^\top \nabla \varphi(\bm z(t))$, that under appropriate assumptions on the parametrisation remains nearly constant throughout the optimization  \cite{jacot2018neural, chizat2018note}. This kernel may be interpreted as a change of metric resulting from a specific parametrisation of the functional space. 

A simple, yet important, observation from (\ref{eq:gradient_flow_functional}) is that the evolution of the predicted function depends on two essential aspects: the initialization and the parametrisation. 
In particular, the coefficients of each neuron can be rescaled using a positive scale factor $k>0$ according to $(a,b,c) \mapsto (a k, b k, c/k)$. Similarly, the normalization factor $\alpha(m)$ can be absorbed into $c_i$ for all $i$.

In order to study the impact of different choices of parametrisation and initialisation, as well as the asymptotic behavior of (\ref{eq:gradient_flow_functional}) as $m$ increases, we introduce the following \emph{canonical parameterization}:
\begin{equation}\label{eq:reduced_parameterization}
    \tilde f_{\bm w}(x) = \frac{1}{m} \sum_{i=1}^m r_i \langle {\tilde x}, d(\theta_i) \rangle_{+}, \qquad \bm w = (\bm r \in \RR^m, \bm \theta \in [0, 2\pi)^m),  \qquad {\tilde x} = (x,1).
\end{equation}
where $d(\theta_i) = (\cos \theta_i, \sin \theta_i) \in S^1$ and we chose $\alpha(m) = m$. As shown in Section \ref{sec:reduced_dynamics}, this choice of normalization, together with an initialisation where $r_i$ are sampled iid from a distribution with $O(1)$ variance, provides a well-defined limit of the dynamics as $m\to \infty$, and provides a basis to compare the other normalization and parametrisation choices.

\subsection{Reparametrisation and Normalization} 

Thanks to the fact that the canonical parametrisation admits a well-defined limit, 
we can study the effect of specific parametrisations and normalisations 
by expressing them as changes of variables relative to $\bm w \mapsto \tilde{f}_{\bm w}$. 
The mapping from the weights $({\bm a},{\bm b},{\bm c})$ to $({\bm r},{\bm \theta})$ in the canonical parametrisation is given by
\begin{equation}
    \bm \pi
    (a_i, b_i, c_i)  = \left(\frac{m}{\alpha(m)} c_i \sqrt{a^2_i + b^2_i}, \arctan(-b_i/a_i) \right) =(r_i,\theta_i). 
\end{equation}
This map satisfies $\tilde f_{\bm \pi(\bm z)} = f_{\bm z}$. We also define the loss with respect to this parameterization as $\tilde L(\bm w) = L(\bm z)$ where $\bm w = \bm \pi(\bm z)$. Section \ref{sec:full-dynamics} explains how the Jacobian of $\bm \pi$ affects the dynamics relative to the canonical dynamics. In particular, this illustrates the 
drastic effects of a choice of normalisation $\alpha(m) = o(m)$. 

Note that $f_{\bm z}(x)$ (and thus $\tilde f_{\bm w}(x)$) are \emph{continuous piecewise linear functions} in $x$. The \emph{knots} of of these functions are the points where the operand inside a ReLU activation changes sign: $e_i = \frac{b_i}{a_i}, \, a_i \ne 0, \, i=1,\ldots,m.$ 
See the left image in Figure~\ref{fig:knots} for an example of a function $f_{\bm z}$ and its knots.

\subsection{Visualizing the Network State}

\begin{figure}
    \centering
    \minipage{0.4\textwidth}
    \includegraphics[width=\textwidth]{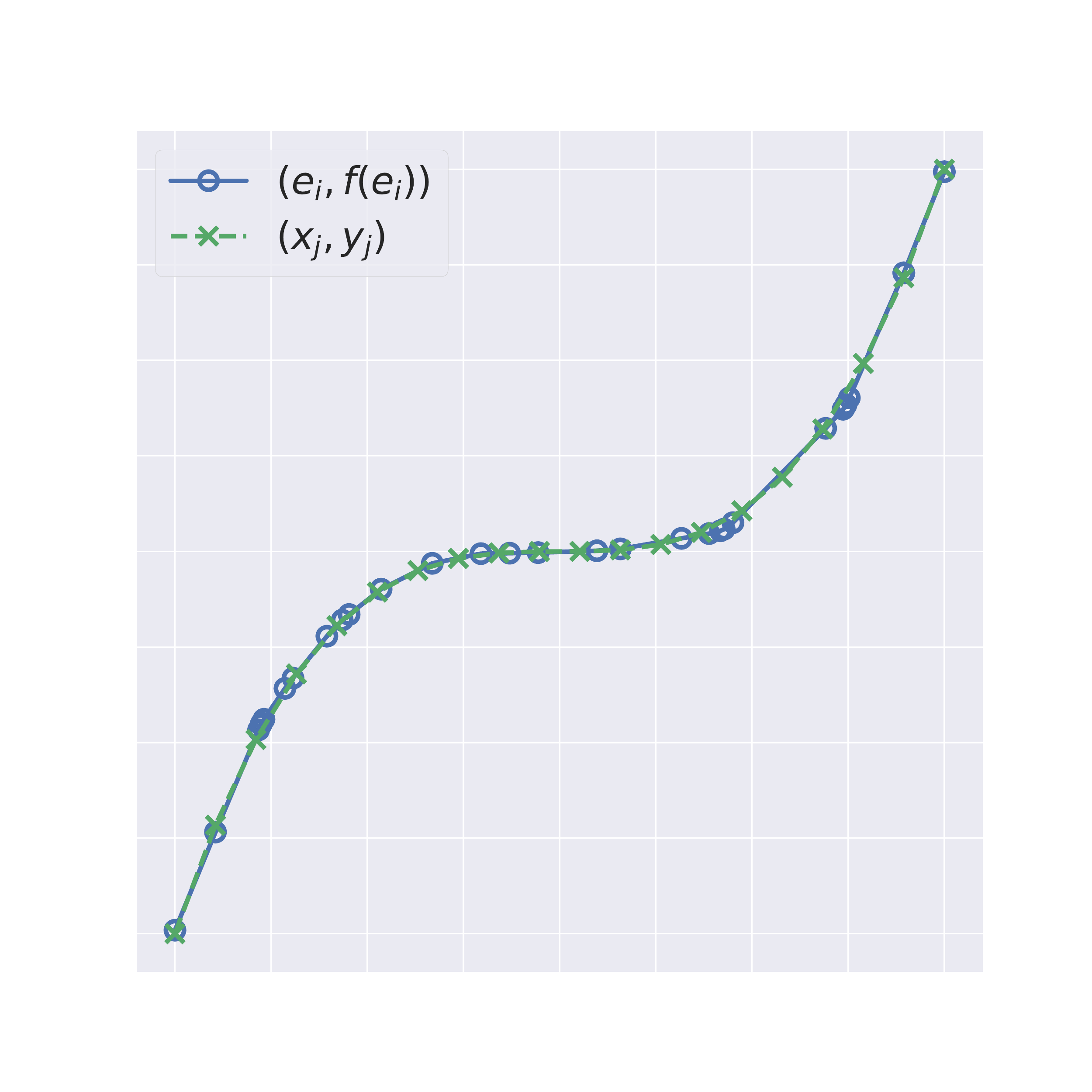}
    \endminipage\hfill
    \minipage{0.4\textwidth}
    \includegraphics[width=\textwidth]{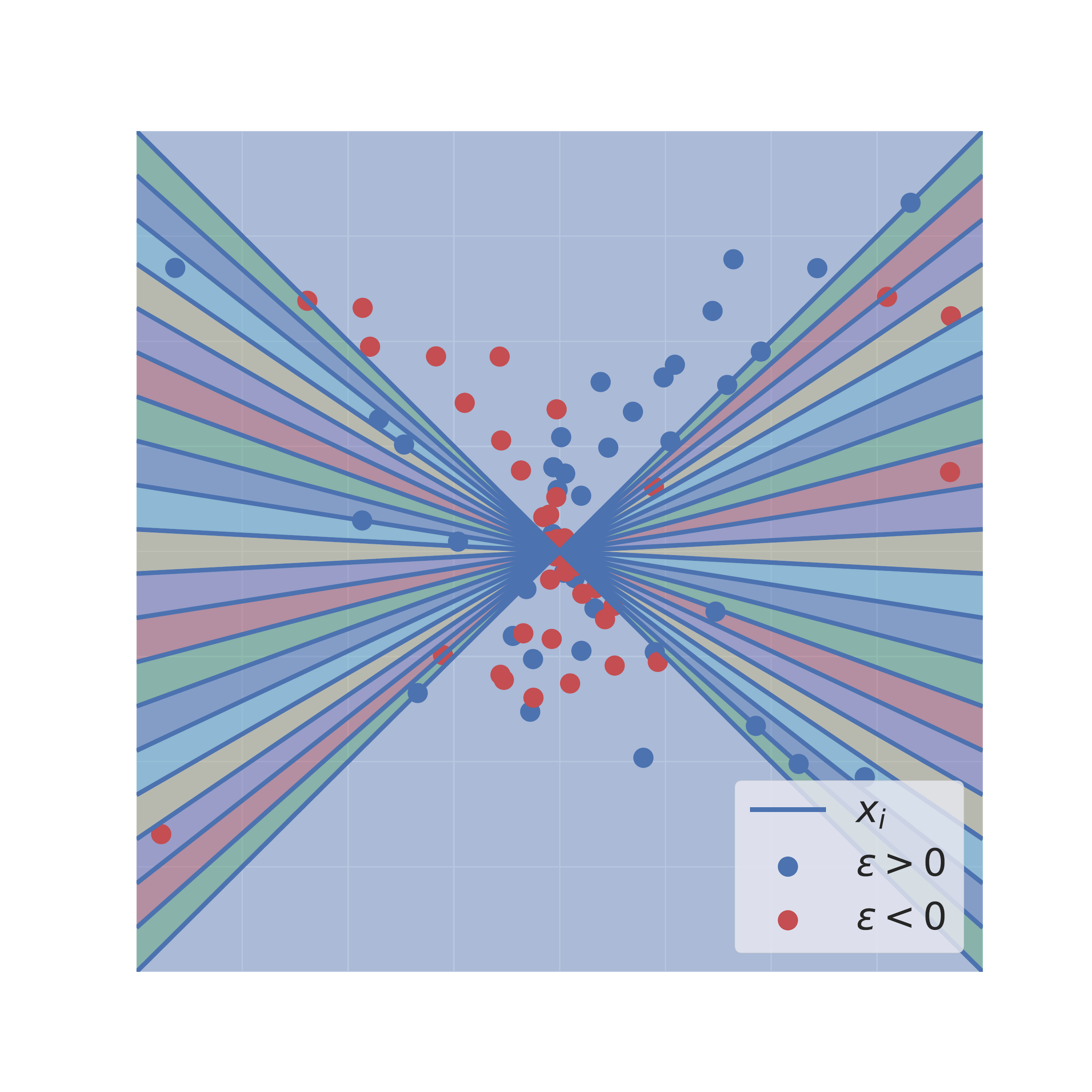}
    \endminipage\hfill\vspace{-15pt}
    \caption{\textit{Left:} Input samples (blue x's) to which we fit a neural network $f_{\bm z}(x)$ using the least squares loss \eqref{eq:leastsquares}. Clearly, $f_{\bm z}(x)$ is piecewise linear with the boundaries between pieces occurring at $(e_i, f_{\bm z}(e_i))_{i=1}^m$ (green circles). 
    \textit{Right:} the neurons visualized as $(u_i, v_i)$ in \eqref{eq:uv}. Each particle represents a neuron and the color indicates the sign of $\epsilon_i$. The samples $x_j$ correspond to the lines $u x_j + v = 0$ in this space. These sample lines divide the space into the colored regions which correspond to different activation patterns. \vspace{-15pt}}
    \label{fig:knots}
\end{figure}

In this paper, we visualize the state of the network in two ways. First, we  plot $f_{\bm z}$ to visualize the piecewise-linear function encoded by the network and its knots at given time (Figure~\ref{fig:knots}, Left). We use the alternative parameterization $\tilde f_{\bm w}$ to visualize the reduced state of the network in $\RR^2$, by plotting a neuron $(r_i, \theta_i)$ as a particle with coordinates
\begin{equation}\label{eq:uv}
(u_i, v_i) = (|r_i|\cos(\theta_i), |r_i|\sin(\theta_i))
\end{equation}
and coloring each particle according to $\epsilon_i = \text{sign}(r_i)$. Figure~\ref{fig:knots} (right) shows an example of this visualization. In the standard parameterization, a neuron $(a_i, b_i, c_i)$ coincides with a point $x$ if $\frac{b_i}{a_i} = \frac{v_i}{u_i} = x$. Thus, in the visualization, a sample point $x_j$ corresponds to the line satisfying the equation $u x_j + v = 0$. These sample lines, divide the phase space into \emph{activation regions} where a neuron has a fixed \emph{activation pattern} (the colored regions in the right image of Figure~\ref{fig:knots}).

\section{Training Dynamics}

Our goal is to solve \eqref{eq:leastsquares} using the \emph{gradient flow} \eqref{eq:gradient_flow} of the loss $L(\bm z)$\footnote{To be precise, we should replace the gradient $\nabla L(\bm z)$ with the \emph{Clarke subdifferential} $\partial L(\bm z)$~\cite{clarke1975generalized}, since $L(\bm z)$ is only piecewise smooth.  At generic smooth points $\bm z$, the subdifferential coincides with the gradient $\partial L(\bm z(t)) = \{\nabla L(\bm z)$\}.}.
We begin in Section~\ref{sec:reduced_dynamics} by investigating the dynamics in the canonical parametrisation:
\begin{equation}\label{eq:reduced_gradient_flow}
    \bm w(0) = \bm w_0, \qquad \bm w'(t) = -\nabla \tilde L(\bm w(t)).
\end{equation}
While the relationship between flows \eqref{eq:gradient_flow} and \eqref{eq:reduced_gradient_flow} is nonlinear, we show in Section~\ref{sec:full-dynamics} that these are related by a simple change of metric.

\subsection{Dynamics in the Canonical Parameters}\label{sec:reduced_dynamics}

We now examine the qualitative behavior of the gradient flow using the canonical parameterization \eqref{eq:reduced_parameterization}. In this setting, 
when $({r}_i, \theta_i)$ are initialized i.i.d. from some base distribution $\mu_0({r},\theta)$, the function $\tilde{f}_{\bm w}$ is well-defined pointwise as $m \to \infty$ (by the law of large numbers). We optimize the loss over both $\bm \theta$ and $\bm {{r}}$ using the reduced gradient flow \eqref{eq:reduced_gradient_flow}.


Using the initialization $(r_i, \theta_i) \sim \mu$, and following the mean-field formulation of single-hidden layer neural networks of \cite{mei_mean_2018, chizat2018global, rotskoff2018neural}, we express the function as an expectation with respect to the probability measure $\mu$ over the cylinder $\mathcal{D} = \mathbb{R} \times S^1$:
\begin{equation}
    \tilde{f}_{\bm w}(x) = \int_\mathcal{D} \varphi(w;x) \mu^{(m)}(\dd w)~,
\end{equation}
where $\varphi(w;x):= {r}_i \langle \tilde{x}, d(\theta_i) \rangle_+$ and $\mu^{(m)}(w) = \frac{1}{m} \sum_{i=1}^m \delta_{w_i}(w)$ is the empirical measure determined by the $m$ particles $w_i$, $i=1\dots m$. The least squares loss in this case becomes 
\begin{eqnarray}
    \tilde{L}(\bm w) &=& \frac{1}{2}\| \tilde{f}_{\bm w} - y \|_\mathcal{X}^2 \\
    &=& C_y - \frac{1}{m} \sum_{i=1}^m \langle \varphi_{w_i} , y \rangle_\mathcal{X} + \frac{1}{2m^2} \sum_{i,i'=1}^m \langle \varphi_{w_i} ,\varphi_{w_{i'}} \rangle_\mathcal{X} ~,
\end{eqnarray}
where $\langle f, g \rangle_\mathcal{X} := \sum_{j=1}^s f(x_j) g(x_j)$ is the empirical dot-product. 
This loss may be interpreted as the Hamiltonian of a system of $m$-interacting particles, under external field $F$ and interaction kernel $K$ defined respectively by 
$F(w):= \langle \varphi_w, y \rangle_\mathcal{X}$,$K(w, w'):= \langle \varphi_{w} ,\varphi_{w'} \rangle_\mathcal{X}$.
We may also express this Hamiltonian in terms of the empirical measure, by abusing notation
$$\tilde{L}(\mu^{(m)}) = C_y - \int_\mathcal{D} F(w) \mu^{(m)}(\dd w) + \frac{1}{2} \iint_{\mathcal{D}^2} K(w,w') \mu^{(m)}(\dd w) \mu^{(m)}(\dd w')~.$$

A direct calculation shows that the gradient $\nabla_{w_i} \tilde{L}(\bm w)$ can be written as 
$$\frac{m}{2} \nabla_{w_i} \tilde{L}(\bm w) = \nabla_w V(w_i; \mu^{(m)})~,$$
where $V$ is the potential function $
    V(w;\mu):= -F(w) + \int_\mathcal{D} K(w, w') \mu(\dd w')$.
    
The gradient flow in the space of parameters $\bm w$ can now be interpreted in Eulerian terms as a gradient flow in the space of measures over $\mathcal{D}$, by using the notion of Wasserstein gradient flows \cite{mei_mean_2018, chizat2018global, rotskoff2018neural}. Indeed,  particles evolve in $\mathcal{D}$ by ``feeling'' a velocity field $\nabla V$ defined in $\mathcal{D}$. 
This formalism allows us now to describe the dynamics independently of the number of neurons $m$, by replacing the empirical measure $\mu^{(m)}$ 
by any generic probability measure $\mu$ in $\mathcal{D}$. The evolution of a measure under a generic time-varying vector field is given by the 
so-called continuity equation:\footnote{Understood in the weak sense, \ie,
$\partial_t \left(\int_\mathcal{D} \phi(w) \mu_t(\dd w)\right) = - \int \langle \nabla \phi(w), \nabla V(w;\mu_t) \rangle \mu_t(\dd w)$, 
$\forall \phi \in C^1_c(\mathcal{D})$ continuously differentiable and with compact support.}
\begin{equation}
\label{eq:continuity}
    \partial_t \mu_t = \mathrm{div} ( \nabla V \mu_t)~.
\end{equation}
 The global convergence of this PDE for interaction kernels arising from single-hidden layer neural networks has been established under mild assumptions in \cite{mei_mean_2018, chizat2018global, rotskoff2019global}. Although the conditions for global convergence hold in the mean field limit $m \to \infty$, a propagation-of-chaos argument from statistical mechanics gives Central Limit Theorems for the behavior of finite-particle systems as fluctuations of order $1/\sqrt{m}$ around the mean-field solution; see \cite{rotskoff2018neural, rotskoff2019global} for further details. 

The dynamics in $\mathcal{D}$ are thus described by the velocity field $\nabla V(w; \mu_t)$, which depends on the current state of the system through the measure $\mu_t(w)$, describing the probability of encountering a particle at position $w$ at time $t$. We emphasize that equation (\ref{eq:continuity}) is valid for any measure, including the empirical measure $\mu^{(m)}$, and is therefore an exact model for the dynamics in both the finite-particle and infinite-particle regime. Let us now describe its specific form in the case of the empirical loss given above. 

Assume without loss of generality that the data points $x_j \in \mathbb{R}$, $j\leq s$ satisfy $x_j \leq x_{j'}$ whenever $j < j'$.
Denote $$\mathcal{C}_j := \{ j' ; j' \leq j\} \text{ for } j=1\dots s, \qquad \mathcal{C}_{s+j} := \{ j' ; j' > j \},\text{ for }j=1\dots s-1~.$$
We observe that for each $j$, two angles $\alpha_j^{\pm} = \arctan(x_j) \pm \pi/2$ partition the circle $S^1$ into $2s-1$ regions $\mathcal{A}_k$ (visualized as the colored regions in Figure~\ref{fig:knots}), 
which are in one-to-one correspondence with the sets $\mathcal{C}_k$, in the sense that 
$$\theta \in \mathcal{A}_k \Longleftrightarrow \{j ; \langle \tilde{x}_j , d(\theta) \rangle \geq 0 \} = \mathcal{C}_k~.$$
We also denote by $\mathcal{B}_j$ 
the half-circle where $\langle \tilde{x}_j, \theta\rangle \geq 0$. Let $t(\theta)$ be the tangent vector of $S^1$ at $\theta$ (so $t(\theta) = d(\theta)^\bot$) and $ w=( {r}, \theta)$, where we suppose $\theta \in \mathcal{A}_k$. A straightforward calculation (presented in  Appendix \ref{sec:mfappendix}) shows that the angular and radial components of the velocity field $\nabla V(w; \mu_t)$ are given by 
\begin{equation}\label{eq:v_field}
    \begin{aligned}
    &\nabla_\theta V(w; \mu_t)  =
         {r} \left \langle \sum_{j \in \mathcal{C}_k} \rho_j(t) \tilde{x}_j, t(\theta) \right \rangle~,\\
    &\nabla_{{r}} V(w; \mu_t) =   \left \langle \sum_{j \in \mathcal{C}_k} \rho_j(t) \tilde{x}_j, d(\theta) \right \rangle ~,
    \end{aligned}
\end{equation}
where $\rho_j(t) =  
\int_{\mathbb{R} \times \mathcal{B}_j} {r} \langle \tilde{x}_j,\theta  \rangle \mu_t(\dd {r}, \dd \theta)  - y_j$
is the residual at point $x_j$ at time $t$. 
These expressions show that the dynamics are entirely controlled by the 
$s$-dimensional vector of residuals $\bm \rho(t)=(\rho_1(t), \dots \rho_s(t))$, and that the velocity field is \emph{piece-wise linear} 
on each cylindrical region $\mathbb{R} \times \mathcal{A}_k$ (e.g. Figure~\ref{fig:attractor} in Appendix~\ref{sec:additional_experiments}). 
Under the assumptions that ensure global convergence of (\ref{eq:continuity}), we have $\lim_{t \to \infty} \tilde{L}(\mu_t) = 0$, 
and therefore $\| \bm \rho(t) \| \to 0$. The oscillations of $\bm \rho(t)$ as it converges 
to zero determine the relative orientation of the flow within each region. 
The exact dynamics for the vector of residuals are given by the following proposition, proved in Appendix \ref{sec:mfappendix}:
\begin{proposition}
\label{prop:resiode}
For each $j$, 
\begin{equation}
\label{eq:oderesiduals0}
\dot{\rho}_j(t) =  -\tilde{x}_j^\top \sum_{k; \mathcal{A}_k \subset \mathcal{B}_j}  \Sigma_k(t) \left(\sum_{j' \in \mathcal{C}_k} \rho_{j'}(t) \tilde{x}_{j'}\right)~, 
\end{equation}
where 
$$\Sigma_k(t) = \int_{\mathbb{R} \times \mathcal{A}_k} \left(r^2 t(\theta)\, t(\theta)^\top + d(\theta)\, d(\theta)^\top\right) \mu_t(\dd r,\dd \theta) $$
tracks the covariance of the measure along each cylindrical region. 
\end{proposition}
Equation (\ref{eq:oderesiduals0}) 
defines a system of ODEs for the residuals $\bm \rho(t)$, but its coefficients are time-varying, and behave roughly as quadratic terms in $\bm \rho(t)$ (since they are second-order moments of the measure whereas the residuals are first-order moments). It may be possible to obtain asymptotic control of the oscillations $\bm \rho(t)$ by applying Duhamel's principle, but this is left for future work.

Now let $w=({r},\theta)$ with $\theta$ at a boundary of two regions $\mathcal{A}_k$, $\mathcal{A}_{k+1}$. The velocity field is modified at the transition  by 
$$
\nabla V(w)\lvert_{\mathcal{A}_k} - \nabla V(w)\lvert_{\mathcal{A}_{k+1}} = \rho_{j*}(t) \left( 
\begin{array}{c}
{r}  \langle \tilde{x}_{j*}, t(\theta) \rangle \\
\langle \tilde{x}_{j*}, \theta \rangle
\end{array}\right)
~,$$
where $j_*$ is such that $\langle \tilde{x}_{j*}, d(\theta) \rangle =0$, since $\theta$ is at the boundary of $\mathcal{A}_k$. It follows that the only discontinuity is in the angular direction, of magnitude $|{r} \rho_{j*}(t)| \|\tilde x_{j^*}\|$.
In particular, an interesting phenomenon arises when the angular components of $\nabla V(w)\lvert_{\mathcal{A}_k}$ and $\nabla V(w)\lvert_{\mathcal{A}_{k+1}}$ have opposite signs, corresponding to an ``attractor'' or ``repulsor'' that attracts/repels particles along the direction given by $\tilde{x}_{j*}$ (See Figure~\ref{fig:attractor} in Appendix~\ref{sec:additional_experiments}). Writing $s_k = \left \langle \sum_{j \in \mathcal{C}_k} \rho_j(t) \tilde{x}_j, t(\theta) \right \rangle$, we deduce from ~\eqref{eq:v_field} that this occurs when
\begin{equation}\label{eq:attractor}
|s_k| < |\rho_{j*}(t)|\|\tilde x_{j^*}\|
 \text{ and } \text{sign}(s_k) \neq \text{sign}(\rho_{j*}(t)).    
\end{equation}
We expand this condition in the following Lemma.
\begin{lemma}\label{le:attractor}
A data point $x_k$ is an attractor/repulsor if and only if
\begin{equation}
    \sum_{i=1}^{k-1} \rho_i \rho_k \langle \tilde x_i, \tilde x_k \rangle > - \rho_k^2 \|\tilde x_k\|^2, \mbox{ or } \sum_{i=k+1}^{s} \rho_i \rho_k \langle \tilde x_i, \tilde x_k \rangle > - \rho_k^2 \|\tilde x_k\|^2.
\end{equation}
\end{lemma}

In words, mass will concentrate towards input points where the residual is currently large and of opposite sign from a weighted average of neighboring residuals. This is in stark contrast with the kernel dynamics (Section~\ref{sec:kernel_dynamics}), where there is no adaptation to the input data points.
We point out that this qualitative behavior has been established in \cite{maennel2018gradient} under appropriate initial conditions, sufficiently close to zero, in line with our mean-field analysis. 

\paragraph{Regularised Objective in Functional Space:} 
The global convergence of \eqref{eq:continuity} studied in 
\cite{chizat2018global} includes the case where the energy functional $\tilde{L}=R(\int \Phi \mu(\dd w))$ is augmented with a regulariser $\Gamma(\mu) = \int \gamma(w) \mu(\dd w)$ sharing the same degree of homogeneity as $\Phi$. If $\gamma(w)=|r|$ this corresponds to the \emph{Total Variation} norm of the signed Radon measure $\tilde{\mu}(w) = \int \mu(\dd r, w)$:
$\Gamma(\mu) = \| \tilde{\mu}\|_{\mathrm{TV}}$. The Wasserstein gradient flow on the resulting regularised objective 
\begin{equation}
\label{birk}
\min_{\tilde{\mu}} \tilde{L}(\tilde{\mu}) + \lambda \|\tilde{\mu} \|_{\mathrm{TV}}
\end{equation}
thus converges to global minimisers under appropriate initial conditions (which only apply in the infinite width regime $m=\infty$). The regularised dynamics are obtained by replacing $\nabla_{{r}} V(w; \mu_t) =   \left \langle \sum_{j \in \mathcal{C}_k} \rho_j(t) \tilde{x}_j, d(\theta) \right \rangle$ in (\ref{eq:v_field}) by 
$$\nabla_{{r}} V(w; \mu_t) =   \left \langle \sum_{j \in \mathcal{C}_k} \rho_j(t) \tilde{x}_j, d(\theta) \right \rangle + \lambda \mathrm{sign}(r)~.$$
In the case of scalar inputs, \cite{savarese2019infinite} recently characterised the solutions of (\ref{birk}) in functional space, as 
\begin{equation}
\label{eq:tvfunctional}
    \min_{f} L(f) + \lambda \int |f''(x)| \dd x~,
\end{equation}
for vanishing boundary conditions. In Section \ref{sec:kernel_dynamics} we will see that 
the kernel regime corresponds to a very different prior, where the $L^1$ norm on the second derivatives is replaced by a Hilbert $L^2$ norm. The distinction between kernel and active regime in terms of $L^1$ versus $L^2$ norms was already studied in \cite{bach2017breaking}.







\subsection{Dynamics in the Full Parameters}
\label{sec:full-dynamics}

The dynamics of gradient flow~\eqref{eq:gradient_flow} are different from the dynamics of the gradient flow~\eqref{eq:reduced_gradient_flow}. For the gradient flow in canonical parameters we have observed adaptive learning behavior under the assumption of iid distribution of parameter initialization. 
The full set of parameters $\bm z = (\bm a, \bm b, \bm c)$, may exhibit both kernel and adaptive behavior depending on the initialization. To characterize this behavior we rely on the following lemma:

\begin{lemma}\label{le:fixed_delta}
If $\bm z(t) = (\bm a(t), \bm b(t), \bm c(t))$ is a solution of the gradient flow \eqref{eq:gradient_flow}, then the quantities
\begin{equation}\label{eq:invariants}
\bm \delta = (c_i(t)^2 - a_i(t)^2 - b_i(t)^2)_{i=1}^m
\end{equation}
remain constant for all $t$.
In particular, given a reduced neuron $(r_i,\theta_i)$, we can uniquely recover the original neuron parameters  $(a_i,b_i,c_i)$ from $\delta_i$ computed from the initialization.
\end{lemma}
Lemma~\ref{le:fixed_delta} allows us to analyze how canonical parameters evolve 
under \emph{full} gradient flow in $(\bm a, \bm b, \bm c)$. Overall, the behavior is qualitatively the same, 
except it is in addition dependent on the relative scale of redundant parameters. 




\begin{proposition}\label{thm:reduced_parameter_grad}
Let $\bm z(t)$ be a solution of the gradient flow \eqref{eq:gradient_flow} of $L(\bm z)$, and let $\bm \delta = (\delta_i) \in \RR^m$ be the vector of invariants~\eqref{eq:invariants}, which depend only on the initialization $\bm z(0)$. If $\bm w(t) = (\bm r(t), \bm \theta(t))$ is curve of canonical parameters corresponding to $\bm z(t)$, then we have that
\begin{equation}
\dot{\bm w}_i(t) = \bm P_i \cdot \nabla_{\bm w_i} \tilde L(\bm w),
\quad i=1,\ldots,m,
\end{equation}
where
\begin{equation}\label{eq:neuron_kernel}
\bm P_i  = \begin{bmatrix}
    \frac{m^2}{\alpha(m)^2}(a_i^2 + b_i^2 + c_i^2)  & 
    0        \\
    0   & \frac{1}{a^2_i + b^2_i} \\
\end{bmatrix}.
\end{equation}
\end{proposition}
With respect to rescaled differentials $d\tau = rd\theta$, corresponding to representing the flow locally in a Cartesian system aligned with the radial direction (pointing away from $\bm z = \bm 0$) and its perpendicular, the flow can be written as
\begin{equation}\label{eq:particle_vfield}
\begin{bmatrix}
dr_i\\
d\tau_i
\end{bmatrix} =
\begin{bmatrix}
    \frac{m^2}{\alpha(m)^2}(a_i^2 + b_i^2 + c_i^2)  & 
    0        \\
    0   & {c_i^2} \\
\end{bmatrix}.
\begin{bmatrix}
\nabla_{r_i} \tilde L (\bm w) dt\\
\nabla_{\tau_i} \tilde L (\bm w) dt\\
\end{bmatrix},
\quad i=1,\ldots,m,
\end{equation}
From these equations, one can see that if $ c_i^2 \ll a_i^2 + b_i^2$ for all $i$ (\ie, 
$\delta_i \ll 0$), then radial motion will dominate. In other words, initializing the first layer with significantly larger values than the second leads to kernel learning. On the other hand, if $ c_i^2 \gg a_i^2 + b_i^2$, then a solution of the gradient flow~\eqref{eq:gradient_flow} will follow the same trajectory as for the reduced gradient flow~\eqref{eq:reduced_gradient_flow}. Also, if $\alpha(m) = o(m)$, the radial component will dominate as $m$ increases.

Figure~\ref{fig:local_trajectories0} shows the trajectories corresponding to different values of $\delta_i$ for each neuron, with $\alpha(m)=m$.
The extreme cases of $\delta = -\infty$ and $\delta = +\infty$ correspond to the ``kernel'' and ``adaptive'' regimes, respectively. Note that as~$\delta$ increases, the neurons cluster at sample points, as explained in our analysis of the reduced dynamics in Section~\ref{sec:reduced_dynamics}, and in accordance to \cite{maennel2018gradient}.

\begin{figure}
    \centering
    \minipage{0.33\textwidth}
    \includegraphics[width=\linewidth]{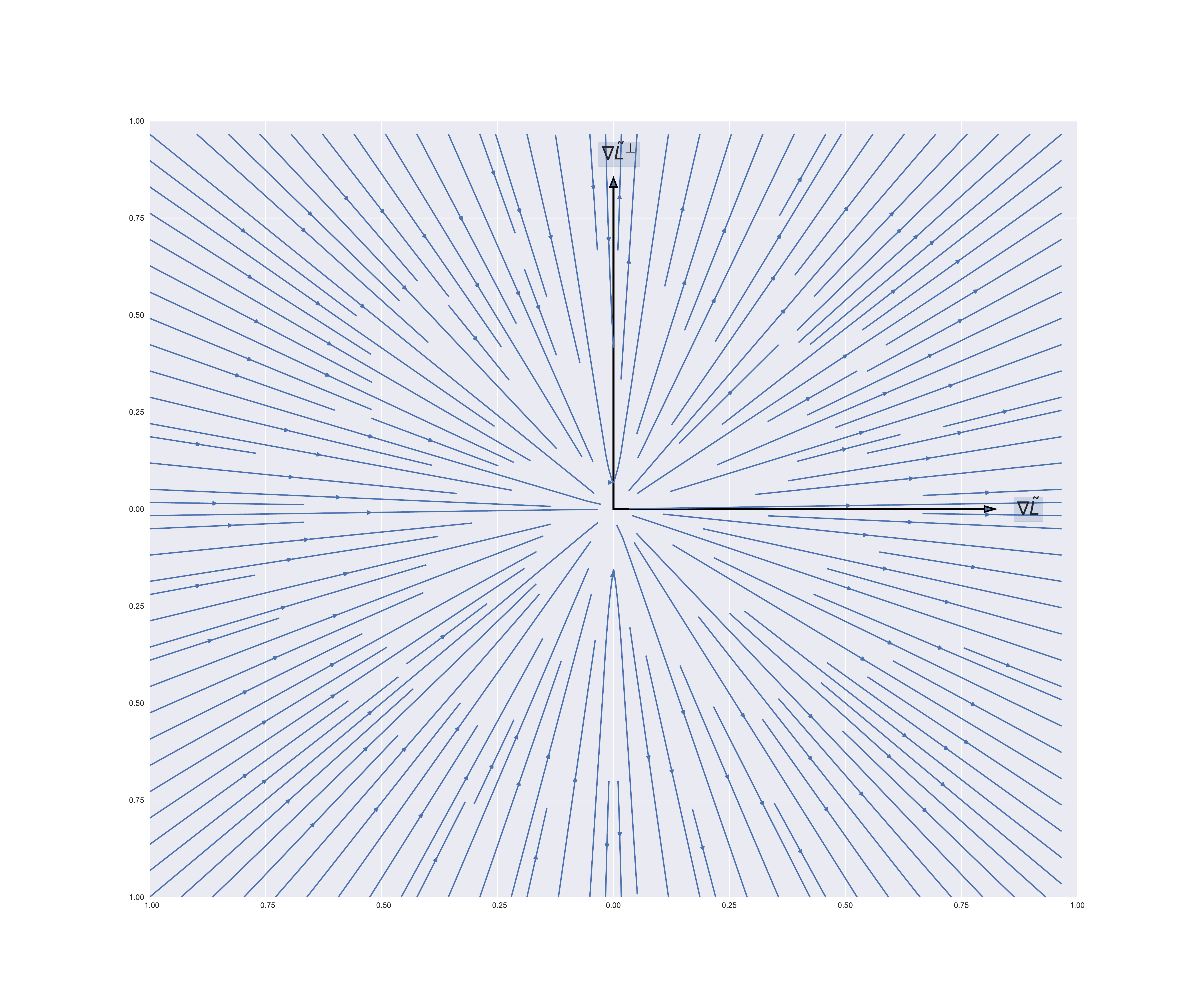}
    \endminipage\hfill
    \minipage{0.33\textwidth}
    \includegraphics[width=\linewidth]{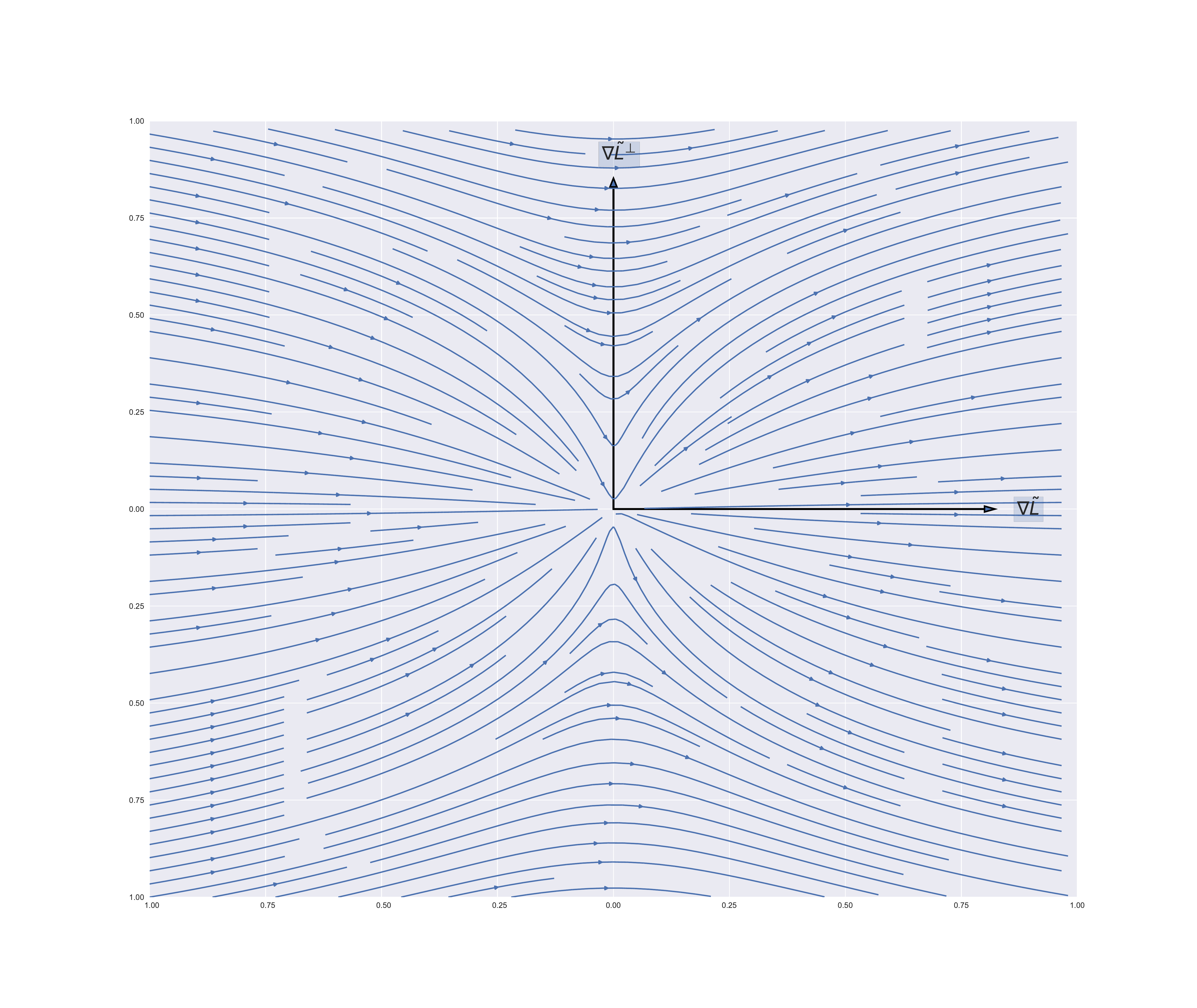}
    \endminipage\hfill
    \minipage{0.33\textwidth}
    \includegraphics[width=\linewidth]{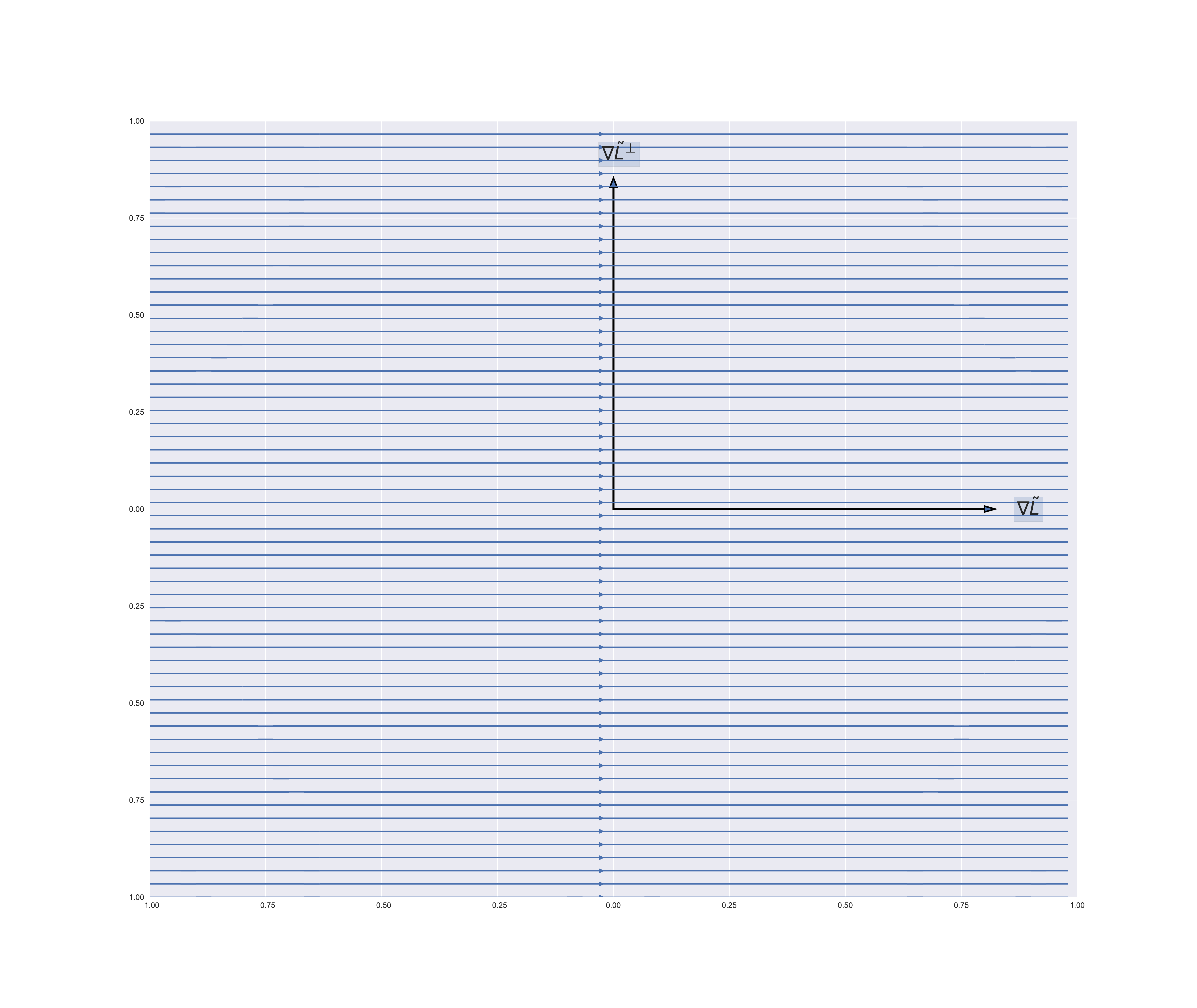}
    \endminipage\\
    \vspace{-10pt}
    \parbox{.32\textwidth}{\centering $\delta = -100$}
    \parbox{.32\textwidth}{\centering $\delta = 0$}
    \parbox{.32\textwidth}{\centering $\delta = 100$}
    \caption{The value of $\delta$ interpolates between different kinds of local trajectories of neurons. The plots are in the coordinate frame $(\nabla \tilde{L}, \nabla \tilde{L}^\bot)$. Left: the neurons move radially towards and away from the origin. Middle: the  trajectories containing both radial and tangential components. Right: the trajectories are parallel to the gradient $\nabla \tilde{L}$.}
    \label{fig:local_trajectories0}
\end{figure}

\subsection{Kernel Dynamics}\label{sec:kernel_dynamics}

We now consider the dynamics in the special case where $\delta \ll 0$, 
and we consider $m \rightarrow \infty$.  To obtain the kernel 
regime in this case, it is sufficient to consider a normalisation $\alpha(m)=o(m)$. 
In particular, when $\alpha(m)=1$, as shown in the previous section, the parameters $\bm a$ and $\bm b$ remain mostly fixed and the parameters $\bm c$ change throughout training, corresponding to the so-called random-features (RF) kernel of Rahimi and Recht \cite{rahimi2008random}. 

In the limit case where $\bm a$ and $\bm b$ are completely fixed to their initial values, if we choose $\bm c$ close to the zero vector, then the least squares problem \eqref{eq:leastsquares} solved using gradient flow, is equivalent to the minimal-norm constraint problem solution: 
\begin{equation}\label{eq:spline_leastsquares}
\begin{aligned}
    &\text{minimize } \|\bm c\|^2\\
    &\text{subject to } f_{\bm z}(x_i) = y_i, \qquad i=1,\ldots,s.
\end{aligned}
\end{equation}
Given an initial distribution $\mu_0$ over the domain $\mathcal{D}_a \times \mathcal{D}_b$ of parameters $a$ and $b$, the random-feature (RF) kernel associated with (\ref{eq:spline_leastsquares}) is given by
\begin{equation}
K_{\mathrm{RF}}(x,x') = \int_{\mathcal{D}_a \times \mathcal{D}_b}  [x a - b]_+ \cdot [x'a - b]_+ \mu_0(\dd a, \dd b)~.    
\end{equation}
The solution of~\eqref{eq:spline_leastsquares} can now be written in terms of this RF kernel using the representer theorem: 
\begin{equation}
    \tilde f_{\bm z}(x) = \sum_{j=1}^s \alpha_j K_{\mathrm{RF}}(x_j, x) ~,
\end{equation}
where $\alpha$ is a vector of minimal RKHS norm that fulfills the interpolation constraints. 
Under appropriate assumptions, the solution to~\eqref{eq:spline_leastsquares} is a \emph{cubic spline}, as shown by the following Theorem proved in Appendix~\ref{app:spline}.  



\begin{theorem} 
\label{prop:splines}
Assume the measure $\mu_0$ has finite second moment $\sigma_{\mu_0}^2:=\mathbb{E}_{(a,b) \sim \mu_0}(a^2+b^2) < \infty$.
Let $\mu_0(a,b) = q(a) \mu_a(b)$ be the decomposition in terms of marginal and conditional, and assume $\mu_a$ is bounded for each $a$. Define $\nu(u)=\int |a| \mu_a(ab) \dd q(a)$. 
Then the solution~\eqref{eq:spline_leastsquares} solves
\begin{equation}
\label{blok}
\min_f  \quad  \| f \|_{\mathrm{RF}}^2 := 
\int_{\Omega} \frac{|f''(u)|^2}{ \nu(u)} \dd u \quad ~s.t.\quad f(x_i) = y_i~,i=1\dots s~,
\end{equation}
where $\Omega:=\text{supp}(\nu)$.  
Moreover, if $\mu_0$ is such that $\mu_0(a,b)= q(a) \mathbf{1}( b \in I_a)$, where $I_a \subset \RR$ is an arbitrary interval, 
then ~\eqref{eq:spline_leastsquares} will be a cubic spline.
\end{theorem}

Notice that the assumptions on $\mu_0$ to obtain an exact cubic spline kernel impose that if $A,B$ is a random vector distributed according to $\mu_0$, then $B|A$ is uniform over an arbitrary interval $I_A$ that can depend upon $A$. The proof illustrates that one may generalise the interval $I_A$ by any countable union of intervals. 
In particular, independent uniform initialization yields cubic splines,
but radial distributions, such as $A,B$ being jointly Gaussian, don't. 
In that case, the kernel is  expressed in terms of trigonometric functions \cite{chizat2018note,bach2017breaking}. In our setup this becomes
\begin{eqnarray}
\label{eq:RFkernel_radial}
   K_{\mathrm{RF}}^r(x,x')&=& \frac{C}{4\pi}\Big\{(\pi - \arctan(x')+\arctan(x))(xx'+1) +  \\ && +\left( \frac{x'}{1+(x')^2} - \frac{x}{1+x^2} \right)(xx'-1)   + (x+x')\left(\frac{(x')^2}{1+(x')^2} - \frac{x^2}{1+x^2} \right)\Big\}~,\nonumber
\end{eqnarray}
as verified in Appendix \ref{app:spline}.


The normalization choice $\alpha(m)=\sqrt{m}$ results in a different kernel, the Neural Tangent Kernel of \cite{jacot2018neural}. In this scaling regime, the linearization of the model becomes 
\begin{eqnarray}
K_{\mathrm{NTK}}(x,x')&=& \nabla_{\bm z} f_{\bm z}(x)^\top \nabla_{\bm z} f_{\bm z}(x') \nonumber \\
&=& \frac{1}{m} \sum_{i=1}^m [a_i x - b_i]_+ [a_i x' - b_i]_+ + \frac{xx'+1}{m}\sum_{i=1}^m c_i^2 \mathbf{1}(a_i x - b_i > 0)\cdot \mathbf{1}(a_i x' - b_i > 0) \nonumber \\
&\stackrel{m\to \infty}\to& K_{\mathrm{RF}}(x,x') + (xx' + 1)\cdot \mathbb{E}(c^2)\int \mathbf{1}(a x > b)\cdot \mathbf{1}(a x' > b ) \mu_0(\dd a,\dd b)~.  
\end{eqnarray}
The extra term in $K_{\mathrm{NTK}}$ captures lower-order regularity as shown in the following corollary
\begin{corollary}
\label{coro:ntk}
Under the same assumptions as in Theorem \ref{prop:splines}, we have
\begin{equation}\label{eq:ntk_kernel}
\| f \|_{\mathrm{NTK}}^2 := \inf \left\{ 
\int_{\Omega} \left(\frac{|f_1''(u)|^2}{ \nu(u)} + \mathbb{E}(c^2) \left(\frac{|f_2'(u)|^2}{\nu(u)} + \frac{| u f_3'(u) - f_3(u)|^2}{u^2 \nu(u)}\right)\right) \dd u ; \, f = f_1 + f_2 + f_3
\right\}\,.
\end{equation}
Additionally, if $\mu_0$ is such that $\mu_0(a,b)= q(a) \mathbf{1}( b \in I_a)$, where $I_a \subset \RR$ is an arbitrary interval, 
then $K_{\mathrm{NTK}}$ is also piece-wise cubic. 
\end{corollary}

We remark that machine learning packages such as PyTorch use a uniform distribution for linear layer parameter initialization by default.
We verify that indeed, solutions to \eqref{eq:leastsquares} converge to cubic splines as $m$ grows in Figure~\ref{fig:spline}. We also point out that in Kernel Learning, early termination of gradient flow acts as a regularizer favoring smooth, non-interpolatory solutions (see \cite{jacot2018neural}).

\begin{figure}
    \centering
    \minipage{0.5\textwidth}
    \includegraphics[width=\linewidth]{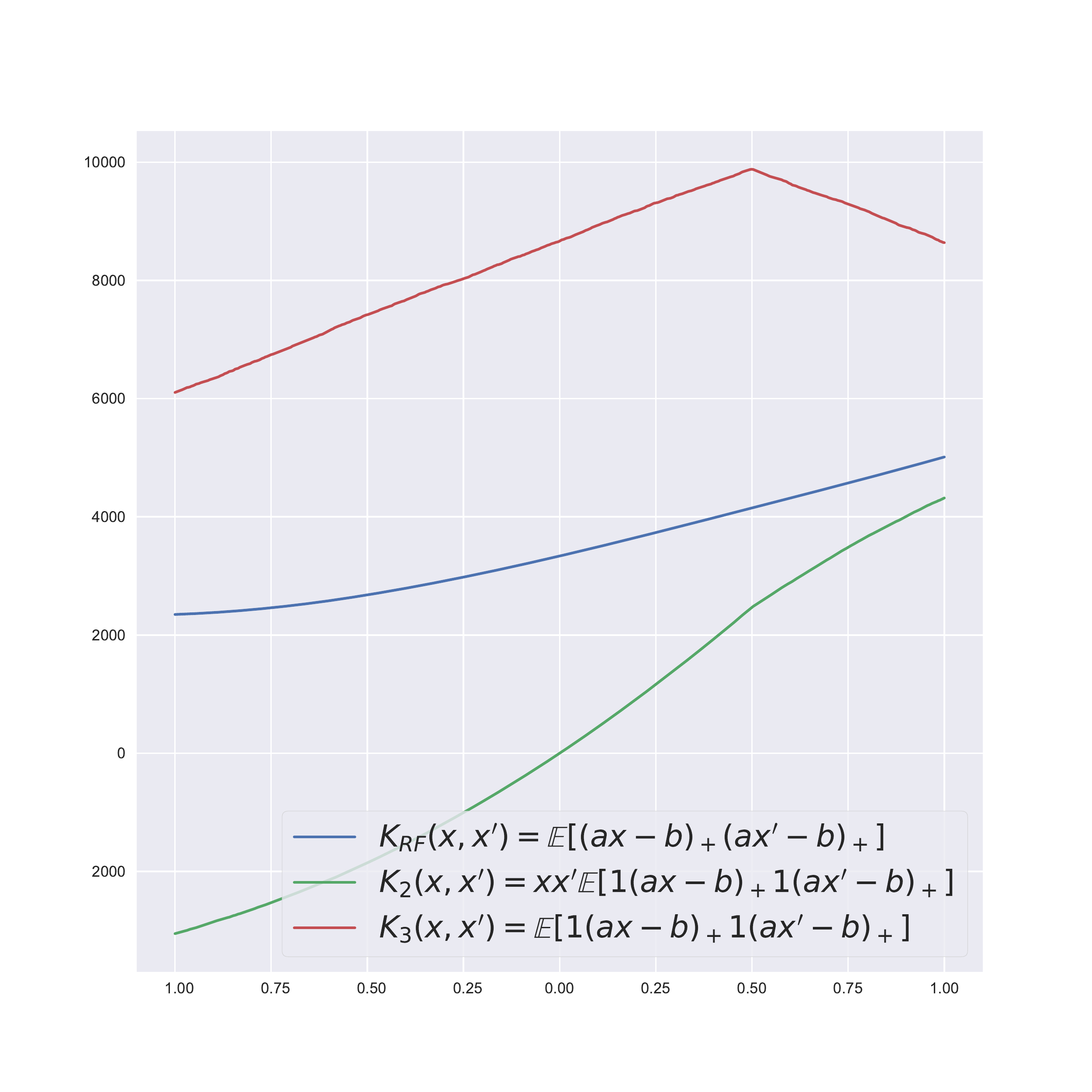}
    \endminipage
    \minipage{0.5\textwidth}
    \includegraphics[width=\linewidth]{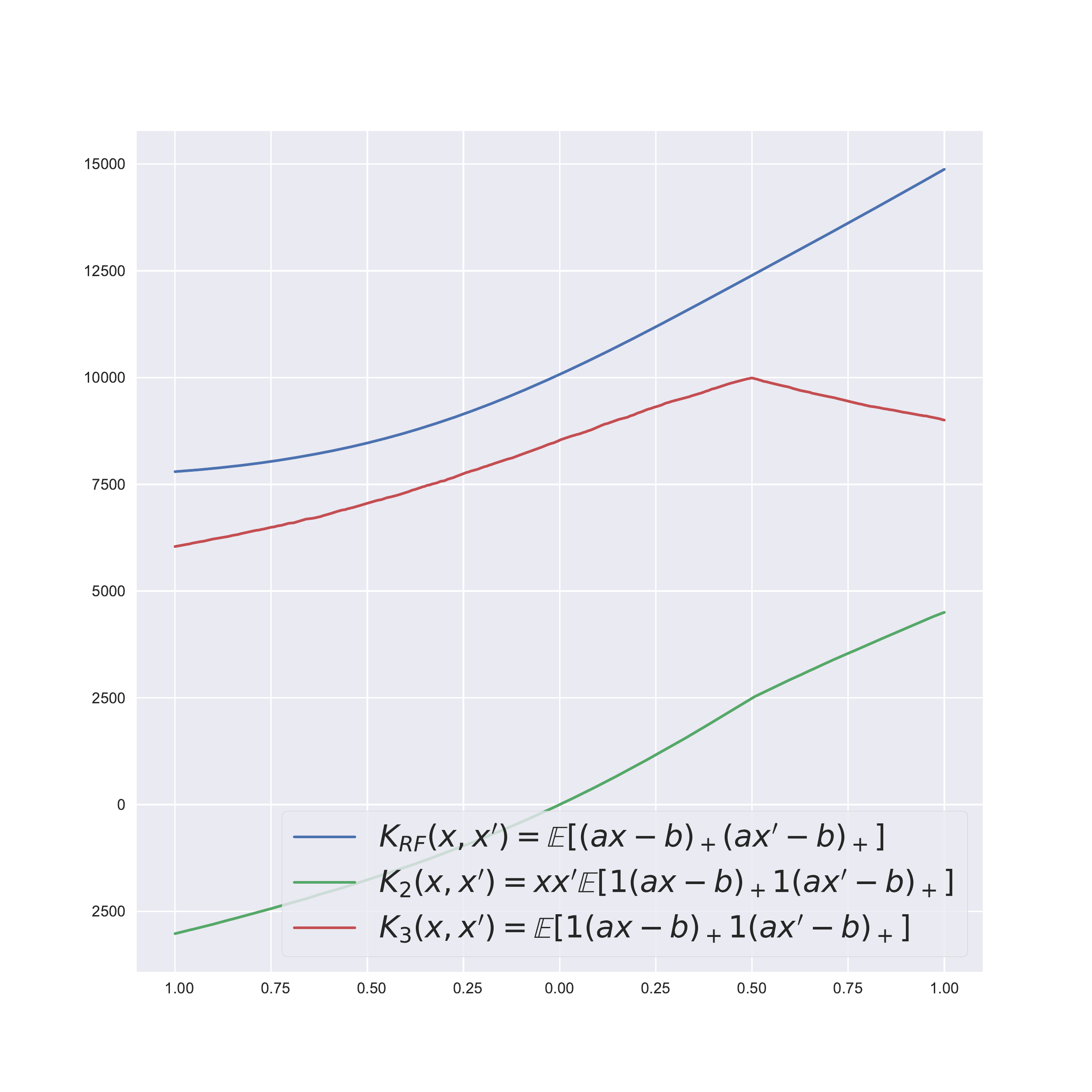}
    \endminipage\hfill
    \parbox{.49\textwidth}{\centering Uniform Initialization}
    \parbox{.49\textwidth}{\centering Gaussian Initialization}
    \caption{1D slice of the plot of the different terms of the tangent kernel \eqref{eq:ntk_kernel} for $x=0.5$ with a network consisting of $m = 25000$ neurons. The parameters $\bm a$ and $\bm$ are initialized from Uniform and Gaussian distributions with $\frac{1}{\sqrt{m}}$ scaling.}
    \label{fig:kernels}
\end{figure}

The analysis and comparison of these kernels has recently been addressed in \cite{bietti2019inductive, ghorbani2019linearized} in the general high-dimensional setting, by describing its spectral decay in terms of spherical harmonics. Our results complement them in the particular one-dimensional setting thanks to the explicit functional form of the resulting RKHS norms. 
Additionally, Savarese et al. \cite{savarese2019infinite} study the functional form of the minimisation in the variation norm, leading to a penality of the form $\int |f''(u)| \dd u$. 
We have instead $L^2$ norms (RKHS) in the kernel regime. The $L^2$ norms do not provide any adaptivity as opposed to the $L^1$ norm \cite{bach2017breaking}. An interesting question is to precisely describe the transition between these two regimes as a function of the initialisation.

\section{Numerical Experiments}
For the numerical experiments, we follow gradient descent using the parameterization \eqref{eq:leastsquares} with $\alpha(m) = \sqrt{m}$, appropriately scaling the weights $\bm a, \bm b, \bm c$ to achieve different dynamical behavior. In addition to those presented here, Section~\ref{sec:additional_experiments} presents additional numerical experiments

\paragraph{Cubic Splines.}
We show in Figure~\ref{fig:spline} that when $-\delta \ll r^2$ (\ie we are in the kernel regime), growing the number of neurons grows causes the network function $f_{\bm z}$ converges to a cubic spline. For this experiment, we fit 10 points sampled from a square wave, and train only the parameters $c$ (\ie $\delta_i = \infty$). 

\begin{figure}
    \centering
    \minipage{0.33\textwidth}
    \includegraphics[width=\linewidth]{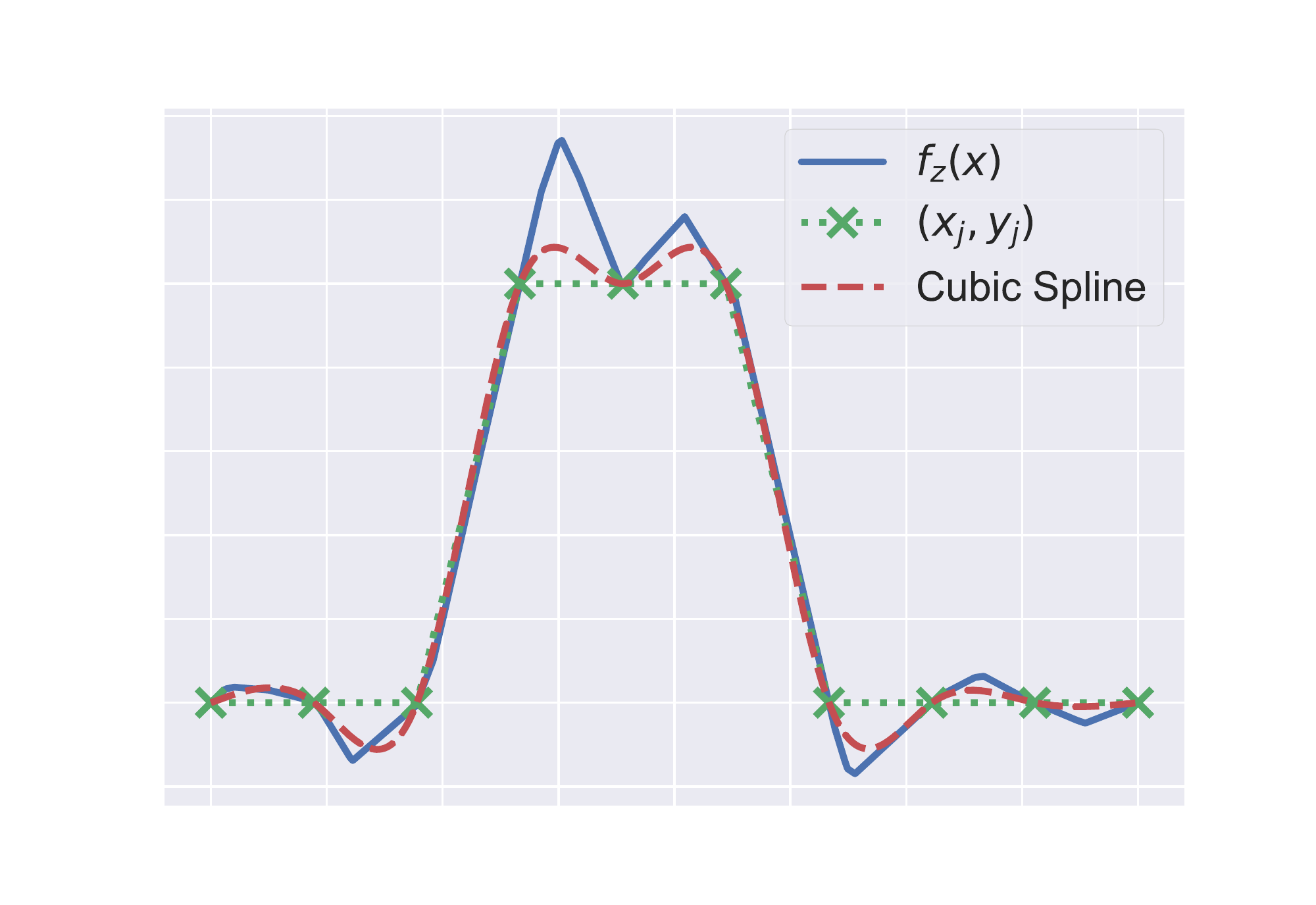}
    \endminipage
    \minipage{0.33\textwidth}
    \includegraphics[width=\linewidth]{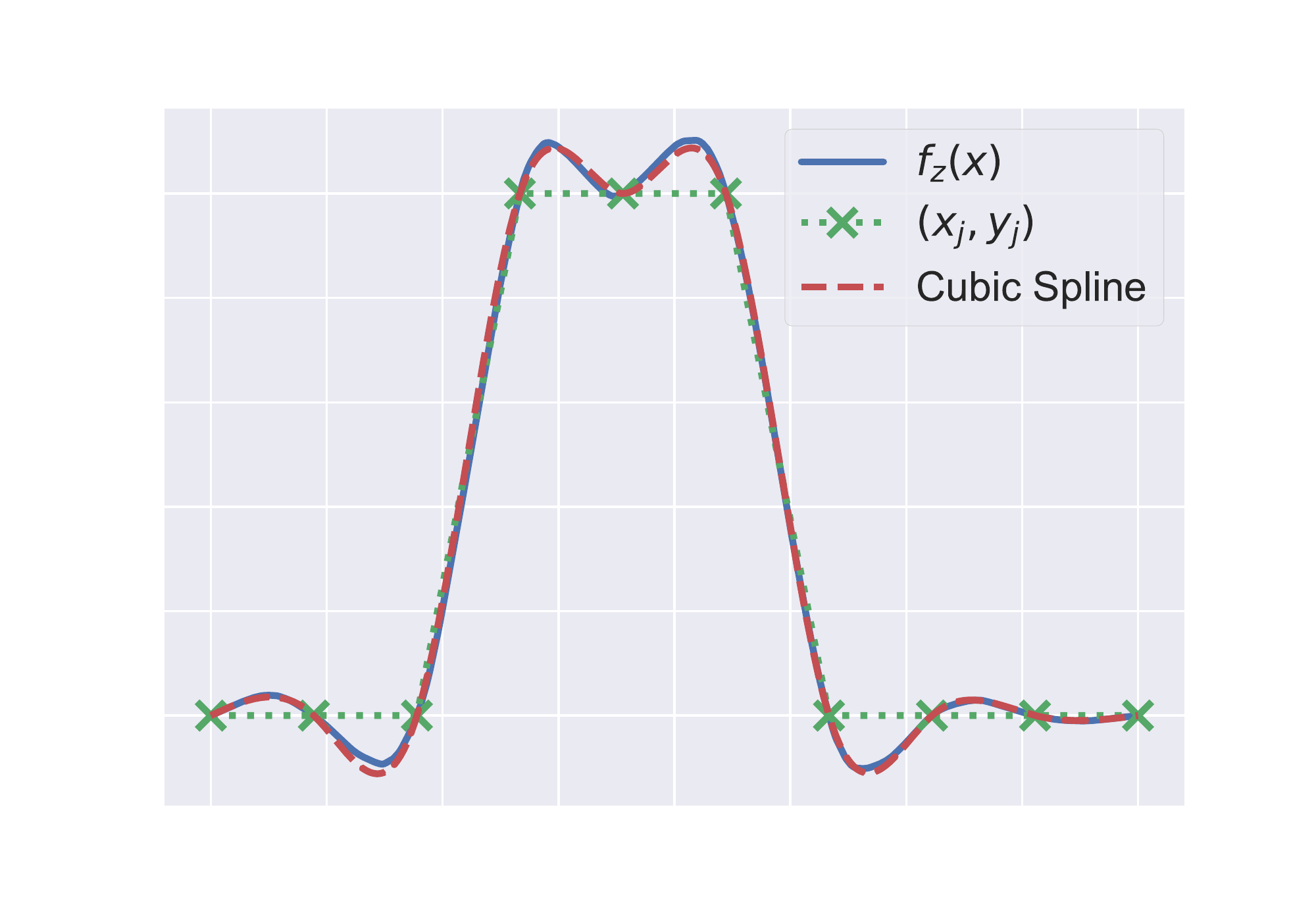}
    \endminipage\hfill
    \minipage{0.33\textwidth}
    \includegraphics[width=\linewidth]{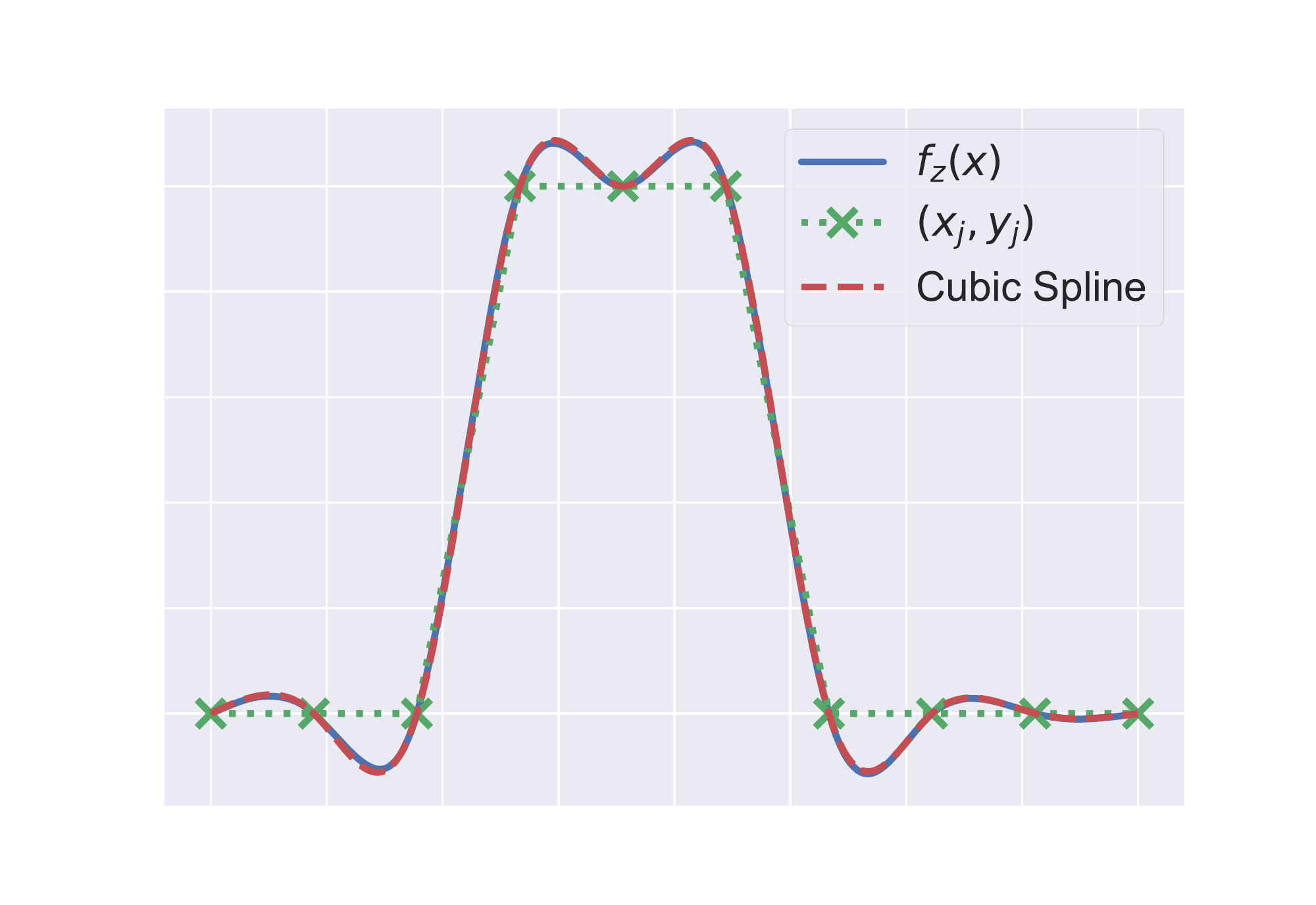}
    \endminipage\hfill\\
    \parbox{.32\textwidth}{\centering $m = 10^2$}
    \parbox{.32\textwidth}{\centering $m = 10^3$}
    \parbox{.32\textwidth}{\centering $m = 10^4$}
    \caption{A cubic spline with vanishing second derivative at its endpoints (blue line) is approximated by a neural network  ($\delta = -100$) while varying the number $m$ of neurons.}
    \label{fig:spline}
\end{figure}

\paragraph{Network Dynamics as a Function of $\delta$.}
We show in Figure~\ref{fig:vary_delta} that as we vary $\delta$, the network function goes from being smooth and non-adaptive in the kernel regime ($\delta = -\infty$, \ie training only the parameter $\bm c$) to very adaptive ($\delta = \infty$, \ie training only the parameters $\bm a, \bm b$). Note how as $\delta$ gets large, clusters of knots emerge at the sample positions (collinear points in the $uv$ diagrams).
\begin{figure}
    \centering
    \minipage{0.2\textwidth}
    \includegraphics[width=\linewidth]{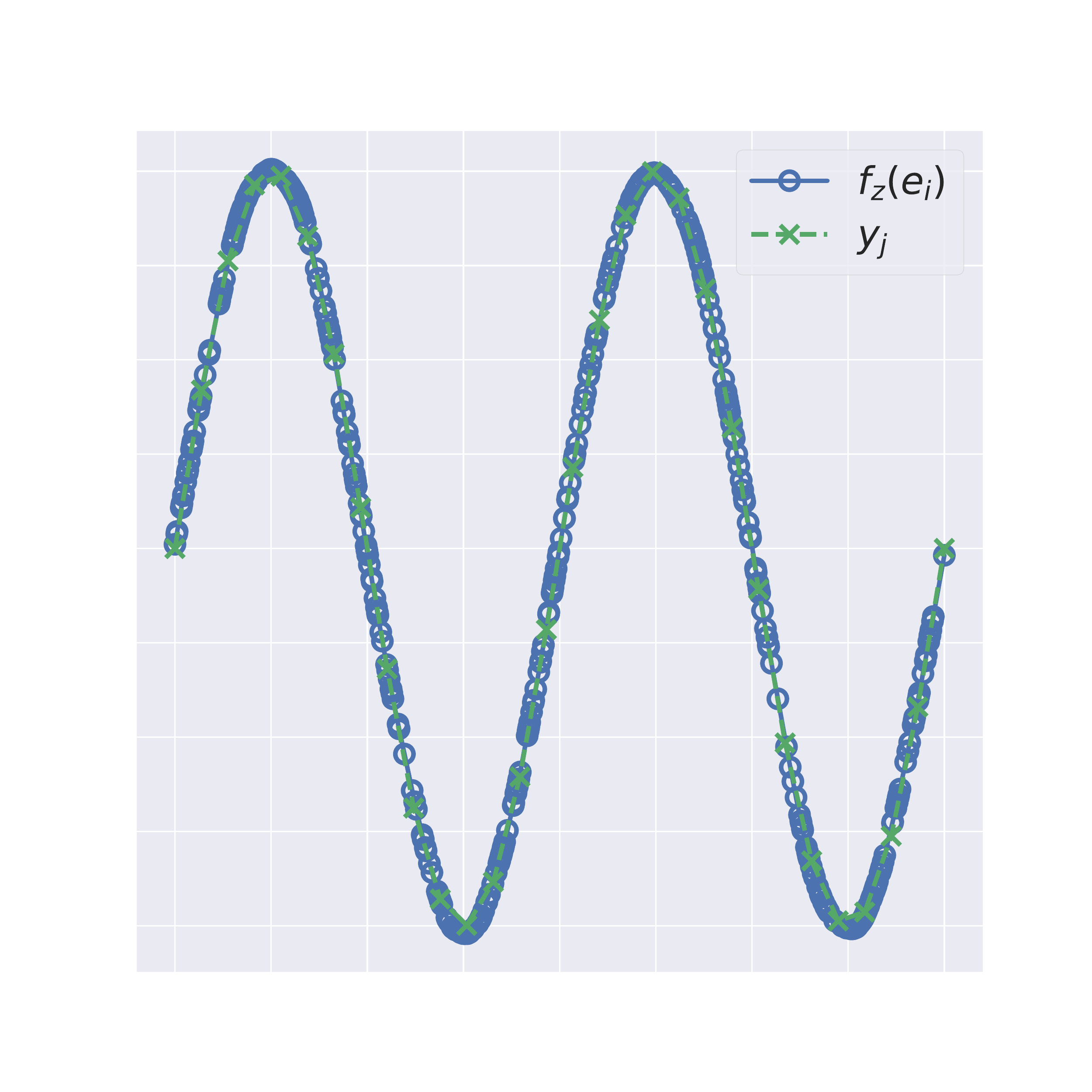}
    \endminipage
    \minipage{0.2\textwidth}
    \includegraphics[width=\linewidth]{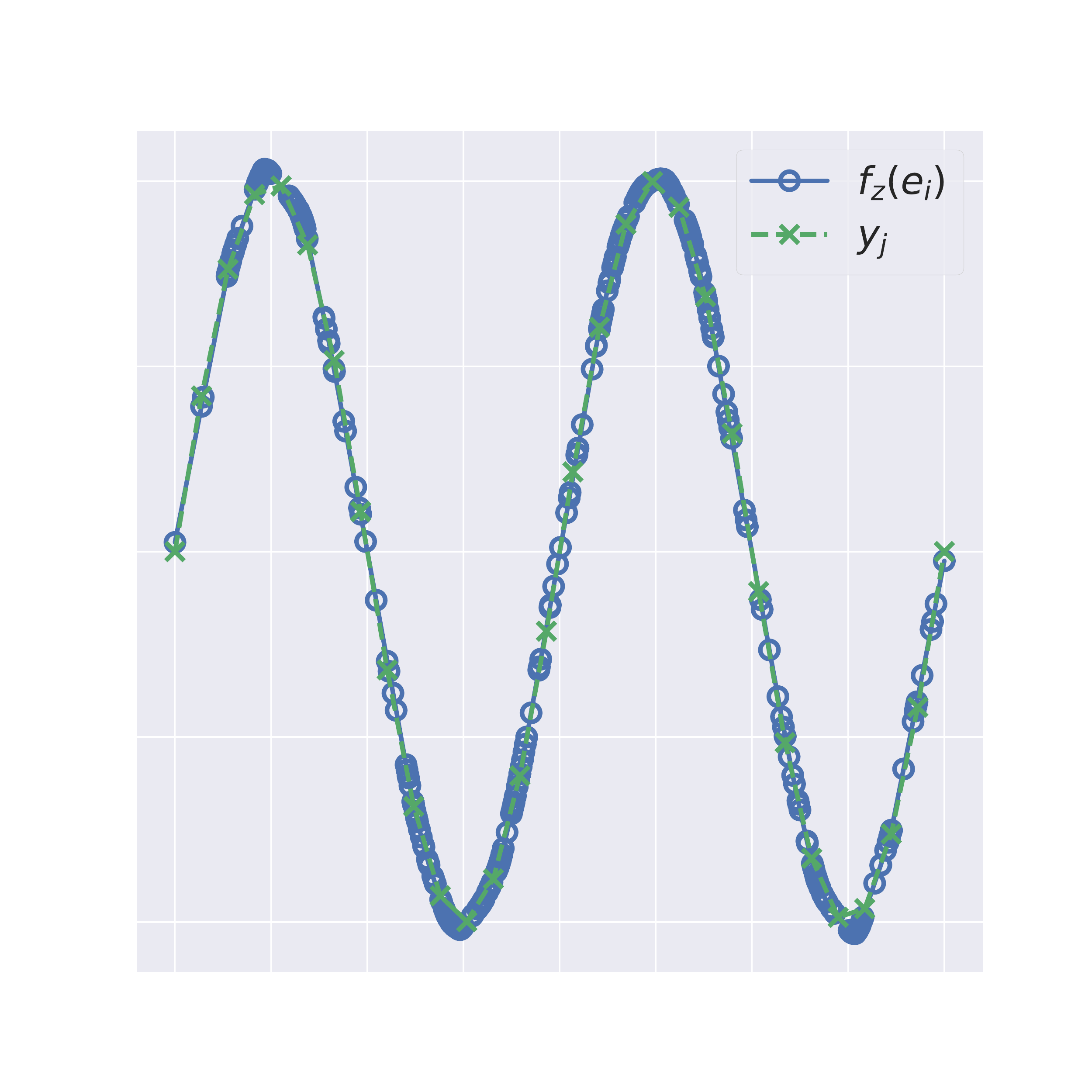}
    \endminipage
    \minipage{0.2\textwidth}
    \includegraphics[width=\linewidth]{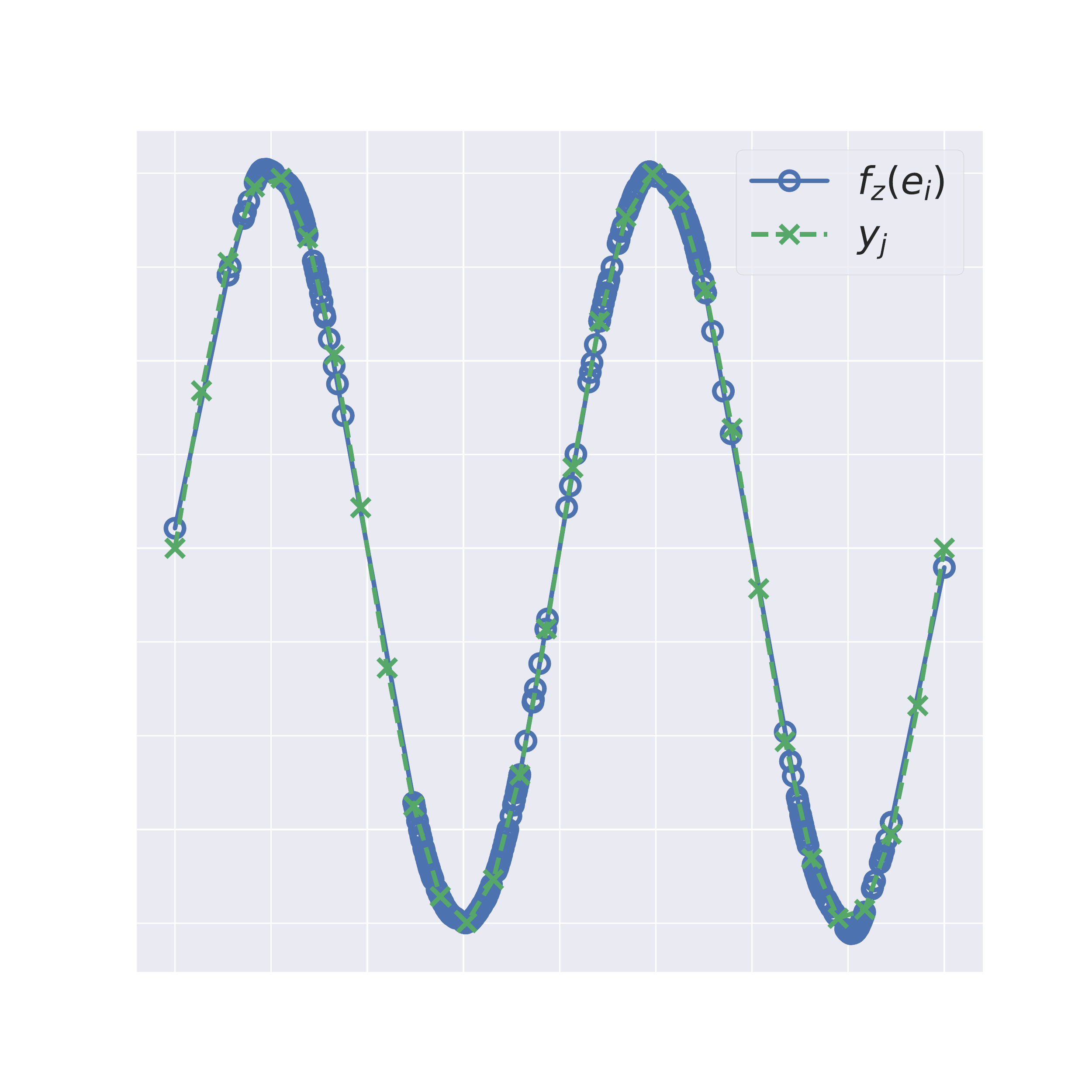}
    \endminipage
    \minipage{0.2\textwidth}
    \includegraphics[width=\linewidth]{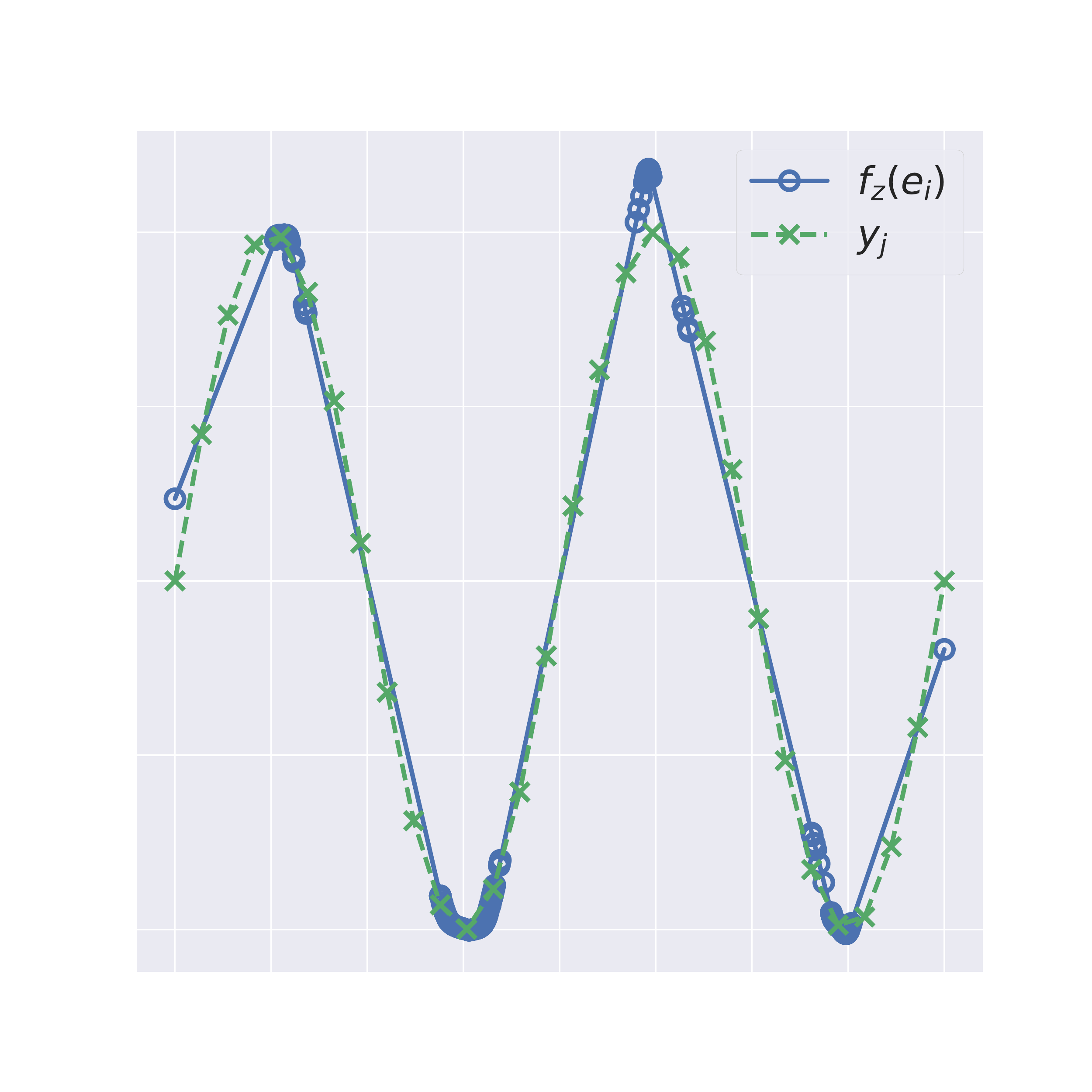}
    \endminipage
    \minipage{0.2\textwidth}
    \includegraphics[width=\linewidth]{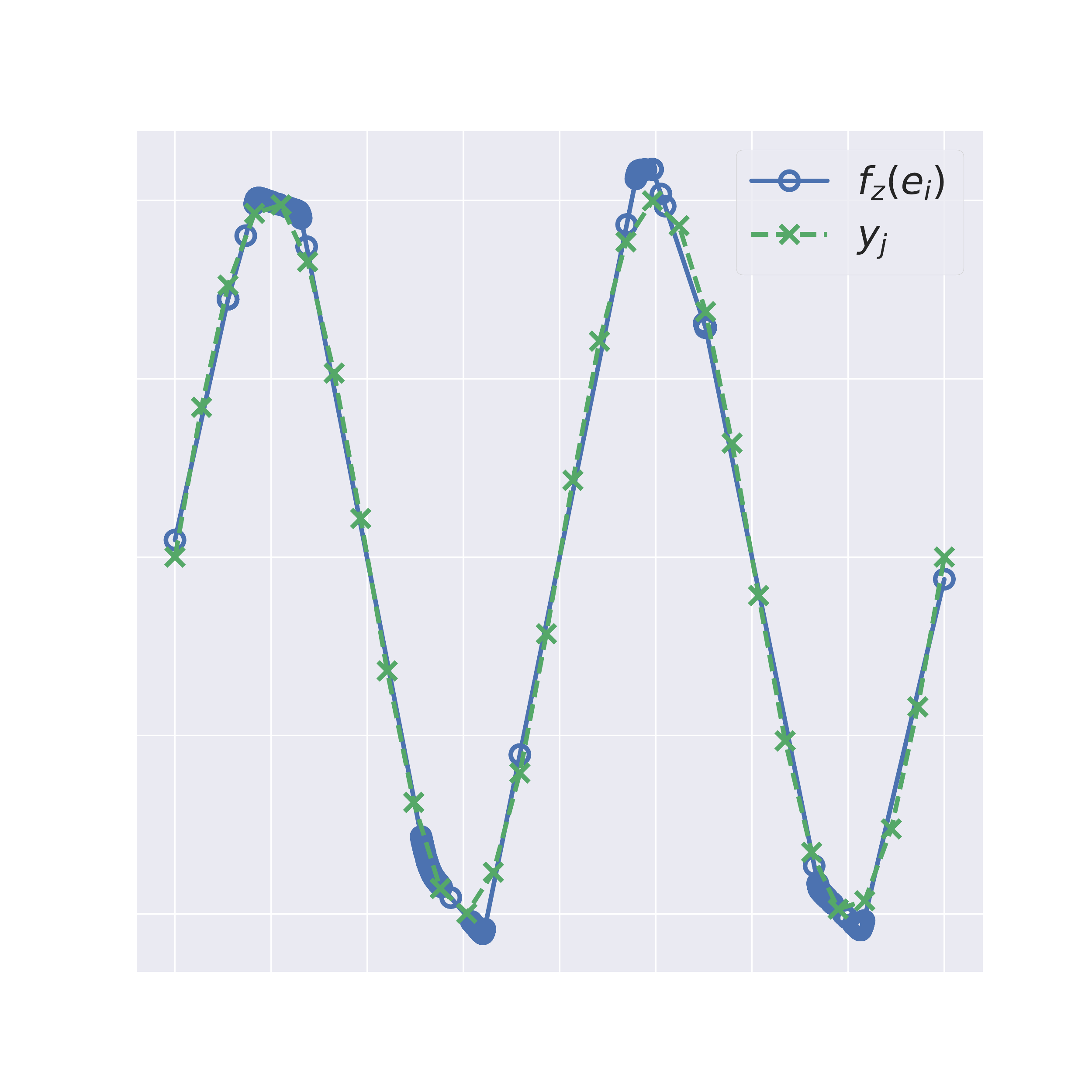}
    \endminipage
    \vspace{-8pt}
    \minipage{0.2\textwidth}
    \includegraphics[width=\linewidth]{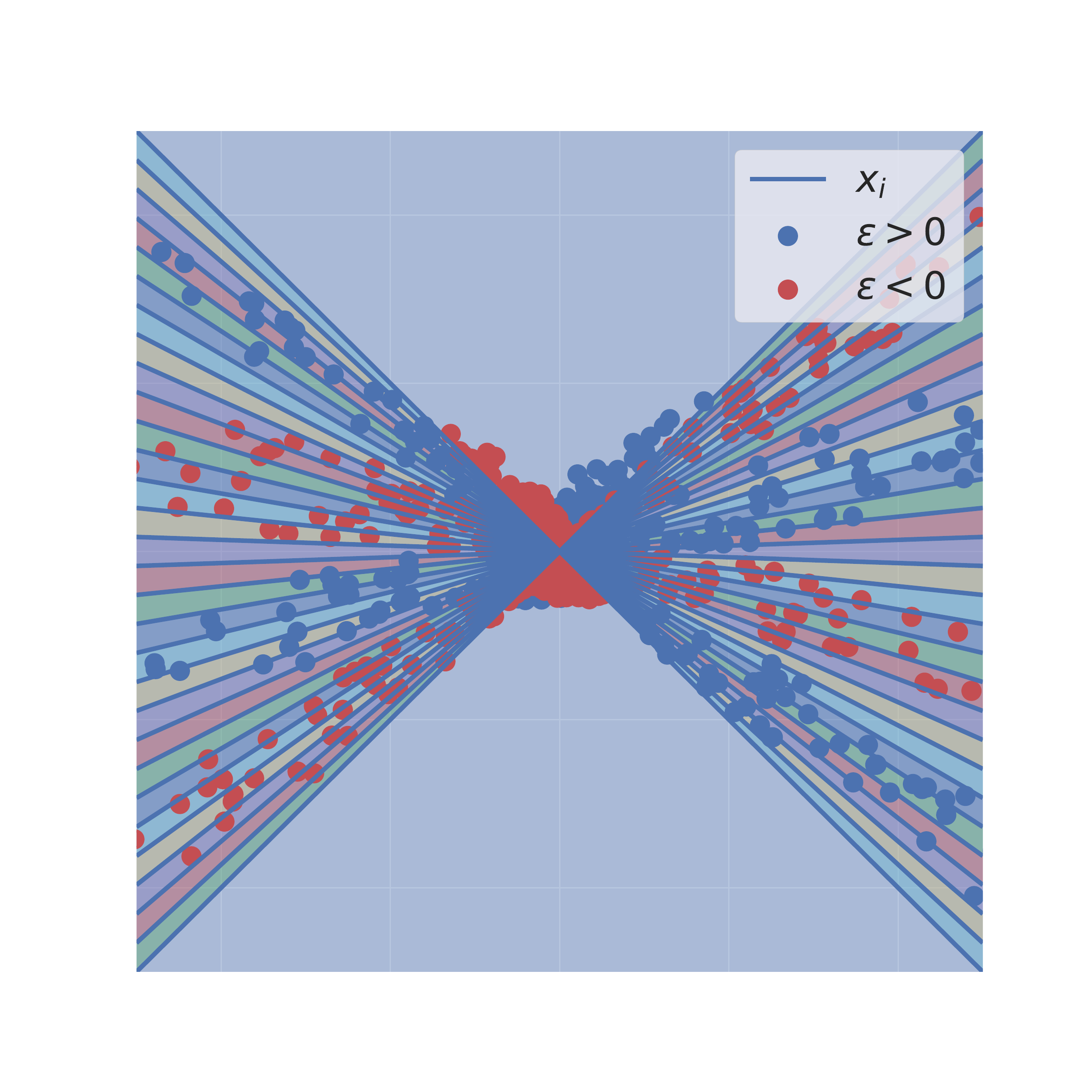}
    \endminipage
    \minipage{0.2\textwidth}
    \includegraphics[width=\linewidth]{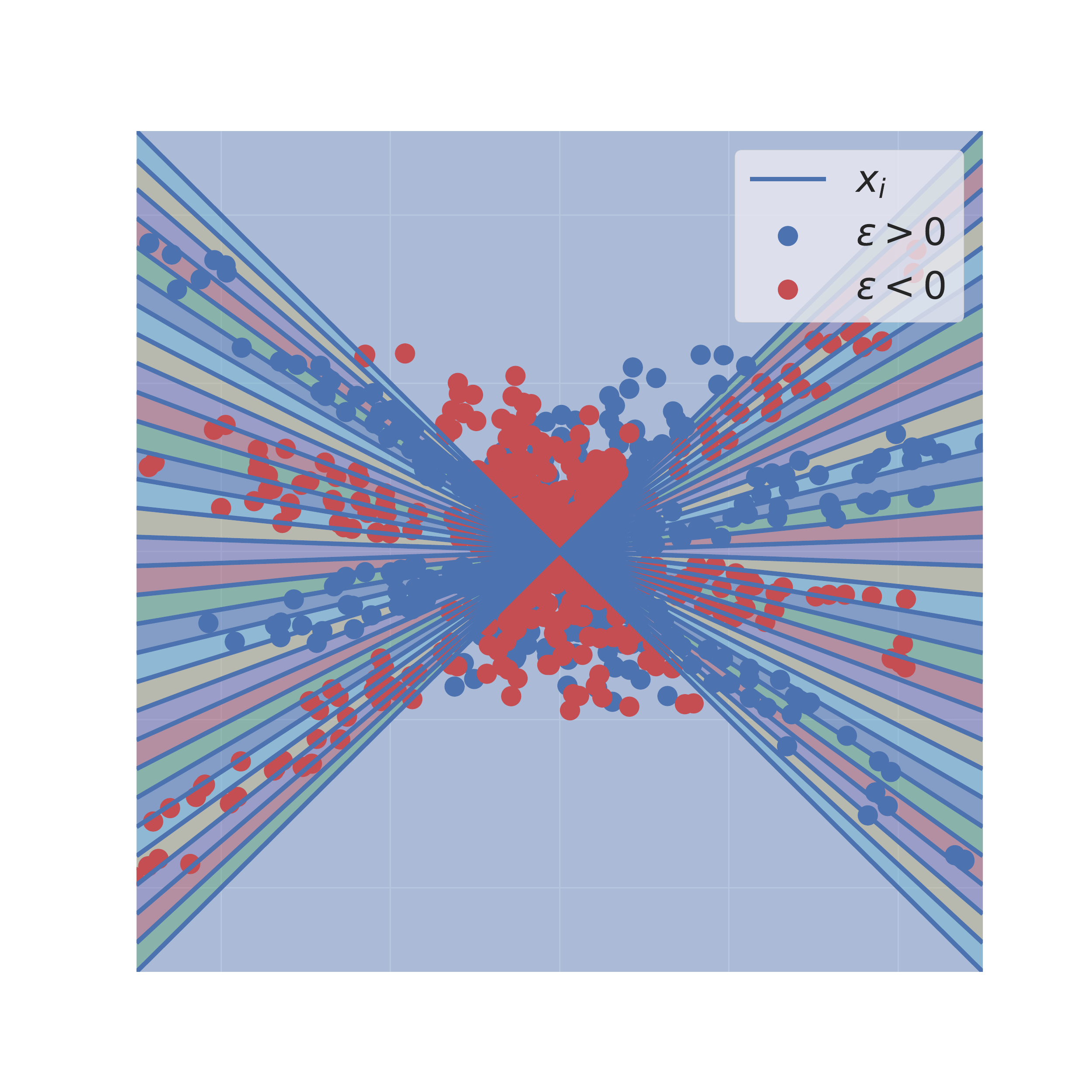}
    \endminipage
    \minipage{0.2\textwidth}
    \includegraphics[width=\linewidth]{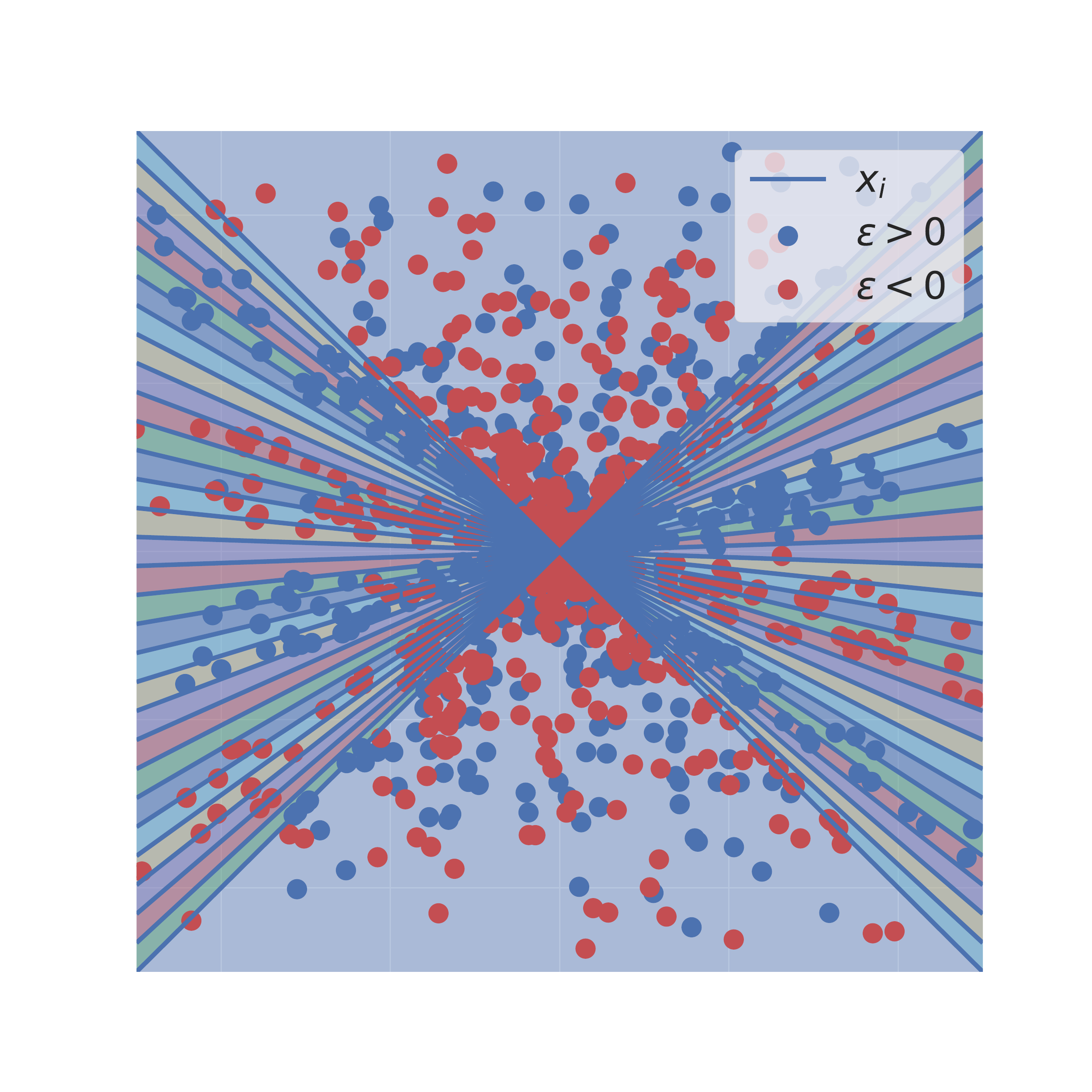}
    \endminipage
    \minipage{0.2\textwidth}
    \includegraphics[width=\linewidth]{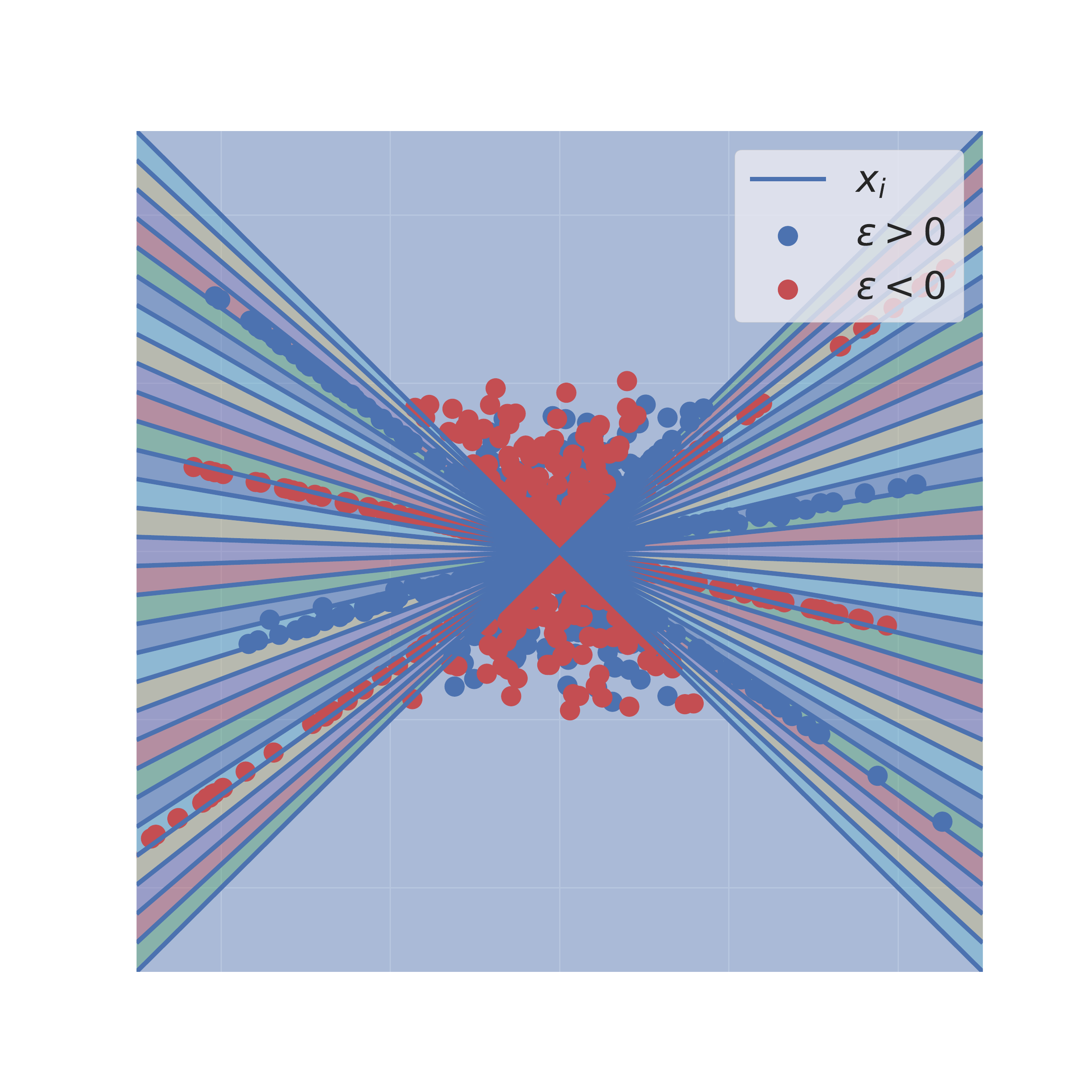}
    \endminipage
    \minipage{0.2\textwidth}
    \includegraphics[width=\linewidth]{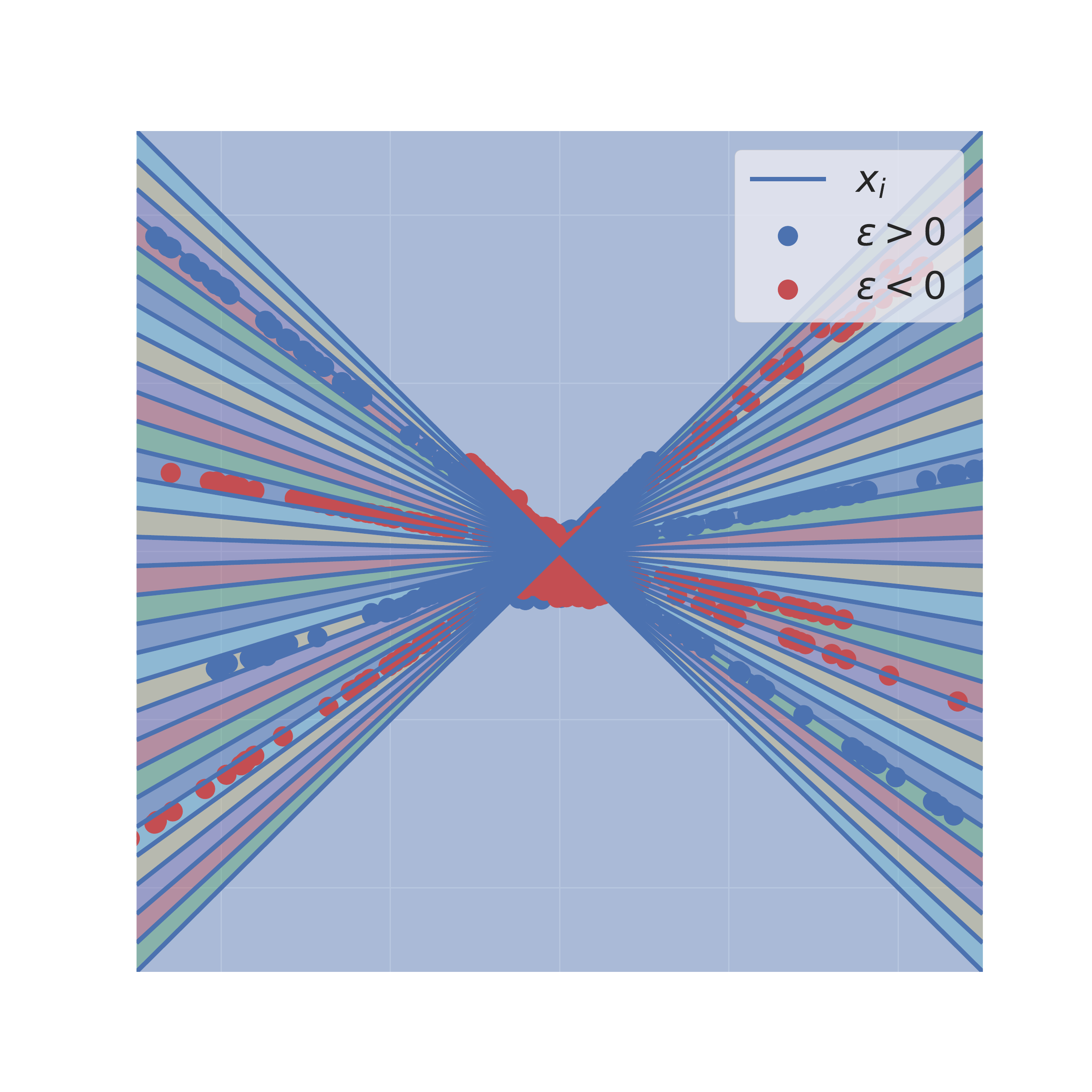}
    \endminipage\\
    \vspace{-5pt}
    \parbox{.19\textwidth}{\centering $\delta = -\infty$}
    \parbox{.19\textwidth}{\centering $\delta = -1$}
    \parbox{.19\textwidth}{\centering $\delta = 0$}
    \parbox{.19\textwidth}{\centering $\delta = 1$}
    \parbox{.19\textwidth}{\centering $\delta = \infty$}
    \caption{Comparison of fitting the network function to  a sinusoid as $\delta$ varies (10000 epochs).\vspace{-20pt}}
    \label{fig:vary_delta}
\end{figure}
\vspace{-0.1cm}
%
%


\section{Concluding Remarks}

We have studied the implicit bias of gradient 
descent in the approximation of univariate functions with single-hidden layer ReLU networks. 
Despite being an extremely simplified learning setup, it provides a clear illustration 
that such implicit bias can be drastically different depending on how the neural architecture is 
parameterized, normalized, or even initialized. Building up on recent theoretical work that studies 
neural networks in the overparameterized regime, we show how the model can behave 
either as a `classic' cubic spline interpolation kernel, or as an adaptive interpolation method, 
where neurons concentrate on sample points where the approximation most needs them. 
Moreover, in the one-dimensional case, we complement existing works \cite{sirignano2018mean} to reveal a clear transition between these two extreme training regimes, which roughly correspond to $W^{1,2}$ and $W^{2,2}$ Sobolev spaces respectively.

%
Although in our univariate setup there is no clear advantage of one functional space over the other, our full description of the dynamics may prove useful in the high-dimensional regime, where the curse of dimensionality hurts Hilbert spaces defined by kernels \cite{bach2017breaking}. 
We believe that the analysis  of the PDE resulting from the mean-field regime (where adaptation occurs) in the low-dimensional  setting will be useful to embark in the analysis of the high-dimensional counterpart. We note however that naively
extending our analysis to high-dimensions would result in an exponential increase in the number of regions that define our piecewise  linear flow, thus we anticipate new tools might be needed to carry it over. Moreover, the interpretation of ReLU features in terms of Green's functions (as first pointed out in \cite{sirignano2018mean}) does not directly carry over to higher dimensions. 
Lastly, another important limitation of the mean-field analysis is that it cannot be easily adapted 
to deeper neural network architectures, since the neurons within the network are no longer exchangeable as in 
the many-particle system described above.

\paragraph{Acknowledgements:}

This work was partially supported by the Alfred P. Sloan Foundation, NSF RI-1816753, NSF CAREER CIF 1845360, Samsung Electronics,
the NSF CAREER award 1652515,
the NSF grant IIS-1320635,
the NSF grant DMS-1436591,
the NSF grant DMS-1835712,
the SNSF grant P2TIP2\_175859,
the Moore-Sloan Data Science Environment, 
the DARPA D3M program, 
NVIDIA, 
Labex DigiCosme,
DOA W911NF-17-1-0438,
a gift from Adobe Research, 
and a gift from nTopology. 
Any opinions, findings, and conclusions or recommendations expressed in this material are those of the authors and do not necessarily reflect the views of DARPA.

\bibliography{main.bib}{}
\bibliographystyle{plain}
\newpage
\appendix
{\Large \bf Appendix}

\section{Spline Kernels}
\label{app:spline}

\subsection{Proof of Theorem \ref{prop:splines}} 

Consider the factorization of $\mu_0$ in terms of the marginal with respect to $a$, $q(a)$ and its conditional $\mu_a(b)$: $\mu_0(a,b) = q(a) \mu_a(b)$.
Observe that 
\begin{eqnarray}
\label{bli}
K_{\mathrm{RF}}(x,x') &=& \int_{\mathcal{D}_a \times \mathcal{D}_b}  [ax - b]_+ \cdot [ax' - b]_+ \dd \mu_0( a, b) \\
&=& \int_{\mathcal{D}_a} K_a(x,x') \mathrm{d}q(a) ~, \text{with }\nonumber
\end{eqnarray}
$$K_a(x,x') := \int_{I_a} [a x - b]_+ [a x' - b]_+ \dd \mu_a(b) ~.$$
We will argue that each $K_a$ defines a RKHS $\mathcal{H}_a$ 
of twice-continuously differentiable functions with appropriate boundary condition, 
and with norm $\| f \|_a$ corresponding to a weighted squared second derivative. 
We will then conclude by leveraging the fact that averages 
of positive definite kernels also define an RKHS \cite{aronszajn1950theory,hotz2012representation}.  

We verify that for each $a \neq 0$ this kernel defines a natural spline. Let $I_a = \text{supp}(\mu_a)\subseteq \mathcal{D}_b$ 
and $\tilde{I}_a=\{ u; au \in I_a\}$. Let $s_a = \inf\{ x \in I_a \}$. 
For $f \in C^2(I_a)$ with $f'$ and $f''$ in $L^2(I_a)$ such that $f( s_a)=f'( s_a)=0$, let $f_a(x):=f(ax)$, $x \in \tilde{I}_a$. Assume first that $I_a$ is an interval (and thus $\tilde{I}_a$ is also an interval).  
By the Taylor integral remainder theorem, we have for $x \in \tilde{I}_a$ that
\begin{equation}
\label{kuku}
    f_a(x)=f(ax) = \int_{s_a}^{ax} (ax-u) f^{''}(u) \dd u = \int_{I_a} [ax-u]_+ \frac{f^{''}(u)}{\mu_a(u)} \mu_a(u) \dd u ~.
\end{equation}
This shows that $g_a=\frac{f^{''}}{\mu_a}$ is the kernel representation of $f_a$ corresponding to $K_a$. 
This kernel defines an RKHS $\mathcal{H}_a$ over the space of $C^2$ functions $f$ supported in $\tilde{I}_a$ with boundary conditions ${f}(s_a/a) = {f}'(s_a/a)=0$. Moreover, 
 its RKHS norm is by definition $$\| f_a \|_{K_a}^2 = \int_{I_a} \left|\frac{f^{''}(u)}{\mu_a(u)}\right|^2 \dd \mu_a(u)~. $$
For general support $I_a$, consider the closure $I_a^c :=(\inf I_a, \sup I_a)$. Since $I_a \subseteq \mathbb{R}$ is measurable and the support of a continuous function, we can write it as $I_a = \cup_{k=1}^\infty I_a^{(k)}$, where $I_a^{(k)}$ are disjoint intervals. We modify the boundary condition for $f$ accordingly, as 
\begin{equation}
\label{prov1}
f(\sup I_a^{(k)}) = f(\inf I_a^{(k+1)}) \text{ and } f'(\sup I_a^{(k)}) = f'(\inf I_a^{(k+1)}) ~\forall \, k.  \end{equation}

We extend the representation (\ref{kuku}) to $I_a^c$ by first extending $g_a$ to $\RR$ as 
$$\bar{g}_a(u) = \left\{ \begin{array}{cc}
\frac{f^{''}(u)}{\mu_a(u)} & \, u \in I_a \\
0 & \text{ otherwise}.
\end{array} \right.$$
We verify that for $f$ in $C^2(I_a^c)$ satisfying the boundary conditions (\ref{prov1}), $\bar{g}_a$ satisfies 
\begin{equation}
\label{kuku2}
    f_a(x) = \int_{I_a^c} [ax-u]_+ \bar{g}_a(u) \mu_a(u) \dd u ~.
\end{equation}
Therefore, if $h \in C^2( \tilde{I}_a^c)$, 
the change of variables $h=f_a$ for some $f$ in $C^2({I}_a^c)$ yields a kernel representation 
$$\hat{h}_a(u) = \left \{ \begin{array}{cc}
\frac{h^{''}(a^{-1}u)}{a^2 \mu_a(u)} & \, u \in I_a \\
0 & \text{ otherwise}.
\end{array} \right. ,$$
and thus
\begin{equation}
    \| h \|^2_{K_a}= \frac{1}{|a|}\int_{\tilde{I}_a} \frac{|h^{''}(u)|^2}{\bar{\mu}_a( u)}  \dd u~, 
\end{equation}
where $\bar{\mu}_a(u) := \mu_a(au)$.
For $f \in C^2(\RR)$, we can thus decompose 
it as $f = \tilde{f} + \bar{f}$, where $\tilde{f} \in \mathcal{H}_a$ is supported in $\tilde{I}_a^c$ and 
satisfies the boundary conditions (\ref{prov1}), 
and $\bar{f}''\equiv 0$ in $\tilde{I}_a$. The projection of $f$ onto $\mathcal{H}_a$ is $\tilde{f}$ and
\begin{equation}
\label{blii}
\|P_{\mathcal{H}_a} f \|_{K_a}^2 = \frac{1}{|a|} \int_{\tilde{I}_a} \left|\frac{f^{''}(u)}{\bar{\mu}_a(u)}\right|^2 \dd \bar{\mu}_a(u) ~.
\end{equation}

Let us now show that we can `integrate' the collection of kernels $K_a$ and define a 
resulting RKHS $\mathcal{H}_{\mathrm{RF}}$ whose kernel is precisely $K_{\mathrm{RF}}$. 
For that purpose, we verify first that for all $x$, 
\begin{equation}
\label{yuy}
    \int_{\mathcal{D}_a} K_a(x,x) \mathrm{d} q(a) < \infty ~.
\end{equation}
Indeed, observe that 
\begin{eqnarray*}
\int_{\mathcal{D}_a} K_a(x,x) \mathrm{d} q(a) &=& \int_{\mathcal{D}_a \times \mathcal{D}_b} [ax - b]_+^2 \dd \mu_0(a,b) \\
&\leq& 2\max(1,x^2) \int (a^2 + b^2) \mu_0(\mathrm{d}a,\mathrm{d}b) \leq 2\max(1,x^2) \sigma^2_{\mu_0} < \infty~,
\end{eqnarray*}
which proves (\ref{yuy}). We can thus apply Theorem 3.1 from \cite{hotz2012representation}, which establishes that $K_{\mathrm{RF}}$ defines an RKHS 
$$\mathcal{H}_{\mathrm{RF}} = \left\{f; f= \int_{\mathcal{D}_a} f_a \mathrm{d}q(a)~;\, f_a \in \mathcal{H}_a \right\}~,$$
with norm
$$\| f \|_{K_{\mathrm{RF}}}^2 :=\inf \left \{\int_{\mathcal{D}_a} \| f_a \|_{K_a}^2 \mathrm{d}q(a) ~;~f=\int_{\mathcal{D}_a} f_a \mathrm{d}q(a)   \right\}~,$$
which, from (\ref{blii}), gives
\begin{equation}
\label{bliii}
    \| f \|_{K_{\mathrm{RF}}}^2 = \inf \left \{\int_{\mathcal{D}_a} \left(\int_{\tilde{I}_a} \frac{|f_a^{''}(u)|^2}{\bar{\mu}_a(u)} \dd u\right) \frac{\mathrm{d}q(a)}{|a|} ~;~f=\int_{\mathcal{D}_a} f_a \mathrm{d}q(a)   \right\}~.
\end{equation}

Let us now show that (\ref{bliii}) 
has an explicit form as a weighted curvature. 
Let $\Omega:=\cup_{a} \tilde{I}_a$. 
From $f = \int_{\mathcal{D}_a} f_a \dd q(a)$ 
we deduce that 
$$\forall~u\in \Omega~,f''(u) = \int_{\mathcal{D}_a} f''_a(u) \dd q(a)~.$$
For each $u \in \Omega$, denote $\beta(a)=f''_a(u)$ for $a \in \mathcal{B}_u:=\{a; u \in \tilde{I}_a\}$, 
and $\nu(u) := \int_{\mathcal{B}_u} |a| \bar{\mu}_a(u) \dd q(a)$. 
Since $\| f\|^2_{K_{\mathrm{RF}}} < \infty$ 
and $f'' \in C^0(\Omega)$, let us first argue 
that $\nu(u)=0$ necessarily implies that 
$f''(u)=0$. Indeed, if $\nu(u)=0$, 
it follows that $\bar{\mu}_a(u)=0$ for all $a$, and since each $f''_a$ is also continuous, from (\ref{bliii}) we deduce that necessarily $f''_a(u)=0$ for all $a$, which implies that $f''(u)=0$.

Let us now assume that $\nu(u)>0$.
The constrained minimisation problem 
\begin{equation}
\label{co1}
    \min_{\beta(a)}\quad  \int_{\mathcal{B}_u} |\beta(a)|^2 \frac{\dd  q(a)}{\bar{\mu}_a(u) |a|} 
    \quad s.t\quad   f''(u) = \int_{\mathcal{B}_u} \beta(a) \dd q(a)
\end{equation}
has an associated Lagrangian 
$$\mathcal{L}(\beta, \lambda) = \int_{\mathcal{B}_u} |\beta(a)|^2 \frac{\dd q(a)}{\bar{\mu}_a(u) |a|} + \lambda\left( f''(u) - \int_{\mathcal{B}_u} \beta(a) \dd q(a)\right)~.$$
The first-order KKT optimality conditions directly give 
\begin{equation}
\label{co2}
\beta(a) = \frac{|a| \bar{\mu}_a(u) f''(u)}{\int_{\mathcal{B}_u} |a'| \bar{\mu}_{a'}(u) \dd q(a')}~,
\end{equation}
resulting in a minimum of (\ref{co1}) at
\begin{equation}
\label{co3}
\frac{|f''(u)|^2}{\nu(u)}~.    
\end{equation}
Finally, we need to show that one can make these optimal choices for each $u$ to define $f_a \in \mathcal{H}_a$ for each $a$. 
From (\ref{co2}), $f''_a$ is defined for each $u \in \tilde{I}_a$ as $f_a''(u) = f''(u) \eta_a(u)$, where  $\eta_a(u)=\frac{|a| \bar{\mu}_a }{\nu(u)}$ is the normalised density satisfying $\int \eta_a(u)\dd q(a)=1$ for all $u$.
We verify that $f'' \eta_a$ is in $L^2(\tilde{I}_a)$ thanks to the smoothness hypothesis of $\mu_0$. Indeed, 
by assumption, the normalised density 
$\eta_a$ is bounded for each $a$, $\sup_u \eta_a(u) \leq K_a$,
and thus
$$\int_{\tilde{I}_a} |f''_a(u)|^2 \dd u = \int_{\tilde{I}_a} |f''(u)|^2 \eta_a(u)^2 \dd u \leq K_a^2  \int_\Omega \frac{|f''(u)|^2}{\nu(u)^2} \dd u < \infty~,$$

since $f'' \in L^2(\Omega)$.
We can therefore define $f_a$ on each $\tilde{I}_a$ thanks the reproducing property 
$$\forall ~a,\,x\in \tilde{I}_a ~,\, f_a(x)=\int f_a''(u) [x-u]_+ \dd u~.$$
Integrating over $a$  yields
\begin{eqnarray*}
\int_{\mathcal{D}_a} f_a(x) \dd q(a) &=& \int_{\mathcal{D}_a} \int f_a''(u) [x-u]_+ (\dd u) \dd q(a)= \int \left(\int f_a''(u) \dd q(a) \right) [x-u]_+  (\dd u) \\
&=& \int f''(u) [x-u]_+ \dd u= f(x) \,,
\end{eqnarray*}
which proves that $(f_a)_a$ is a feasible set to represent $f$ in (\ref{bliii}).

Finally, by applying the optimality conditions (\ref{co1}, \ref{co3}) to each $u \in \Omega$ one obtains $$\|f \|^2_{K_{\mathrm{RF}}} = \int_{\Omega} \frac{|f''(u)|^2}{ \nu(u)} \dd u\,$$
which concludes the proof of the first statement.

Finally, let us verify that in the specific case when $\mu_0(a,b) = q(a) \mathbf{1}( b \in I_a)$ the kernel is piece-wise cubic. 
Fix $a \in \mathcal{D}_a$ and suppose first that $a>0$. Let $I_a=[s_a, S_a]$, and denote $\tilde{I}_a = [s_a/a,S_a/a]$.
A direct calculation shows that when $x' \in \tilde{I}_a$, $K_a(x,x')$ is $0$ when $x < s_a/a$, cubic in $x$ when $x \leq x'$, and linear when $x > x'$.
When $x' < s_a/a$, then $K_a(\cdot,x')\equiv 0$, and finally when $x' > S_a/a$ we have that $K_a(x, x')=0$ for $x < s_a/a$, $K_a(x,x')$ is cubic for $x \in \tilde{I}_a$ and linear for $x > S_a/a$. We also verify that its first and second derivatives are continuous. In summary, the  evaluation functions $K_a(\cdot, x')$ are $0$ for $x < s_a/a$, a piece-wise cubic inside $\tilde{I}_a$, and linear for $x>S_a/a$. When $a<0$ we verify the same holds true by flipping the sense of inequalities. 
$\square$


\subsection{Proof of Corollary \ref{coro:ntk}}

The proof follows closely the previous section. 

Observe first that $K_{\mathrm{NTK}}$ can 
be written as the sum of three kernels:
$K_{\mathrm{NTK}} = K_{\mathrm{RF}} + \mathbb{E}(c^2)(K_1 + K_2)$, with 
$$K_1(x,x') = \mathbb{E}_{\mu_0} [\mathbf{1}(ax-b>0)\mathbf{1}(ax'-b>0)]~,~K_2(x,x') = xx'\mathbb{E}_{\mu_0} [\mathbf{1}(ax-b>0)\mathbf{1}(ax'-b>0)]~.$$

We will first assume that $\mu_0$ is concentrated at $a=1$ and uniform $b\in(0,1)$.
We can then extend to the general case by the previous proof. 
Observe that for $f \in C^1([0,1))$ with $f(0)=0$ the Taylor integral remainder at order $1$ now gives
$$f(x) = \int_0^1 f'(u) \mathbf{1}( x-u>0) \dd u~,~x\in(0, 1)~.$$
It follows that $f'$ is the kernel representation associated to $K_1(x,x')$. 

Similarly, 
$$x f(x) = \int_0^1 f'(u) x \mathbf{1}( x-u>0) \dd u~,~x\in(0, 1)~,$$
which shows that $f'$ is the kernel representation of $g(x) = x f(x)$ corresponding to the kernel $K_2(x,x')$. 
A change of variables then implies that 
$$f' = \left(\frac{g}{x}\right)'=\frac{x g' - g}{x^2}$$
is the contribution of $K_2$ to the RKHS norm of $K_{\mathrm{NTK}}$. 
The proof is completed by repeating the argument used in the proof of Theorem \ref{prop:splines} that replaces the uniform measure in the interval $(a,b) \in \{1\} \times (0,1)$ by a smooth density $\mu_0$. 
We leave the derivation of an explicit form of the mixed norm for future work. $\square$.

\subsection{Radial Kernel}
We verify by direct calculation that 
\begin{eqnarray}
K_{\mathrm{RF}}^r(x,x') &=& \int_{\mathcal{D}_a \times \mathcal{D}_b}  [ax + b]_+ \cdot [ax' + b]_+ \mu_0(da, db) \\
&=& \frac{1}{2\pi}\int_{\arctan(x')}^{\arctan(x)+\pi}\int_0^\infty  r^3 (\cos(\eta)x+\sin(\eta))(\cos(\eta)x'+\sin(\eta)) \mu_0(dr) d\eta \nonumber \\ 
&=& \frac{C}{2\pi} \int_{\arctan(x')}^{\arctan(x)+\pi} \left(x x' \cos^2(\eta) + \sin^2(\eta) + (x+x')\cos(\eta)\sin(\eta) \right) d\eta \nonumber \\
&=& \frac{C}{2\pi} \left[x x' \left(\frac{\eta}{2} +\frac{\sin(2\eta)}{4}\right) + \left(\frac{\eta}{2} -\frac{\sin(2\eta)}{4}\right) + \frac{(x+x')}{2}\sin^2(\eta)  \right]_{\arctan(x')}^{\arctan(x)+\pi} \nonumber \\
&=& \frac{C}{4\pi}\Big\{(\pi - \arctan(x')+\arctan(x))(xx'+1) + \left( \frac{x'}{1+(x')^2} - \frac{x}{1+x^2} \right)(xx'-1) + \nonumber \\ && \quad + (x+x')\left(\frac{(x')^2}{1+(x')^2} - \frac{x^2}{1+x^2} \right)\Big\} \nonumber ~,
\end{eqnarray}

\section{Mean Field Computations}
\label{sec:mfappendix}

Making use of notation introduced in Section~\ref{sec:reduced_dynamics}, we have that if $w=(\hat{r}, \theta)$, with $\theta \in \mathcal{A}_k$, then
\begin{eqnarray}
\label{eq:meanfield1}
    \nabla_\theta V(w; \mu_t) &=&
    -\nabla_\theta F(w) + \int_\mathcal{D} \nabla_\theta K(w, w') \mu_t(dw') \nonumber \\
    &=& -\hat{r}\left( \sum_{j \in \mathcal{C}_k} y_j \langle \tilde{x}_j, t(\theta)\rangle  -\int_\mathcal{D} r'\sum_{j \in \mathcal{C}_k} \langle \tilde{x}_j,t(\theta)\rangle \langle \tilde{x}_j, d(\theta') \rangle_+ \mu_t(d\hat{r}', d\theta') \right) \nonumber \\
    &=& -\hat{r} \sum_{j \in \mathcal{C}_k} \langle \tilde{x}_j, t(\theta)\rangle \left(y_j - \int_{\mathbb{R} \times \mathcal{B}_j} \hat{r}' \langle \tilde{x}_j, d(\theta') \rangle \mu_t(d\hat{r}', d\theta') \right) \nonumber \\
    &=& \hat{r} \left \langle \sum_{j \in \mathcal{C}_k} \rho_j(t) \tilde{x}_j, t(\theta) \right \rangle~,
\end{eqnarray}
where $\rho_j(t) = f_{\mu_t}(x_j) - y_j =    \int_{\mathbb{R} \times \mathcal{B}_j} c \langle \tilde{x}_j,\theta  \rangle \mu_t(d\hat{r}, d\theta)  - y_j$
is the residual at point $x_j$ at time $t$. 
Similarly, the field in the direction of the charges is given by
\begin{equation}
\label{eq:meanfield3}
    \nabla_{\hat{r}} V(w; \mu_t) =  \left \langle \sum_{j \in \mathcal{C}_k} \rho_j(t) \tilde{x}_j, \theta \right \rangle ~.
\end{equation}

\subsection{Proof of Proposition \ref{prop:resiode}}
We observe that for each $j$, 
\begin{eqnarray}
\label{eq:oderesiduals}
\dot{\rho}_j(t) &=& \partial_t f_{\mu_t}(x_j) = \partial_t \left( \int_\mathcal{D} \varphi(w;x_j) \mu_t(dw) \right) \nonumber \\ 
&=& - \int_\mathcal{D} \langle \nabla_w \varphi(w;x_j), \nabla V(w;\mu_t) \rangle \mu_t(dw) \nonumber \\
&=& - \int_\mathcal{D} \left(  \nabla_\theta \varphi(w;x_j) \cdot \nabla_\theta V(w;\mu_t)  +  \nabla_r \varphi(w;x_j) \cdot \nabla_r V(w;\mu_t)  \right) \mu_t(dw) \nonumber \\
&=& -\sum_{k; \mathcal{A}_k \subset \mathcal{B}_j} \int_{\mathbb{R} \times \mathcal{A}_k} \left( r^2  \tilde{x}_j^\top (t(\theta)t(\theta)^\top) (\sum_{j' \in \mathcal{C}_k} \rho_{j'}(t) \tilde{x}_{j'}) + \tilde{x}_j^\top (\theta \theta^\top) (\sum_{j' \in \mathcal{C}_k} \rho_{j'}(t) \tilde{x}_{j'}) \right) \mu_t(dw) \nonumber  \\
&=& -\tilde{x}_j^\top \sum_{k; \mathcal{A}_k \subset \mathcal{B}_j}  \Sigma_k(t) (\sum_{j' \in \mathcal{C}_k} \rho_{j'}(t) \tilde{x}_{j'})~, 
\end{eqnarray}
where 
$$\Sigma_k(t) = \int_{\mathbb{R} \times \mathcal{A}_k} \left(r^2 t(\theta)\, t(\theta)^\top + \theta\, \theta^\top\right) \mu_t(dr,d\theta) $$
tracks the covariance of the measure along each  cylindrical region. 

\section{Changing Metric in the Dynamics}

\begin{lemma}\label{le:fixed_delta_supp}
If $\bm z(t) = (\bm a(t), \bm b(t), \bm c(t))$ is a solution of the gradient flow \eqref{eq:gradient_flow}, then the quantities
\begin{equation}\label{eq:invariants_supp}
\bm \delta = (\delta_i = c_i(t)^2 - a_i(t)^2 - b_i(t)^2)_{i=1}^m
\end{equation}
remain constant for all $t$. In particular, given a reduced neuron $(r_i,\theta_i)$, we can uniquely recover the original neuron $(a_i,b_i,c_i)$, since
\begin{equation}\label{eq:c_uv_supp}
    c_i^2 = \frac{\delta_i + \sqrt{\delta_i^2 + 4 r_i^2}}{2}. 
\end{equation}
\end{lemma}

\begin{proof}
 The gradient equations of the loss $L(\bm z)$ can be written as
 \begin{equation}\label{eq:derivative_equations}
\begin{aligned}
&\nabla_{a_i}L(\bm z) = c_i \sum_{j=1}^s  \mathds 1[a_i x_j - b_i \ge 0] x_j r_j,\\
&\nabla_{b_i}L(\bm z) = c_i \sum_{j=1}^s \mathds 1[a_i x_j - b_i \ge 0] r_j,\\
&\nabla_{c_i}L(\bm z) = \sum_{j=1}^s  \mathds 1[a_i x_j - b_i \ge 0](a_i x_j - b_i) r_j.\\
\end{aligned}
\end{equation}
 From these expressions we see that
\begin{equation*}\label{eq:delta_i}
\begin{aligned}
\dot{\delta}_i &= 2 c_i \dot{c}_i - 2 a_i \dot{a}_i - 2 b_i \dot{b}_i \\
               &= 2 c_i \nabla_{c_i} L(\bm z) - 2 a_{i} \nabla_{a_{i}} L(\bm z) - 2 b_i \nabla_{b_i} L(\bm z) \\
               &= 0.
\end{aligned}
\end{equation*}
Using $r^2_i = c_i \sqrt{a_i^2 + b_i^2}$, we see that $c_i^2 - \frac{r^2_i}{c_i^2} = \delta_i$ implies $c_i^4 - \delta_i c_i^2 - r^2 = 0$, and thus~\eqref{eq:c_uv_supp}.
\end{proof}

\begin{proposition}\label{thm:reduced_parameter_grad_app}
Let $\bm z(t)$ be a solution gradient flow \eqref{eq:gradient_flow} of $L(\bm z)$, and let $\bm \delta = (\delta_i) \in \RR^m$ be the vector of invariants~\eqref{eq:invariants}, which depend only on the initialization $\bm z(0)$. If $\bm w(t) = (\bm r(t), \bm \theta(t))$ is curve of reduced parameters corresponding to $\bm z(t)$, then we have that
\begin{equation*}
\dot{\bm w}_i(t) = \bm P_i \cdot \nabla_{\bm w_i} \tilde L(\bm w),
\quad i=1,\ldots,m,
\end{equation*}
where
\begin{equation*}
\bm P_{\delta_i}(r_i) = \begin{bmatrix}
    \frac{m^2}{\alpha(m)^2} (a_i^2 + b_i^2 + c_i^2)  & 
    0        \\
    0   & \frac{1}{a^2_i + b^2_i} \\
\end{bmatrix} = 
\begin{bmatrix}
\frac{m^2}{\alpha(m)^2}\left(\frac{r_i^2}{c(r_i)^2} + c(r_i)^2\right) & 0\\
0 &  \frac{c(r_i)^2}{r_i^2}\\
\end{bmatrix},
\end{equation*}

and $c(r_i)^2 = \frac{\delta_i + \sqrt{\delta_i^2 + 4 r_i^2}}{2}$.
\end{proposition}
\begin{proof} The Jacobian of the mapping $\bm \pi$ from parameters to reduced parameters is given by
\begin{equation*}
\nabla(\bm \pi)(a_i,b_i,c_i) = 
    \begin{bmatrix}
    \frac{m}{\alpha(m)}\frac{ca}{\sqrt{a_i^2 + b_i^2}} & \frac{m}{\alpha(m)}\frac{cb}{\sqrt{a_i^2 + b_i^2}} & \frac{m}{\alpha(m)}\sqrt{a_i^2 + b_i^2}\\
    - \frac{b}{{a_i^2 + b_i^2}} & \frac{a}{{a_i^2 + b_i^2}} & 0
    \end{bmatrix}, \qquad i=1,\ldots,m.
\end{equation*}
This implies that the \emph{tangent kernel} $\bm P_{\delta_i}(r_i) = \nabla({\bm \pi})^\top \nabla({\bm \pi})$ is as in~\eqref{eq:neuron_kernel}. We emphasize that the fact that this kernel can be written only as a function of $\bm w$ (and, in fact, only of $\bm r$) relies in essential manner on Lemma~\ref{le:fixed_delta}.
\end{proof}

\section{Additional Numerical Experiments}\label{sec:additional_experiments}

In Figure~\ref{fig:trajectories}, we plot the trajectories of neurons for $\delta = \pm \infty$ over 10000 epochs. We see that, if $\delta = -\infty$, the neurons move radially away from the origin and thus the knot positions do not change (top row). In stark contrast, if $\delta = \infty$, the neurons adapt to the input data, and the knots ``stick'' to input samples (bottom row).
%
    

\begin{figure}
    \centering
    \parbox{0.5\textwidth}{\centering $\delta = -\infty$}\hfill
    \parbox{0.5\textwidth}{\centering $\delta =  \infty$}
    \minipage{0.50\textwidth}
    \includegraphics[width=\linewidth]{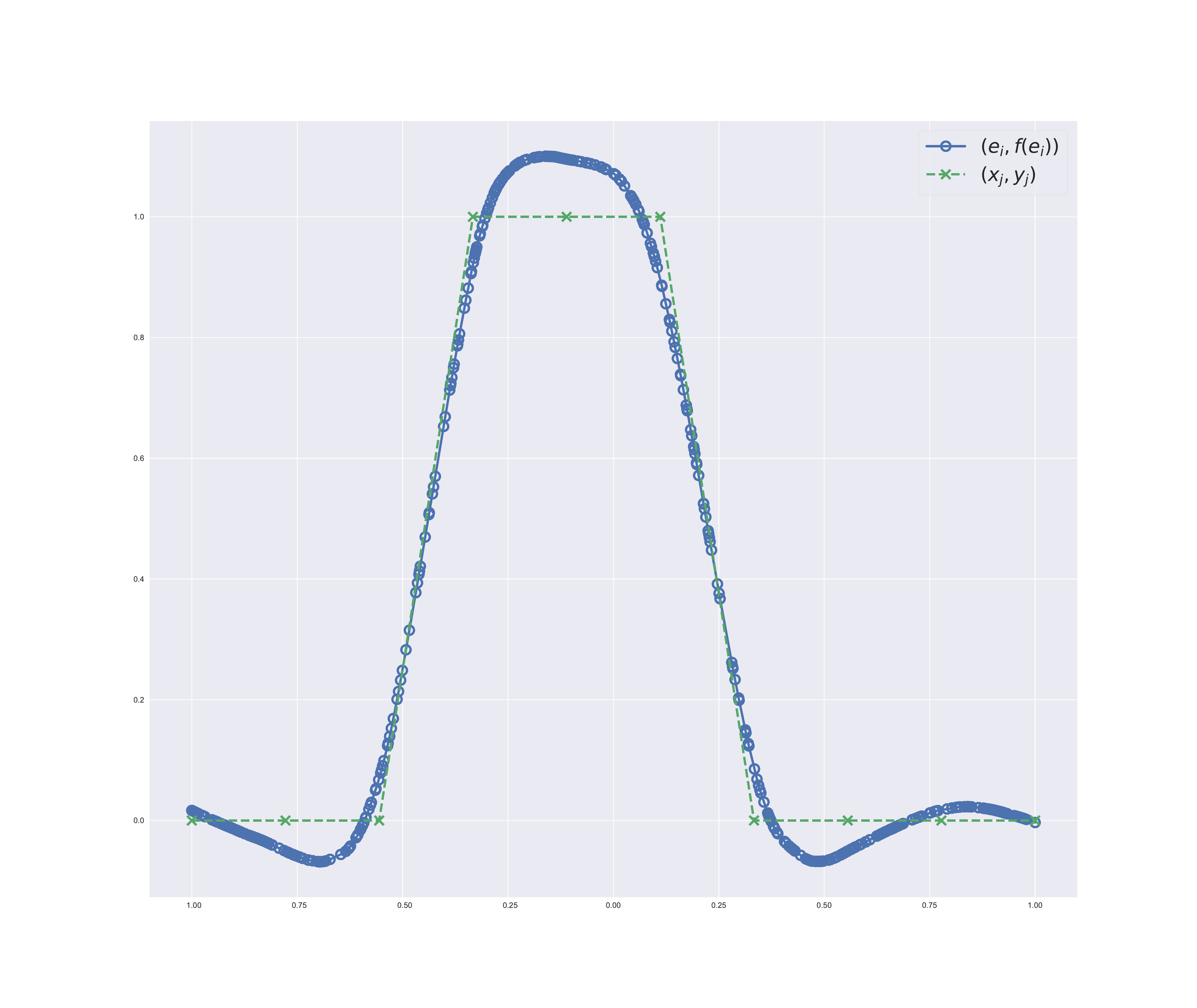}
    \endminipage\hfill
    \minipage{0.50\textwidth}
    \includegraphics[width=\linewidth]{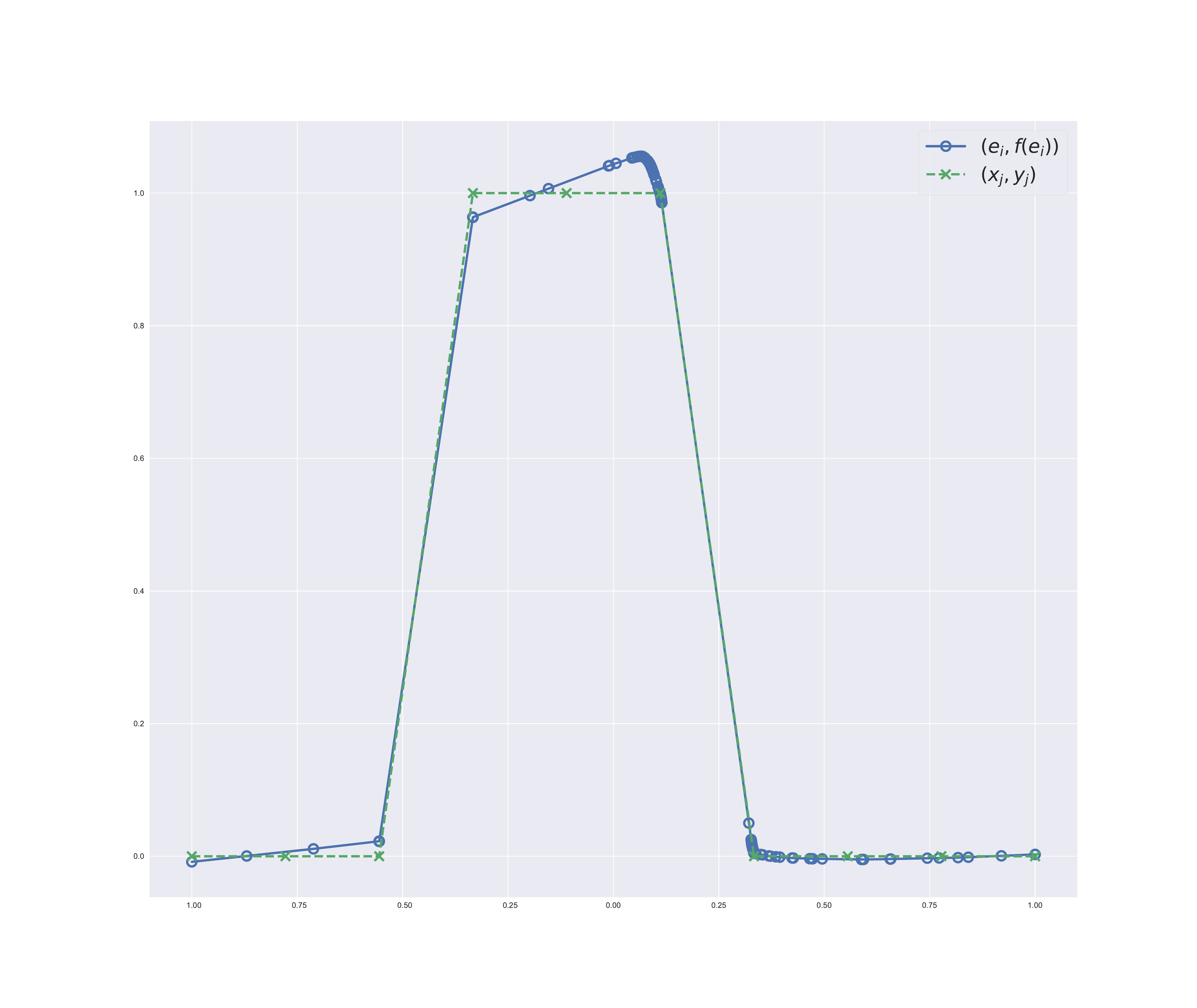}
    \endminipage\hfill\\
    \minipage{0.50\textwidth}
    \includegraphics[width=\linewidth]{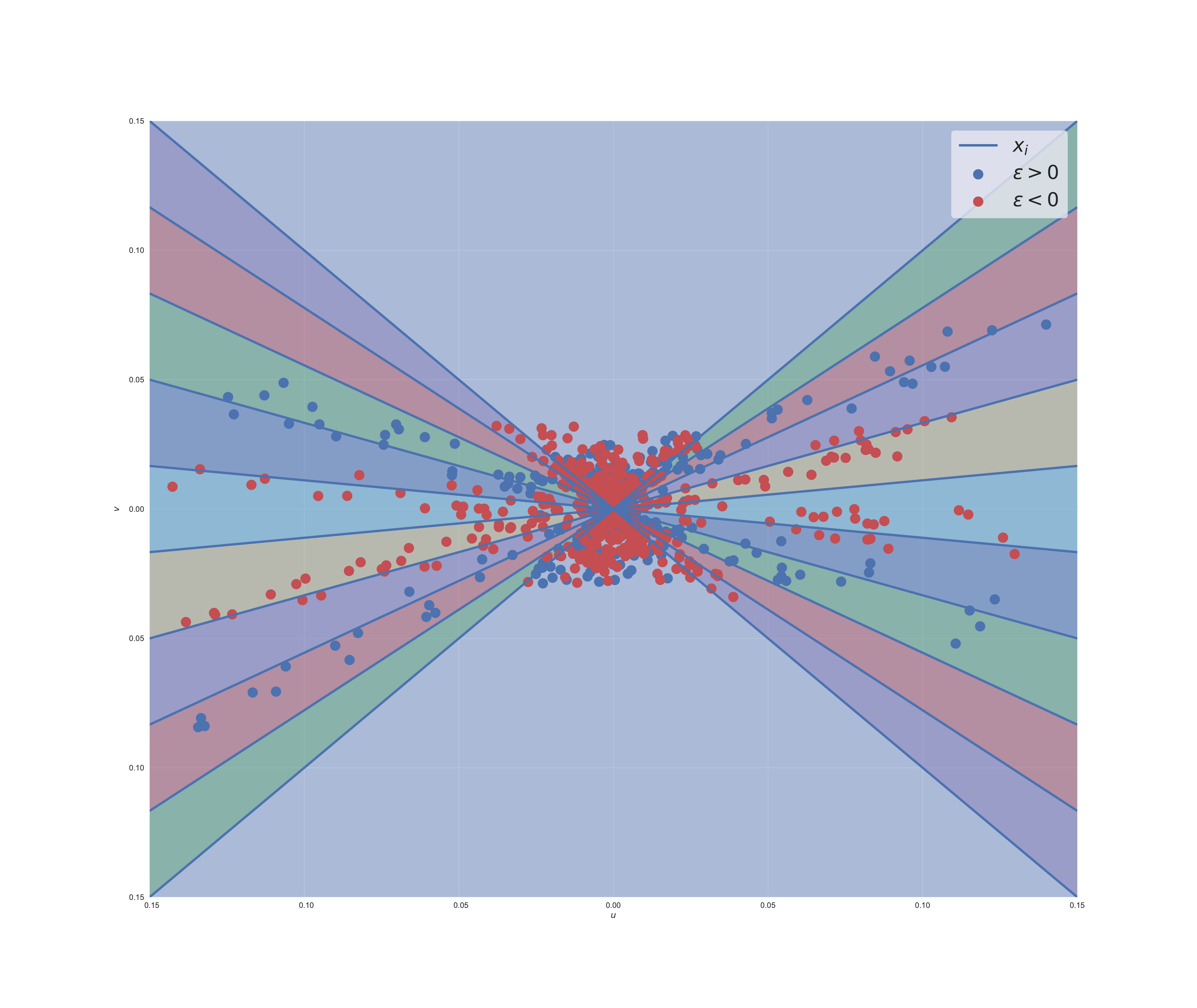}
    \endminipage\hfill
    \minipage{0.50\textwidth}
    \includegraphics[width=\linewidth]{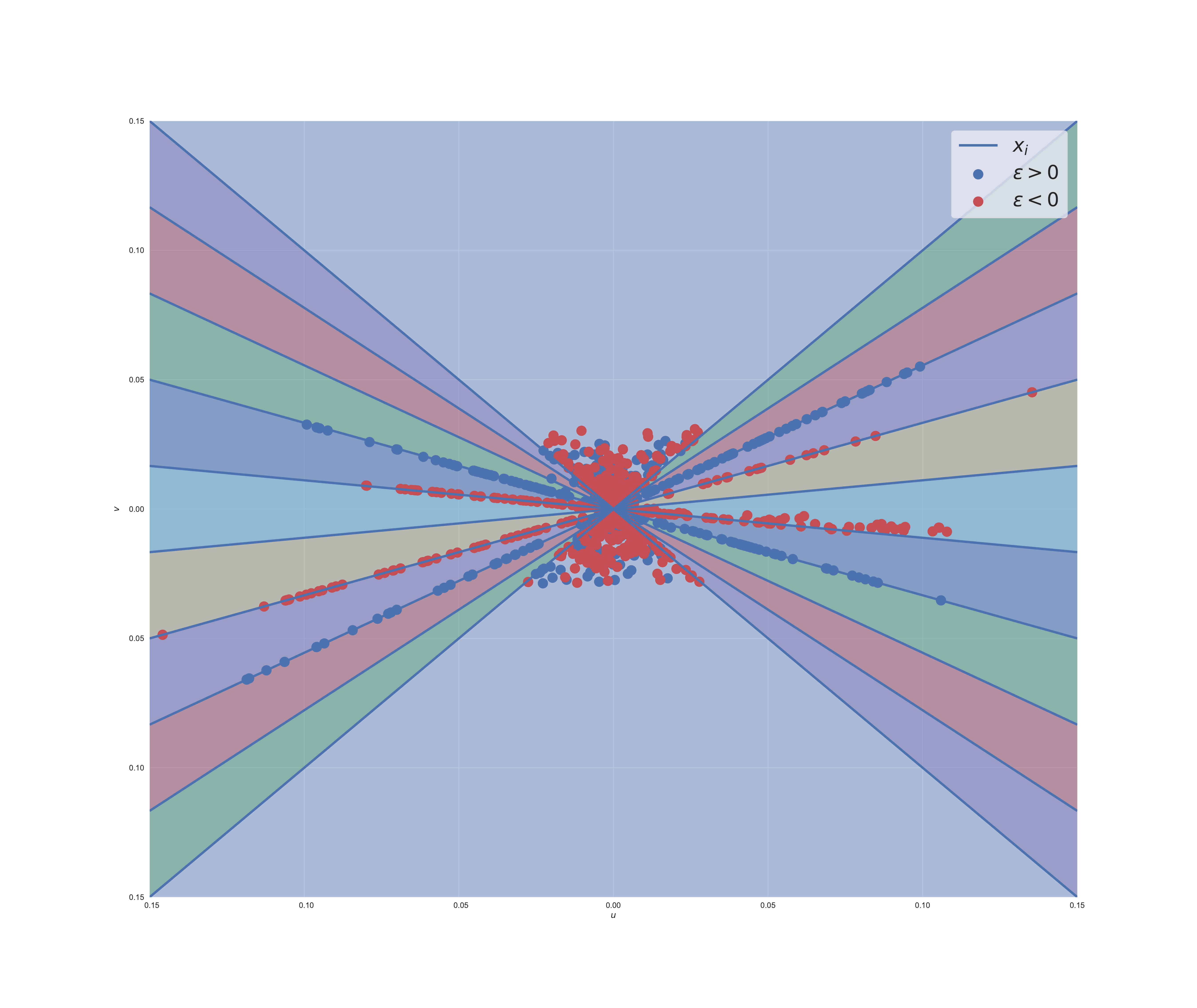}
    \endminipage\hfill\\
    \minipage{0.50\textwidth}
    \includegraphics[width=\linewidth]{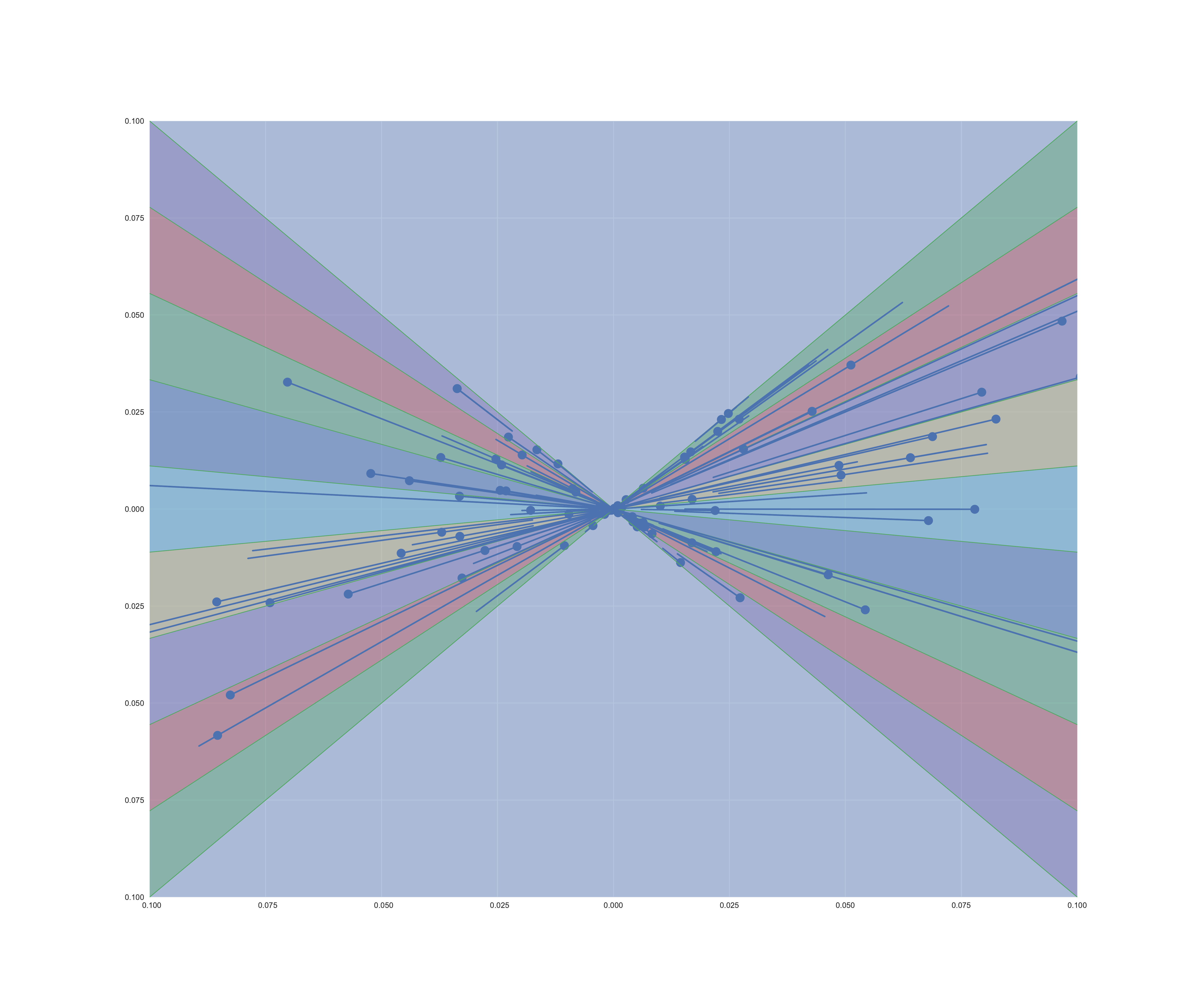}
    \endminipage\hfill
    \minipage{0.50\textwidth}
    \includegraphics[width=\linewidth]{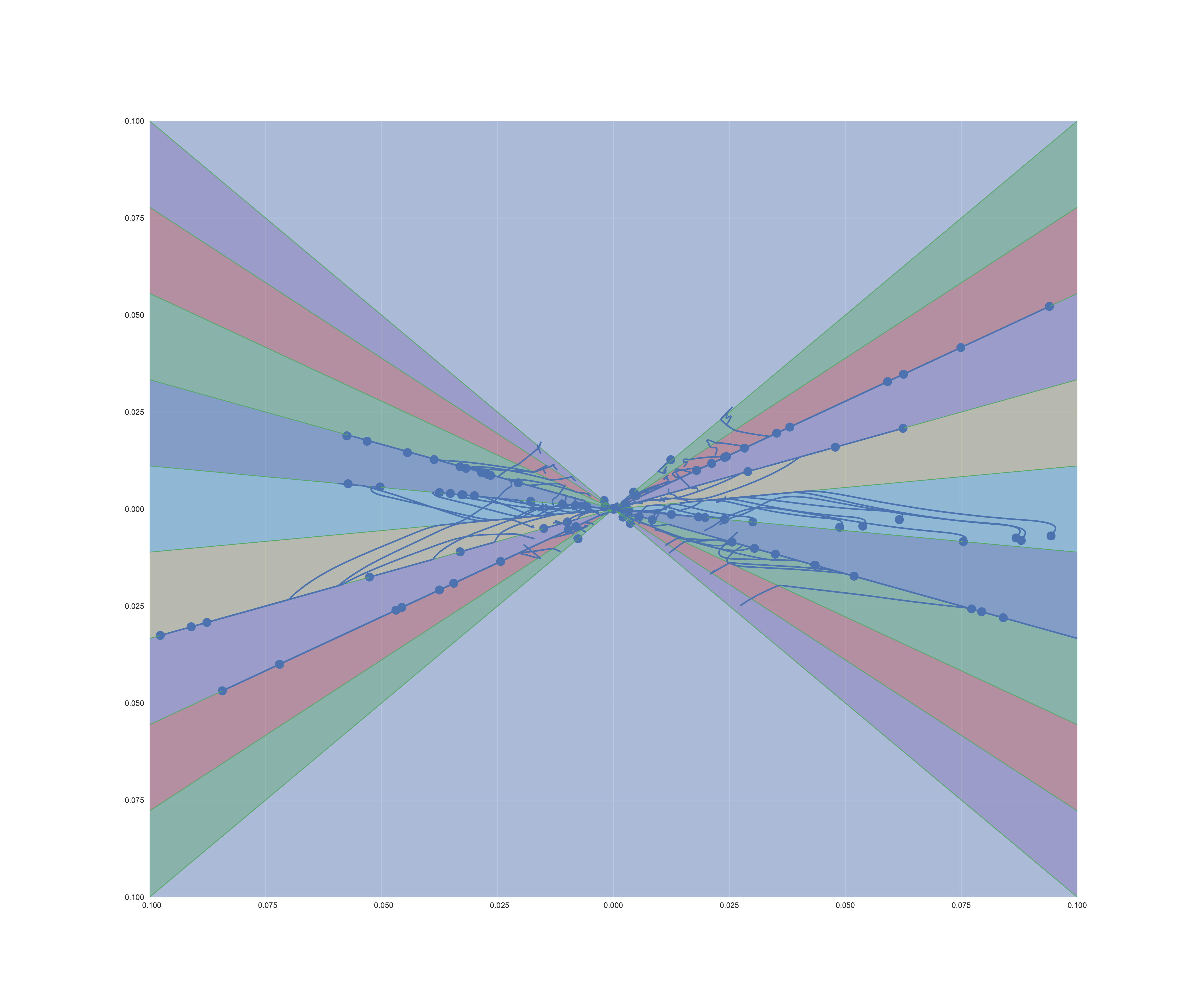}
    \endminipage
    \caption{Evolution of 1000 neurons over 10000 epochs for $\delta = \pm \infty$ while fitting 10 points sampled from a square wave. Left: plotted network function after training. Middle: state of the network in $uv$ space after training. Right: training trajectories of 100 random neurons.}
    \label{fig:trajectories}
\end{figure}

We remark in Figure~\ref{fig:different_funcs_same_init} that the same initial function can yield extremely different results depending on $\delta$. 

\begin{figure}
    \centering
    \minipage{0.33\textwidth}
    \includegraphics[width=\linewidth]{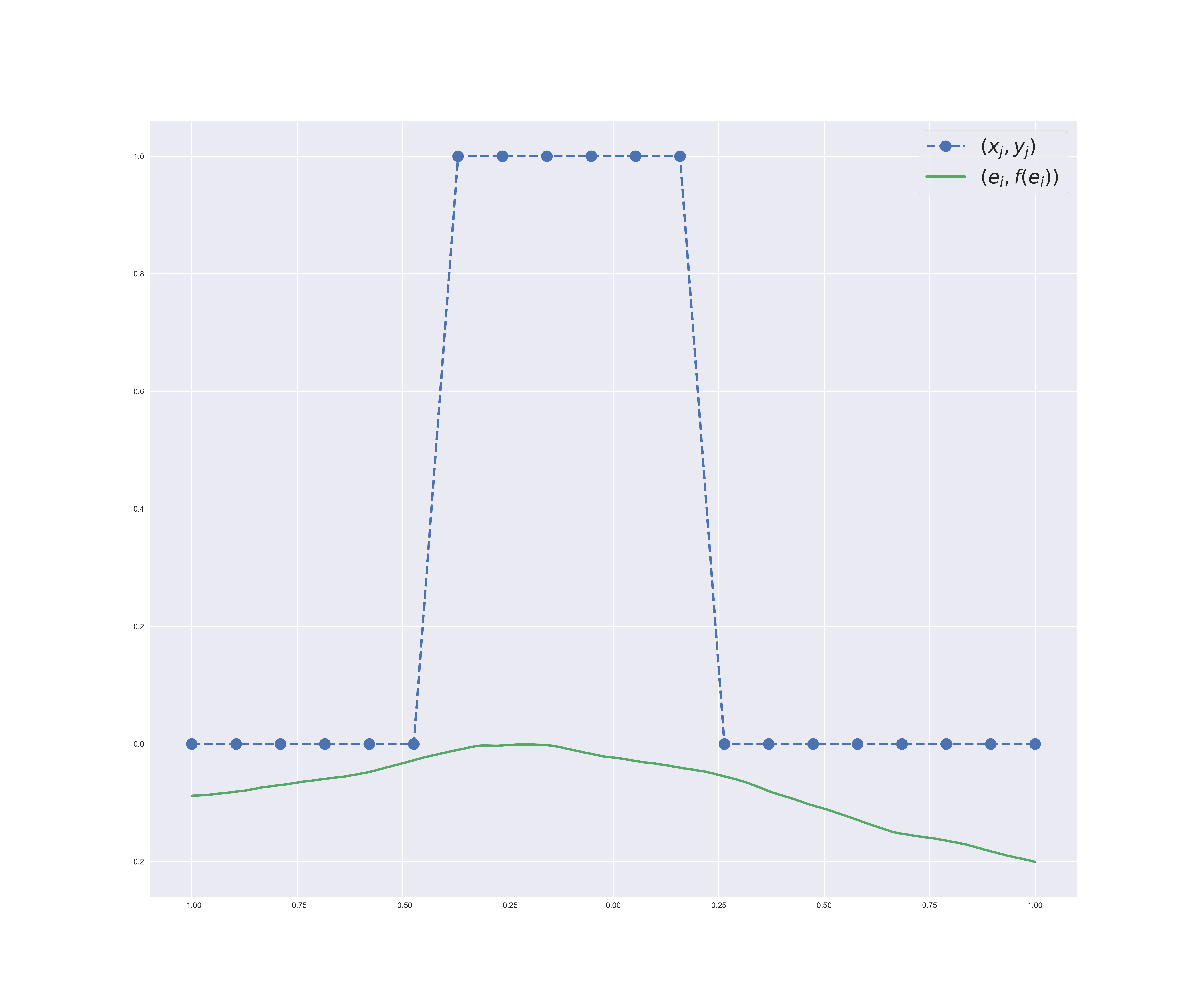}
    \endminipage\hfill
    \minipage{0.33\textwidth}
    \includegraphics[width=\linewidth]{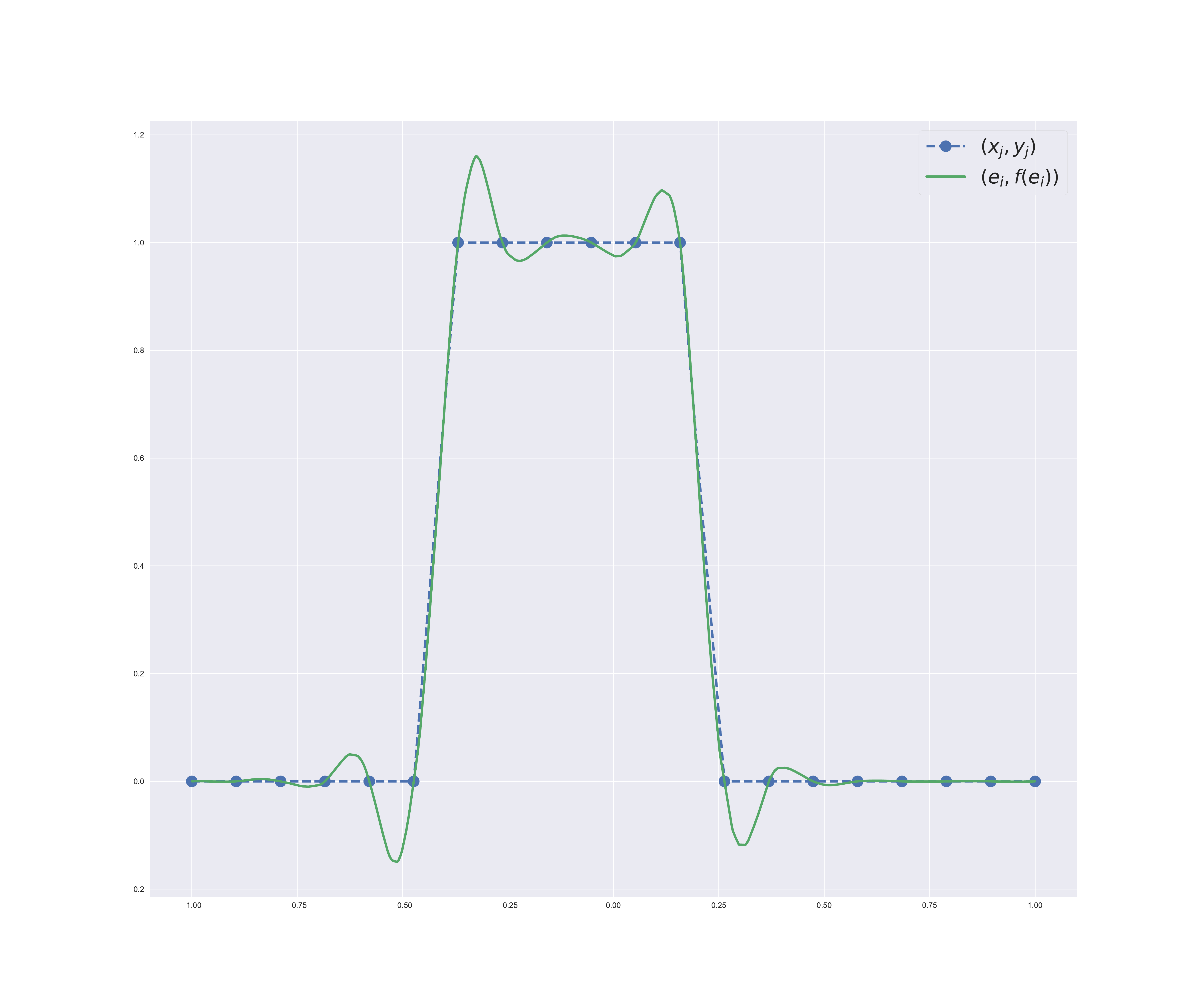}
    \endminipage\hfill
    \minipage{0.33\textwidth}
    \includegraphics[width=\linewidth]{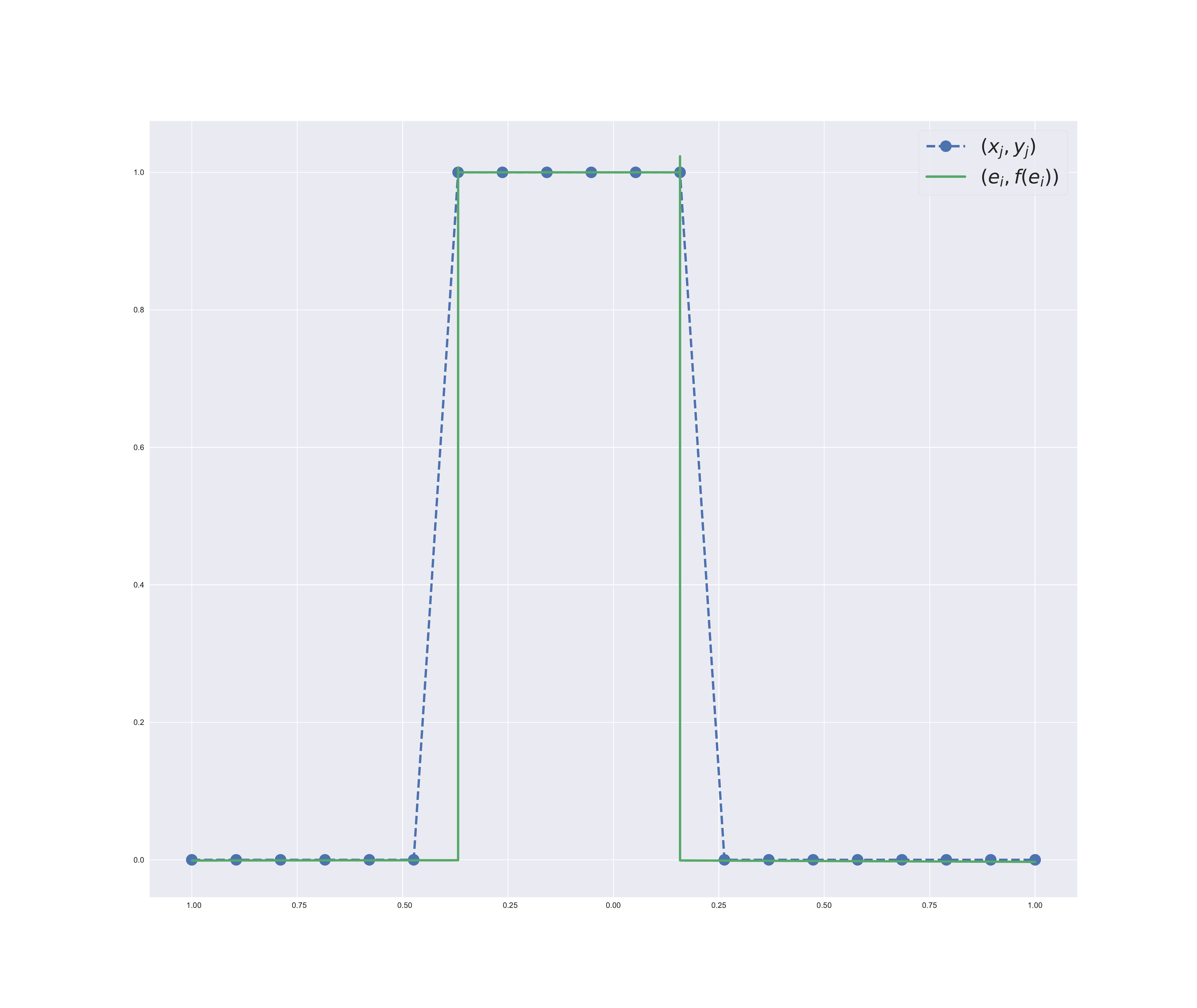}
    \endminipage\\
    \vspace{-5pt}
    \parbox{.32\textwidth}{\centering Epoch $0$\\ $\bm z$ = ($\bm a, \bm b, \bm c$)}
    \parbox{.32\textwidth}{\centering Epoch $10^4$\\ $\bm z$ = ($\bm a, \bm b, \bm c$)}
    \parbox{.32\textwidth}{\centering Epoch $10^4$\\ $\bm z$ = ($10^3 \bm a, 10^3 \bm b, 10^{-3} \bm c$)}
    \caption{Left: A network (green) with an initial set of parameters initial parameters ($\bm a, \bm b, \bm c$) is used to approximate a given function (blue). Scaling the initial parameters (right) leads to a very different fit (middle).}
    \label{fig:different_funcs_same_init}
\end{figure}

We now show the effect of varying the number of neurons during training. In this example we fit 20 samples from a sine wave using 20, 200, and 2000 neurons respectively. In PyTorch, the default initializatioon is such that $\bm a, \bm b ~ U(-1, 1)$ and $\bm c ~ U(-1/m, 1/m)$. Thus, as we scale down the numberof neurons, the value of $\delta$ grows, making the network function adapt more to the data. Figure~\ref{fig:varying_m} shows the results of this experiment. 

\begin{figure}
    \centering
    \minipage{0.5\textwidth}
    \includegraphics[width=\linewidth]{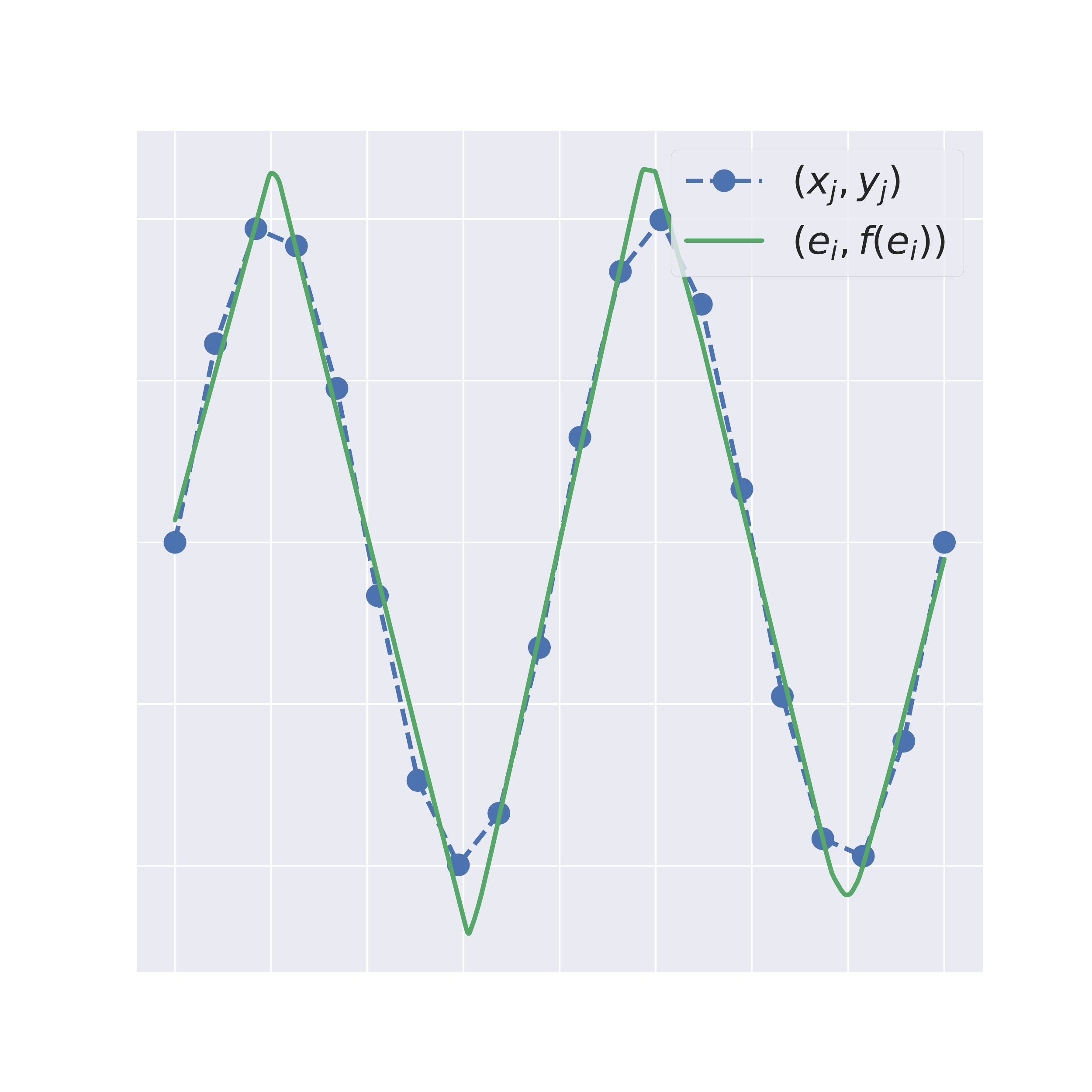}
    \endminipage\hfill
    \minipage{0.5\textwidth}
    \includegraphics[width=\linewidth]{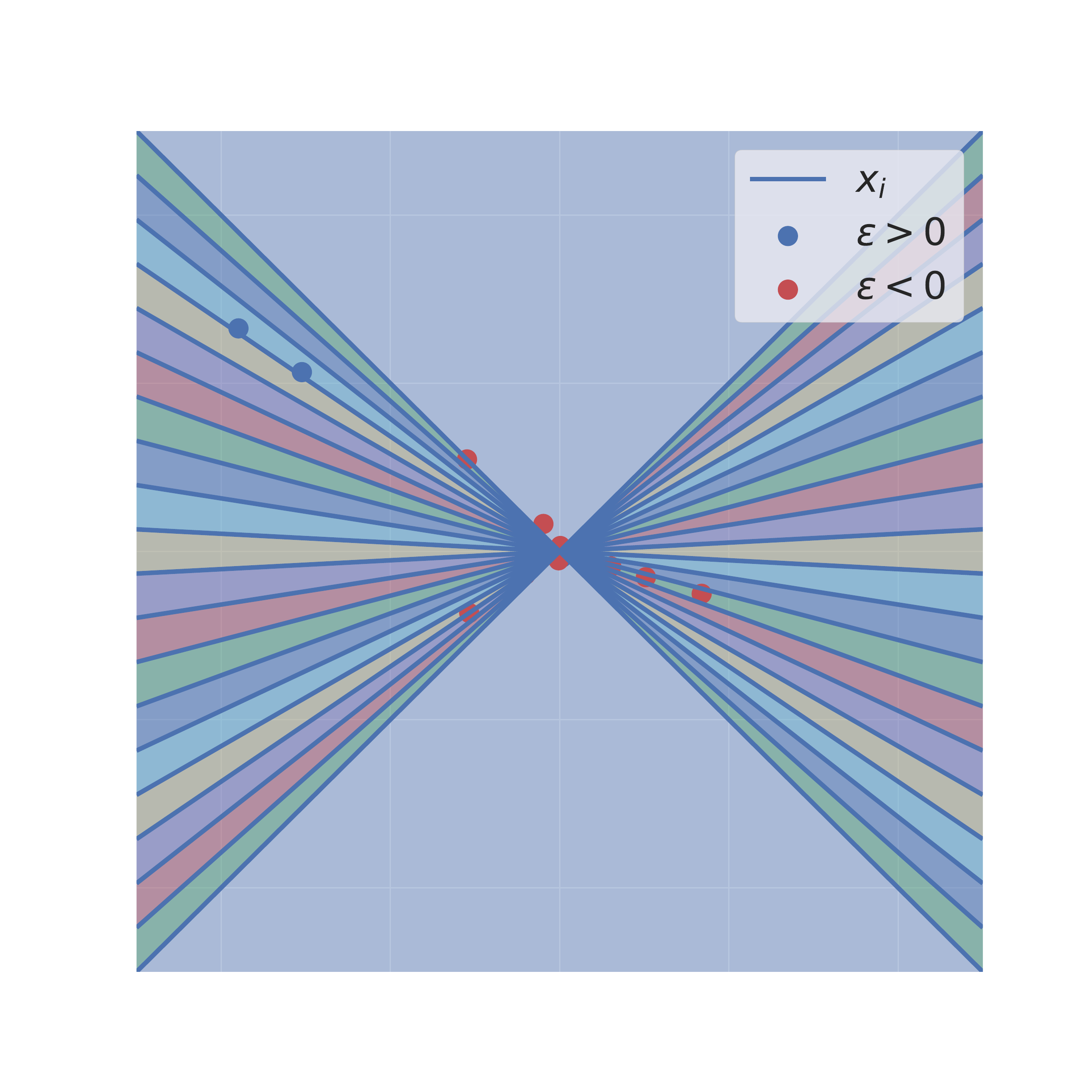}
    \endminipage\hfill
    
    \minipage{0.5\textwidth}
    \includegraphics[width=\linewidth]{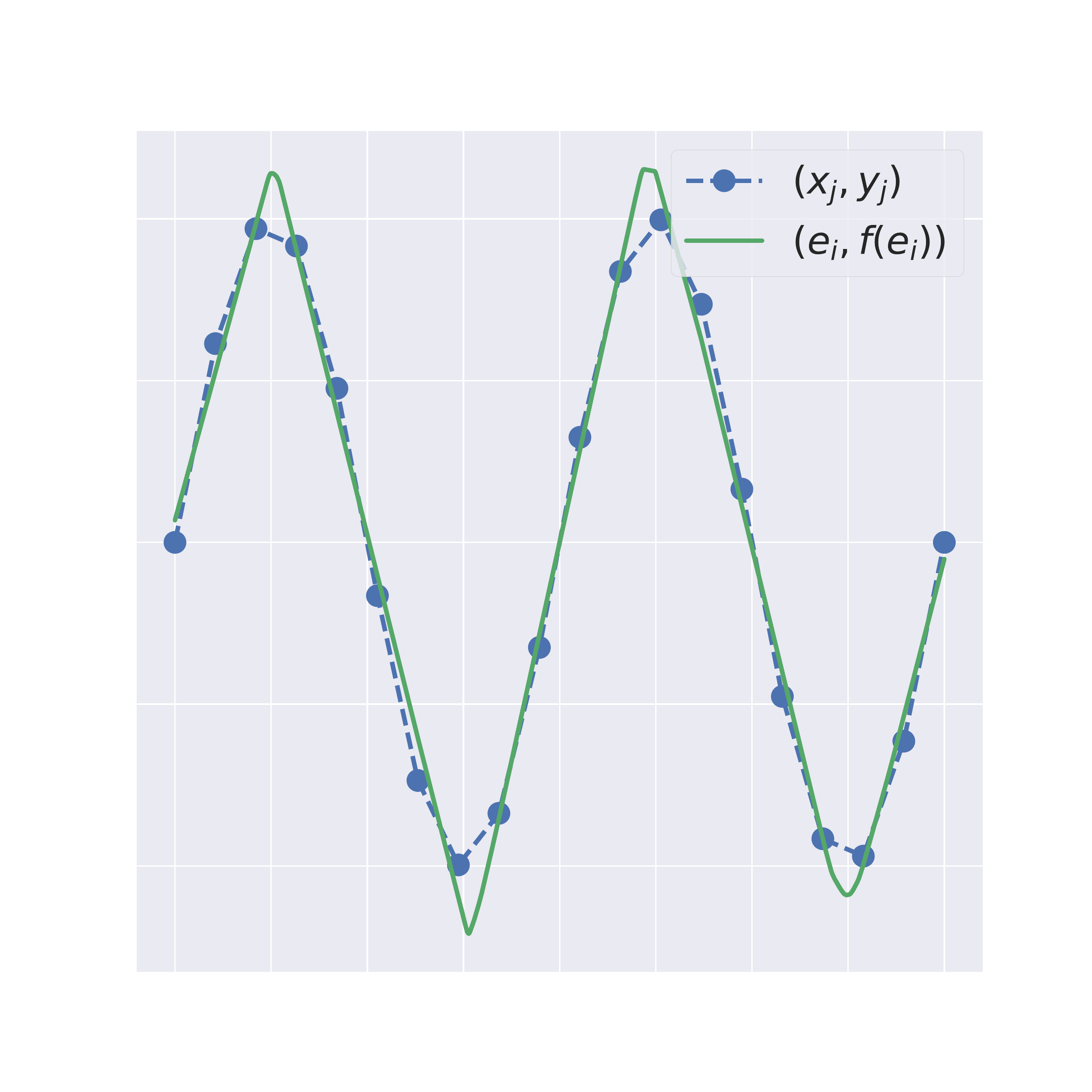}
    \endminipage\hfill
    \minipage{0.5\textwidth}
    \includegraphics[width=\linewidth]{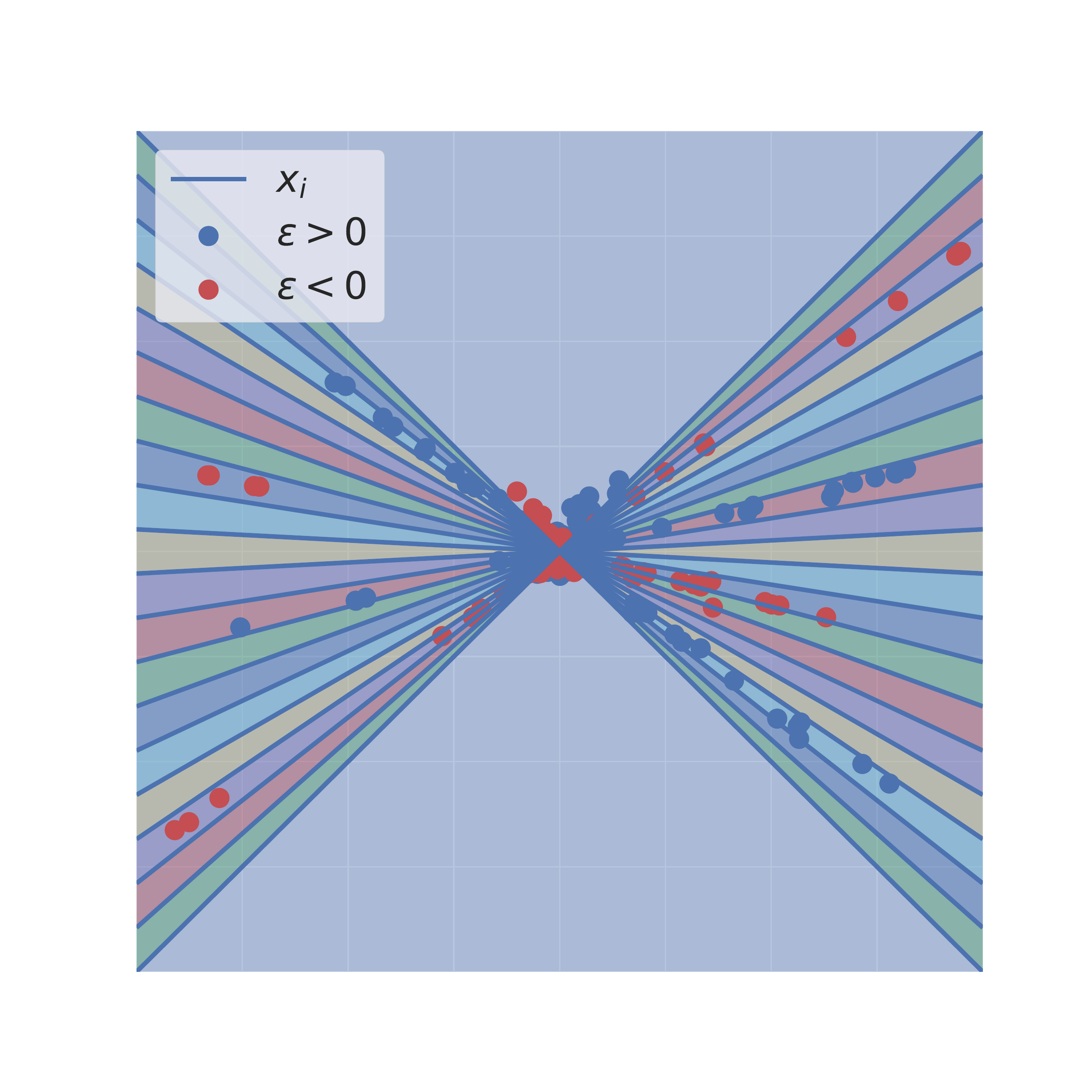}
    \endminipage\hfill
    
    \minipage{0.5\textwidth}
    \includegraphics[width=\linewidth]{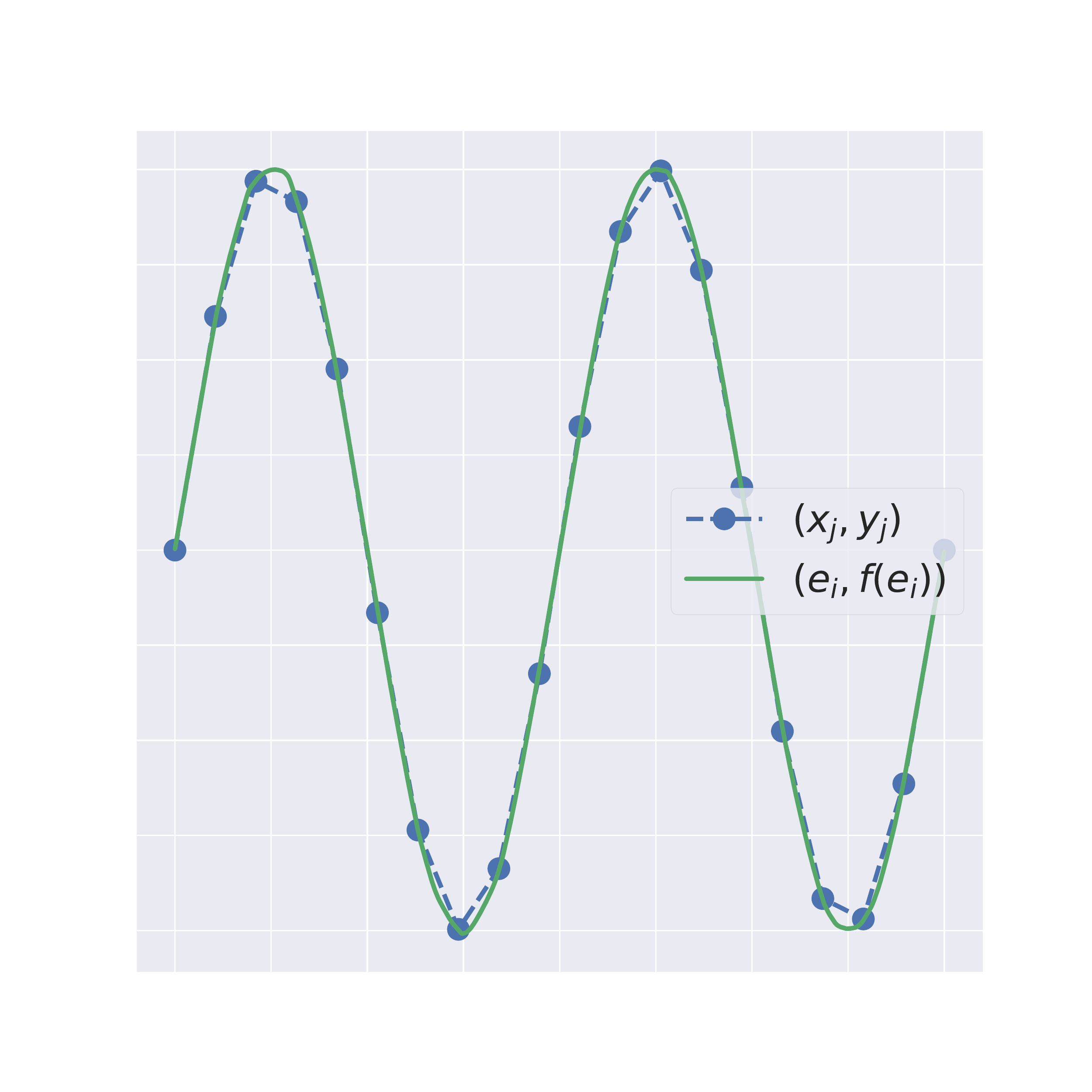}
    \endminipage\hfill
    \minipage{0.5\textwidth}
    \includegraphics[width=\linewidth]{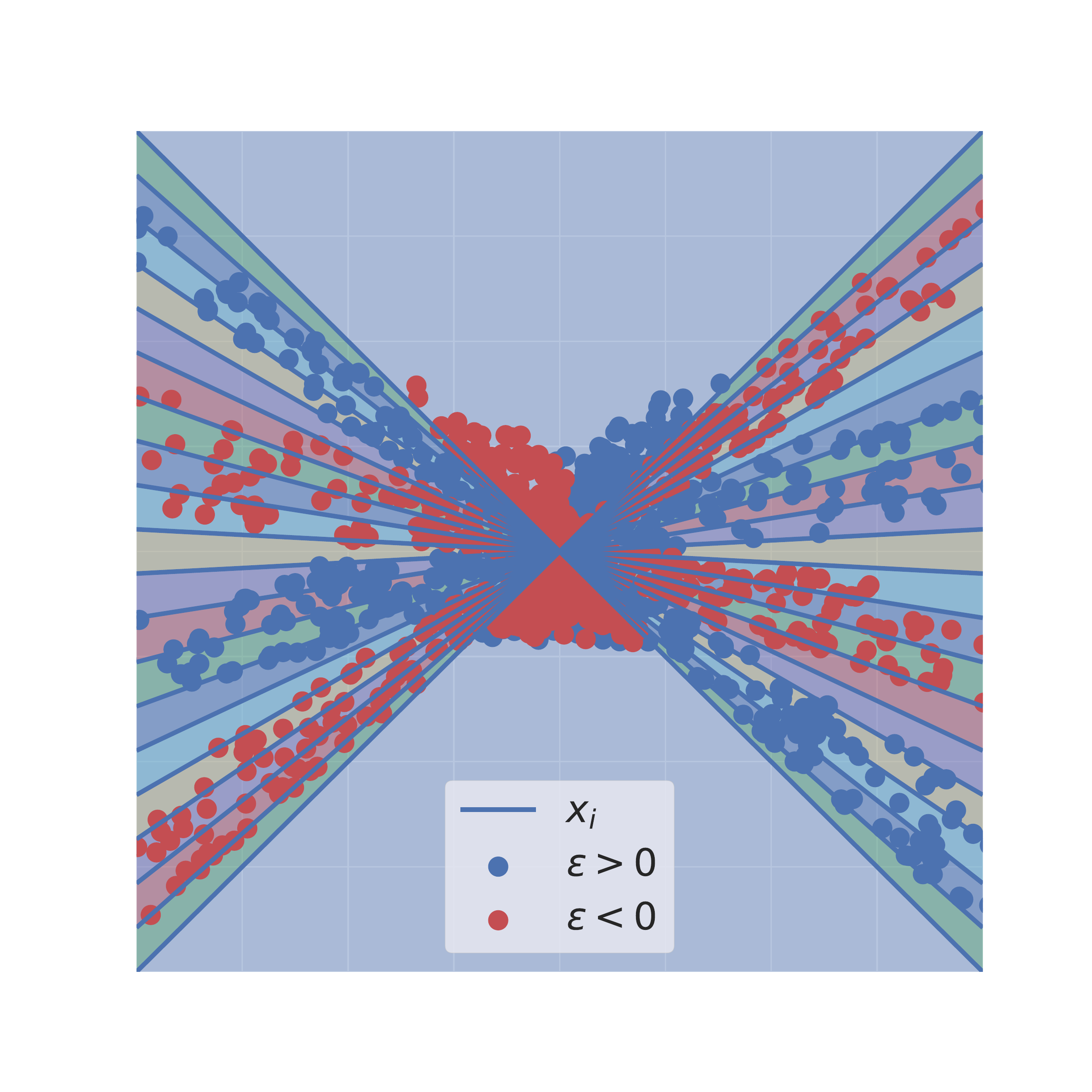}
    \endminipage\hfill
    
    \caption{The effect of varying the number of neurons $m$ the top image uses 20 neurons, the middle uses 200 and the bottom uses 2000. Observe that with fewer neurons, the function is adaptive to the data since $\delta$ gets larger.}
    \label{fig:varying_m}
\end{figure}

\subsection{Visualizing Attractor Samples}
We can visualize the vector field $(\partial_t(\hat{r}, \theta)$ by considering the change of metric from $w = (\hat{r}, \theta)$ to $(u, v)$ with the map 
$$
\pi_{(u, v)}(\hat{r}, \theta) = (|\hat{r}|\cos\theta, |\hat{r}|\sin\theta) = (u, v).
$$ 
Assuming we know the sign of $\hat{r}$, the vector field 

\begin{equation}
    \begin{bmatrix}
    \partial_t u\\
    \partial_t v
    \end{bmatrix} = 
    D\pi_{(u, v)} D\pi_{(u, v)}^T
    \begin{bmatrix}
    \partial_t r\\
    \partial_t \theta
    \end{bmatrix}
\end{equation}

Observing that $D\pi_{(u, v)} D\pi_{(u, v)}^T = I$, we have simply that
$$
\begin{bmatrix}
    \partial_t u\\
    \partial_t v
    \end{bmatrix} = \begin{bmatrix}
    \partial_t r\\
    \partial_t \theta
    \end{bmatrix}
$$

Figure~\ref{fig:attractor} shows a plot of this vector felt by a single particle in $uv$ in the case where $\delta = \infty$. In this case, the partial derivative $\partial_t r$ remains unchanged. Furthermore, we remark that at the boundaries of samples, the vector field can change directions, causing these samples to become ``attractors'' or ``repulsors'' (see Lemma~\ref{le:attractor} in the main document).
\begin{figure}
    \centering
    \includegraphics[width=\textwidth]{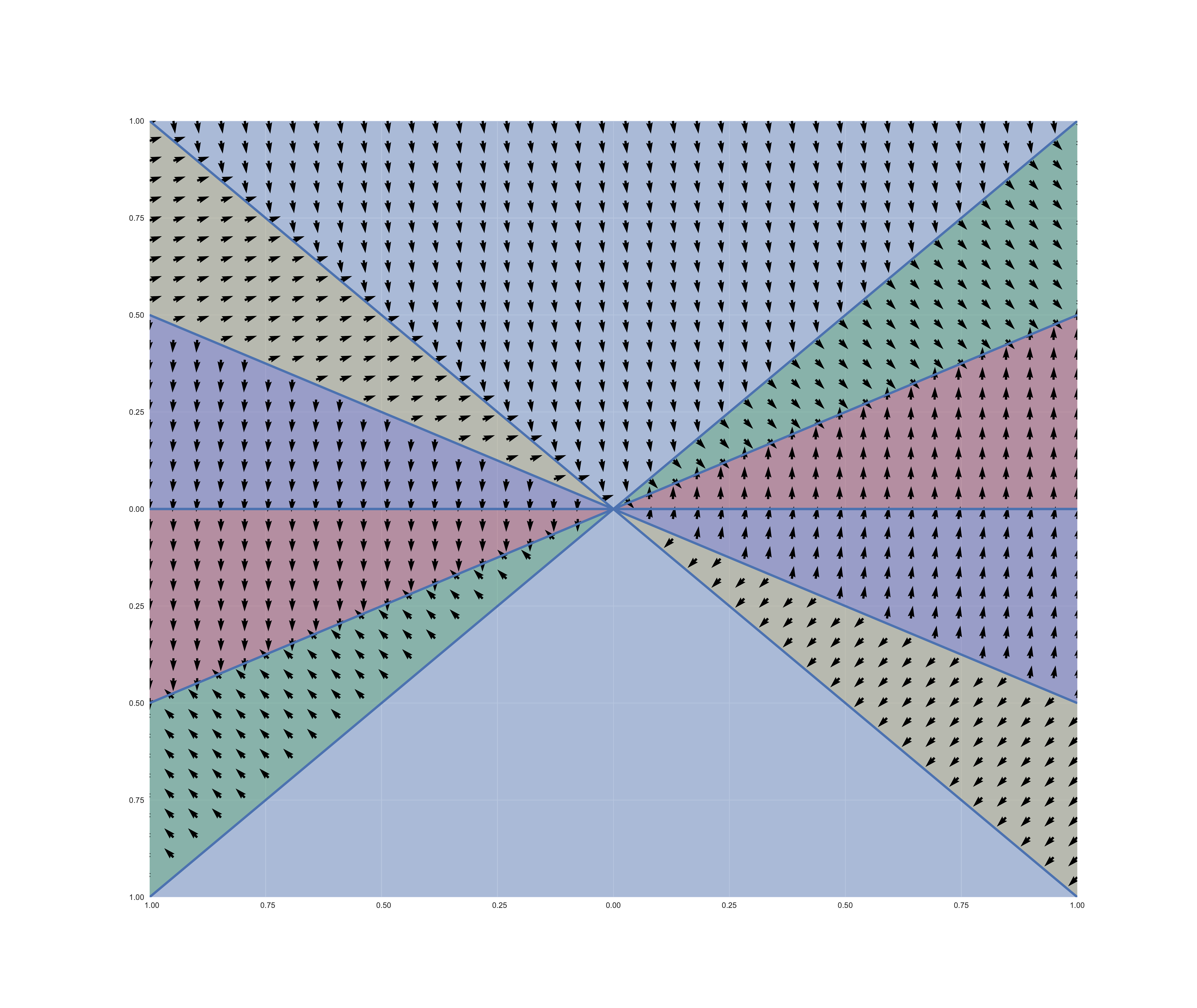}
    \minipage{0.5\textwidth}
    \includegraphics[width=\linewidth]{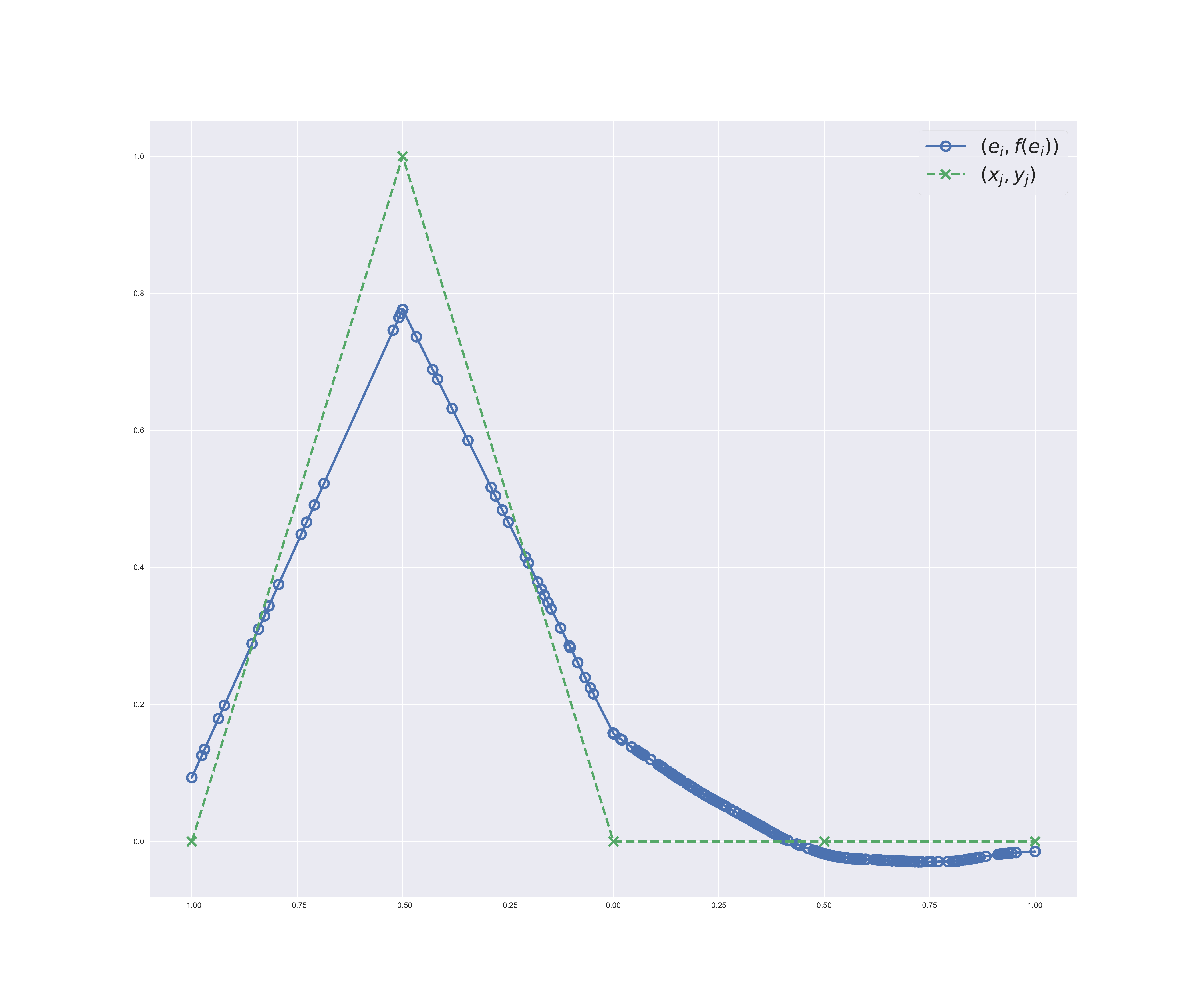}
    \endminipage\hfill
    \minipage{0.5\textwidth}
    \includegraphics[width=\linewidth]{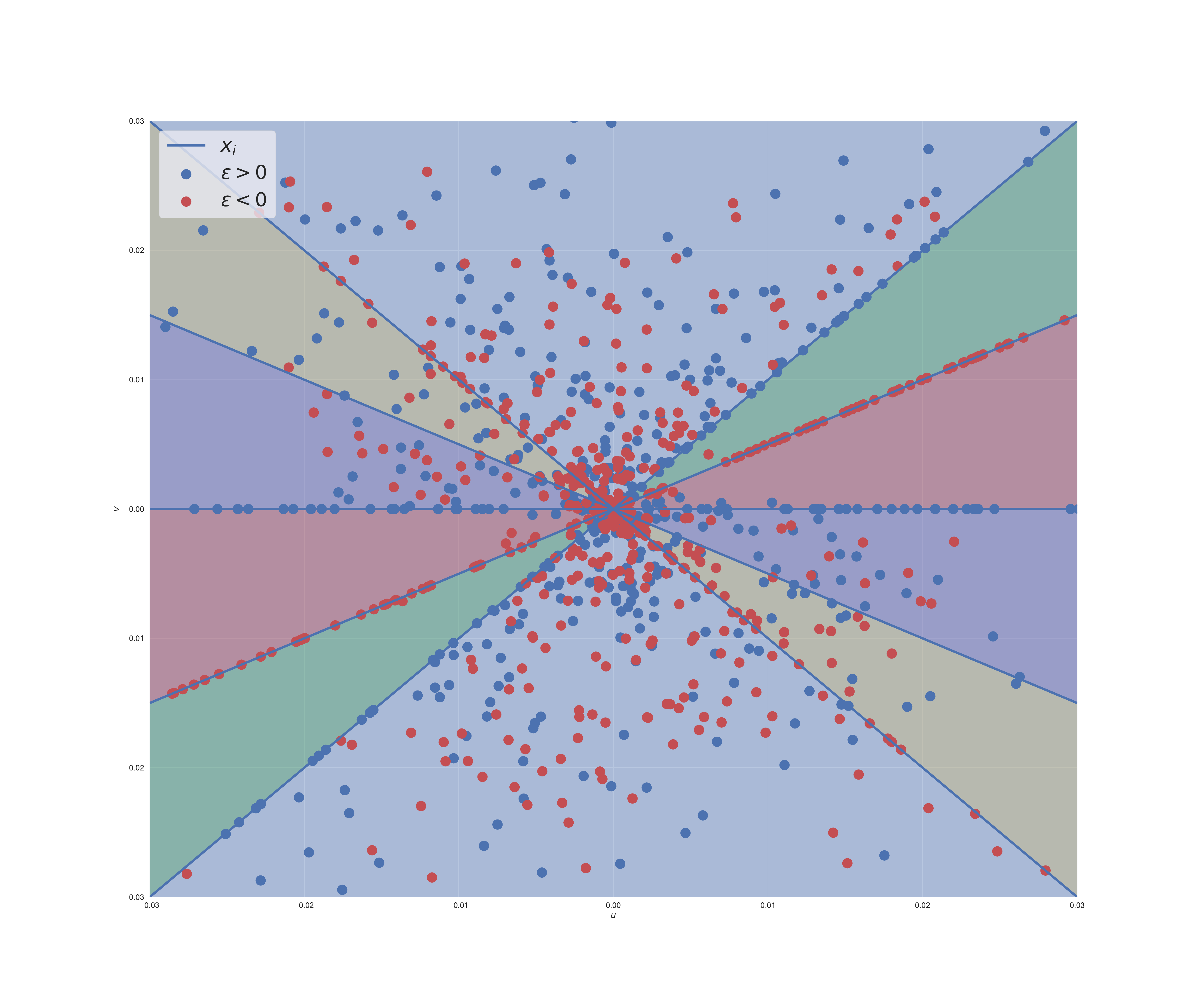}
    \endminipage\hfill
    \caption{\emph{Top:} The gradient field \eqref{eq:particle_vfield} felt by a particle. Note how the vectors change directions at certain samples. These samples are ``attractors'' or ``repulsors'' where particles get stuck or get pushed away from. \emph{Bottom Left:} A plot of the network function for the gradient field in the top image. Observe how there are clusters of neurons (blue circles) aligned with certain samples. \emph{Bottom Right:} A plot of the neurons in $uv$ space. Observe how the red neurons cluster at ``attractor'' points in the top image}
    \label{fig:attractor}
\end{figure}

\end{document}